\documentclass{article}

\PassOptionsToPackage{square,sort&compress,numbers}{natbib}
\usepackage[final]{neurips_2022}

\usepackage[utf8]{inputenc} 
\usepackage[T1]{fontenc}    
\usepackage[hidelinks]{hyperref}       
\usepackage{url}            
\usepackage{booktabs}       
\usepackage{amsfonts}       
\usepackage{nicefrac}       
\usepackage{microtype}      
\usepackage{xcolor}         
\usepackage{multicol}
\usepackage{multirow}
\usepackage{makecell}
\usepackage{comment}
\usepackage{relsize}
\usepackage{wrapfig}
\usepackage{thm-restate, amsmath, amssymb, amsfonts, mathtools}
\usepackage[ruled,vlined]{algorithm2e}

\usepackage{amsthm}
\makeatletter
\newtheorem*{rep@theorem}{\rep@title}
\newcommand{\newreptheorem}[2]{%
\newenvironment{rep#1}[1]{%
 \def\rep@title{#2 \ref{##1}}%
 \begin{rep@theorem}}%
 {\end{rep@theorem}}}
\makeatother

\newtheorem{theorem}{Theorem}
\newreptheorem{theorem}{Theorem}

\newreptheorem{assumption}{Assumption}
\newtheorem{proposition}[theorem]{Proposition}

\newreptheorem{lemma}{Lemma}

\newreptheorem{corollary}{Corollary}

\theoremstyle{remark}
\newtheorem{remark}{Remark}
\newtheorem*{remark*}{Remark}

\DeclareMathOperator*{\argmin}{arg\,min}
\newcommand{\norm}[1]{\left\lVert#1\right\rVert}
\newcommand{\snorm}[1]{\lVert#1\rVert}
\newcommand{\pder}[2]{\frac{\partial#2}{\partial #1}}
\newcommand{\spder}[2]{\partial_{#1}#2}
\newcommand{\der}[2]{\frac{\mathrm{d}#2}{\mathrm{d} #1}}
\newcommand{\sder}[2]{\mathrm{d}_{#1}#2}
\newcommand{\Id}{\mathrm{Id}}

\newcommand{\calL}{\mathcal{L}}
\newcommand{\calH}{\mathcal{H}}

\newcommand{\btheta}{\theta}
\newcommand{\bphi}{\phi}
\newcommand{\bpsi}{\psi}
\newcommand{\blambda}{\lambda}
\newcommand{\bomega}{\omega}

\newcommand{\bmu}{\mu}
\newcommand{\beps}{\epsilon}
\newcommand{\tphi}{\tilde{\bphi}}

\newcommand{\tu}{\tilde{\bu}}
\newcommand{\sphi}{\bphi_{\mathrm{ss}}}
\newcommand{\su}{\bu_{\mathrm{ss}}}
\newcommand{\bnu}{\nu}
\newcommand{\nphi}{\bphi_{\mathrm{next}}}
\newcommand{\npsi}{\bpsi_{\mathrm{next}}}
\newcommand{\nbu}{\bu_{\mathrm{next}}}
\newcommand{\neps}{\beps_{\mathrm{next}}}
\newcommand{\bxi}{\xi}

\newcommand{\bu}{u}

\newcommand{\by}{y}
\newcommand{\bx}{x}

\newcommand*{\Scale}[2][4]{\scalebox{#1}{$#2$}}
\newcommand*{\opt}{_{\Scale[0.42]{\bigstar}}}
\newcommand*{\fss}{^*}
\newcommand*{\css}{_*}

\newcommand{\ff}{f(\bphi, \btheta)}

\newcommand{\ffdown}{f(\bphi\css, \btheta)}
\newcommand{\LCPLag}{\mathcal{L}(\bphi, \bpsi, \blambda, \bmu, \btheta)}
\newcommand{\LCPLagOpt}{\mathcal{L}(\bphi\opt, \bpsi\opt, \blambda\opt, \bmu\opt, \btheta)}

\newcommand{\HH}{\mathcal{H}}

\newcommand{\EE}{\mathbb{E}}

\newcommand{\calN}{\mathcal{N}}
\usepackage[toc,page,header]{appendix}
\usepackage{minitoc}
\usepackage[capbesideposition=outside,capbesidesep=quad]{floatrow}
\usepackage{sidecap}
\usepackage{tabularx}
\usepackage{lscape}
\usepackage{multibib}
\newcites{S}{Supplementary References}
\newcolumntype{L}{>{\arraybackslash}X} 

\doparttoc 
\faketableofcontents 

\title{The least-control principle\\for local learning at equilibrium}

\author{%
  \textbf{Alexander Meulemans\thanks{Equal contribution}~ \thanks{Work partly done at the Institute of Neuroinformatics,
  University of Z\"{u}rich \& ETH Z\"{u}rich}$^{~\text{  1}}$, Nicolas Zucchet$^{*\text{1}}$, Seijin Kobayashi$^{*\dagger\text{1}}$} \\
  \textbf{Johannes von Oswald$^{\text{1}}$, João Sacramento$^{\text{2}}$}\\
  \\
  $^{\text{1}}$Department of Computer Science, 
  ETH Z\"{u}rich\\
  $^{\text{2}}$Institute of Neuroinformatics, University of Zürich and ETH Zürich\\
  \texttt{\{ameulema, nzucchet, seijink, voswaldj, rjoao\}@ethz.ch}  
}

\begin{document}

\maketitle
\vspace{-0.4cm}

\begin{abstract}

Equilibrium systems are a powerful way to express neural computations. As special cases, they include models of great current interest in both neuroscience and machine learning, such as deep neural networks, equilibrium recurrent neural networks, deep equilibrium models, or meta-learning. Here, we present a new principle for learning such systems with a temporally- and spatially-local rule. Our principle casts learning as a \emph{least-control} problem, where we first introduce an optimal controller to lead the system towards a solution state, and then define learning as reducing the amount of control needed to reach such a state. We show that incorporating learning signals within a dynamics as an optimal control enables transmitting activity-dependent credit assignment information, avoids storing intermediate states in memory, and does not rely on infinitesimal learning signals. In practice, our principle leads to strong performance matching that of leading gradient-based learning methods when applied to an array of problems involving recurrent neural networks and meta-learning. Our results shed light on how the brain might learn and offer new ways of approaching a broad class of machine learning problems.

\end{abstract}

\vspace{-0.25cm}

\section{Introduction}
The neural networks of the cortex are both layered and highly recurrent \citep{felleman_distributed_1991,douglas_neuronal_2004}. Their high degree of recurrence and relatively low depth stands in contrast to the prevailing design of artificial neural networks, which have high depth and little to no recurrence. This discrepancy has triggered a recent wave of research into recurrent networks that are more brain-like and which achieve high performance in perceptual tasks \citep{liao_bridging_2016,nayebi_task-driven_2018,kubilius_brain-like_2019,van_bergen_going_2020,linsley_stable_2020}. Concurrently, another line of recent work has shown that repeating a short sequence of neural computations until convergence can lead to large gains in efficiency, reaching the state-of-the-art in various machine learning problems while reducing model size \citep{bai_deep_2019,bai_multiscale_2020,huang_implicit2_2021}.

As we develop models that come closer to cortical networks by way of their recurrence, the precise mechanisms supporting learning in the brain remain largely unknown. While gradient-based methods currently dominate machine learning, standard methods for efficient gradient computation result in non-local rules that are hard to interpret in biological terms \citep{grossberg_competitive_1987,crick_recent_1989}. This issue is particularly aggravated when applying these methods to complex systems involving recurrence \citep{lillicrap_backpropagation_2019}. Indeed, while multiple interesting proposals \citep{kording_supervised_2001,lee_difference_2015,lillicrap_random_2016,sacramento_dendritic_2018,roelfsema_control_2018,whittington_theories_2019,richards_dendritic_2019,richards_deep_2019,lillicrap_backpropagation_2020,payeur_burst-dependent_2021} have emerged for how to efficiently compute loss function gradients for feedforward networks in biologically-plausible ways,  apart from a few notable exceptions \citep{scellier_equilibrium_2017,bellec_solution_2020} much less progress has been made for recurrent networks. Furthermore, the majority of these methods requires that error feedback does not influence network activity, a property at odds with many experimental findings on activity-dependent plasticity \cite{martin2000synaptic}. In this paper, we focus on the problem of learning such recurrent systems using biologically-plausible, activity-dependent local rules. To make progress in this longstanding question we make the assumption that our system is at equilibrium, and formalize learning as the following optimization problem:
\begin{equation}
    \label{eq:original_optimization_problem}
    \min{}_{\!\btheta} \, L(\bphi\fss) \quad \mathrm{s.t.} \enspace f(\bphi\fss, \btheta) = 0,     
\end{equation}
where $\btheta$ are the parameters we wish to learn, $L$ is a loss function which measures performance, and $\dot{\bphi} = f(\bphi,\btheta)$ is the system dynamics which is at an equilibrium point $\bphi^*$. Our model is thus \emph{implicitly} defined through a dynamics that is at equilibrium. Expressing our problem in this general form allows us to model a very broad class of learning systems, without being restricted to a particular type of neural network \citep{amos_differentiable_2019,gould_deep_2021,el_ghaoui_implicit_2021}.  In particular, we cover both feedforward and recurrent neural architectures. More generally, $f$ is not even restricted to being a neural dynamics. Consider the case where $f$ defines a learning algorithm governing the dynamics of the  \emph{weights} of a network; in this case \eqref{eq:original_optimization_problem} becomes a meta-learning problem, where the goal is to tune the (meta-)parameters of a learning algorithm such that its performance improves.

Inspired by a recent control-based learning method for feedforward neural networks \citep{meulemans_minimizing_2022}, we present a new principle for solving problems of the form \eqref{eq:original_optimization_problem} which yields learning rules that are (i) gradient-following, (ii) local in space, (iii) local in time, and in particular do not require storing intermediate states, (iv) activity-dependent by embedding error information into the network activity, and (v) not reliant on infinitesimally small learning signals. To meet all five criteria at once, we depart from direct gradient-based optimization of the loss, and we reformulate gradient-based learning within the framework of optimal control as a problem of \emph{least-control}.

Such problems can be approached in two steps. First, an optimal controller provides additional feedback input leading the dynamical system to a least-control state: an equilibrium point in which the loss $L$ is minimized with the least amount of control. Subsequently, the parameters are changed to further reduce the amount of control at the resulting equilibria. Critically, this minimization can be achieved using a local gradient-based rule for which all the required information is available at the controlled equilibria. This should be contrasted to the ubiquitous backpropagation of error algorithm \cite{werbos_beyond_1974, rumelhart_learning_1986} (and its variant for equilibrium systems \citep{almeida_backpropagation_1989,pineda_recurrent_1989}) which yields non-local parameter updates based on the outcome of two separate phases. Importantly, our least-control problem is intimately related to the original learning problem; we make the connection between the two precise by identifying mathematical conditions under which least-control solutions correspond to minima of the original objective $L(\phi^*)$, or a bound of it.

We apply our principle to learn equilibrium neural networks featuring a variety of architectures, including both fully-connected and convolutional feedfoward networks, laterally-connected recurrent networks of leaky integrator neurons, as well as deep equilibrium models \citep{bai_deep_2019,bai_multiscale_2020,huang_implicit2_2021}, a recent family of high-performance models which repeat until convergence a sequence of complex computations. To further demonstrate the generality of our principle, we then consider a recently studied meta-learning problem where the goal is to change the internal state of a complex synapse such that future learning performance is improved \citep{zucchet_contrastive_2022}. We find that our single-phase local learning rules yield highly competitive performance in both application domains. Our results extend the limits of what can be achieved with local learning rules, opening novel avenues to an old problem.

\section{The least-control principle}
\label{section:principle}
 
\begin{figure}[ht]
    \centering
    \includegraphics{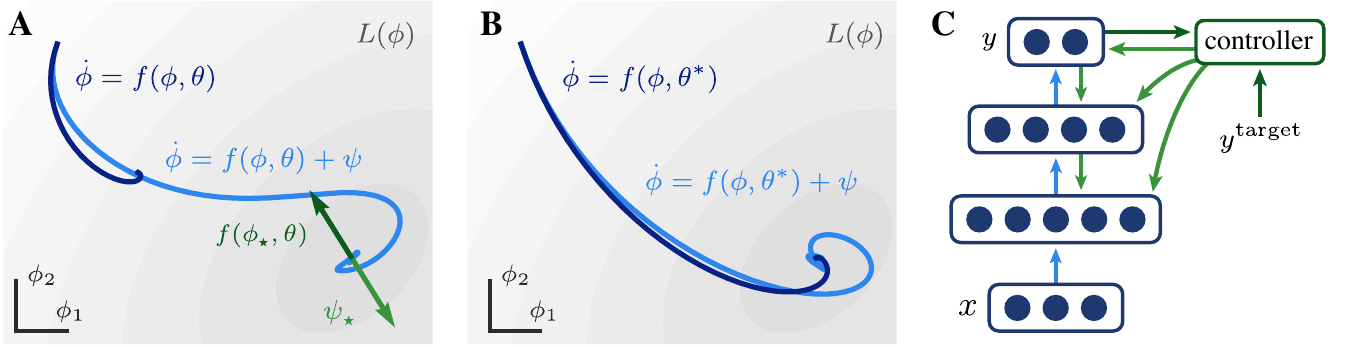}
    \caption{(A, B) Intuition behind the least-control principle. (A) During learning, the free dynamics (dark blue) is augmented with a controller (light blue) that drives the system towards an equilibrium that minimizes the loss function $L$ (grey, the darker the smaller the loss is). Our principle prescribes minimizing the amount of optimal control $\bpsi\opt$ (light green) needed at an equilibrium that minimizes $L$. (B) After learning, the free equilibrium coincides with the controlled equilibrium and hence minimizes the loss $L$. (C) Example of an instantiation of the principle for the supervised learning of a feedforward neural network. The controller both drives the output units $y$ of the network towards a target value $y^\mathrm{target}$ and influences the rest of the network. The control signal $\psi$ can be computed directly from the output error or indirectly by leveraging internal feedback within the network.}
    \label{fig:intuition}
\end{figure}
\vspace{-0.25cm}
\subsection{The principle}
\label{sec:the_principle}

We introduce the least-control principle for learning dynamical systems that evolve in time according to $\dot{\bphi}(t) = f(\bphi(t),\btheta)$ until an equilibrium $\bphi\fss$ is reached. 
Instead of directly minimizing a loss evaluated at equilibrium $L(\bphi\fss)$ as in \eqref{eq:original_optimization_problem}, we augment the dynamics $\dot{\bphi}(t) = f(\bphi(t),\btheta) + \bpsi(t)$ with a control signal $\bpsi(t)$ that drives the system towards a new controlled equilibrium where the loss is minimized (c.f. Figure \ref{fig:intuition}).
Our principle recasts learning as minimizing the amount of control needed to reach that state:
\begin{equation} \label{eqn:least_control_principle}
	\min_{\btheta, \bphi, \bpsi} \, \frac{1}{2}\snorm{\bpsi}^2 \quad \mathrm{s.t.} \enspace  \ff + \bpsi = 0, ~ \nabla_{\bphi} L(\bphi) = 0.
\end{equation}
The two constraints in the equation above ensure that we reach a steady state of the controlled dynamics, and that the loss $L$ is minimized at this state. Without loss of generality, we consider a single data point above
 (c.f.~Section~\ref{sec_app:multiple_datapoints} in the supplementary materials).

To solve \eqref{eqn:least_control_principle}, we first fix the parameters $\btheta$. We then find an optimal control $\bpsi\opt$ and controlled equilibrium $\bphi\opt$ that minimize the following least-control objective, for the current setting of $\btheta$, 
\begin{equation}
    \label{eqn:least_control_problem}
	\min_{\bphi, \bpsi} \, \frac{1}{2}\snorm{\bpsi}^2 \quad \mathrm{s.t.} \enspace  \ff + \bpsi = 0, ~ \nabla_{\bphi} L(\bphi) = 0.
\end{equation}
We call $\bphi\opt$ the least-control state and $\bpsi\opt$ the least control. Here, $\bpsi\opt$ represents an optimal control with terminal cost $\snorm{\bpsi}^2$, and a terminal constraint enforcing that the optimally controlled equilibrium $\bphi\opt$ minimizes the loss $L$. Critically, both credit assignment and system dynamics are now combined into a single phase through the optimal control $\bpsi\opt$. We then take a gradient-following update in $\btheta$ on $\HH(\btheta) := \frac{1}{2}\snorm{\bpsi\opt}^2$, the least-control objective of Eq.~\ref{eqn:least_control_problem}, which implicitly depends on $\btheta$ through $\bpsi\opt$.

Theorem \ref{theorem:first_order_updates} shows that the optimal control $\bpsi\opt$ contains enough information about the learning objective so that the gradient is easy to compute. We prove Theorem~\ref{theorem:first_order_updates}, along with all the theoretical results from Section~\ref{section:principle}, in Section~\ref{sec_app:proofs_LCP}. We use $\partial_x$ to denote partial derivatives w.r.t.~the variable $x$. 
\begin{theorem}[Informal]
    \label{theorem:first_order_updates}
    Under some mild regularity conditions, the least-control principle yields the following gradient for $\btheta$:
    \begin{align}\label{eqn:theta_lcp_gradient}
        \nabla_{\btheta} \HH(\btheta) &= -\spder{\btheta}{f}(\bphi\opt, \btheta)^\top \bpsi\opt = \spder{\btheta}{f}(\bphi\opt, \btheta)^\top f(\bphi\opt, \btheta).
    \end{align}
\end{theorem}

To illustrate what our principle achieves, let us briefly consider the setting where $f(\bphi,\btheta)$ implements a feedforward neural network, typically learned by backpropagation of a supervised error signal. In this case, the variable $\bphi$ corresponds to the state of the neurons (e.g.~their firing rate or postsynaptic voltage) and $\btheta$ to the synaptic connection weights. The least control $\bpsi\opt$ then drives the network to an equilibrium that minimizes the output loss, while being of minimum norm. The resulting weight update (c.f.~Section \ref{sec_app:rnn_local_updates}) now represents a local Hebbian rule multiplying presynaptic input with the postsynaptic control signal: the neural activity implicitly encodes credit assignment information, in such a way that local learning becomes possible. By contrast, in the backpropagation algorithm, errors do not directly influence neural activity and are computed in a separate phase.

The simplicity of our parameter update becomes apparent when contrasting it with the gradient associated with the original objective \eqref{eq:original_optimization_problem}, as the implicit function theorem (c.f. Section~\ref{sec_app:rbp}) gives
\begin{equation}
\label{eq:ift_gradient}
    \nabla_{\btheta} L(\bphi^*) = - \partial_{\btheta} f(\bphi^*, \btheta)^\top \left (\partial_{\bphi} f(\bphi^*, \btheta) \right )^{-\top} \nabla_{\bphi} L(\bphi^*) \!.
\end{equation}
Calculating this gradient directly requires inverting the Jacobian $\partial_{\bphi} f(\bphi^*, \btheta)\!$, which is intractable for large-scale systems. The standard way of dealing with this issue is to resort to the two-phase recurrent backpropagation algorithm, which estimates $\nabla_{\btheta} L(\bphi^*)$ by running a second linear dynamics until equilibrium, while holding the first equilibrium state $\bphi^*$ in memory. For acyclic computation graphs, of which feedforward neural networks are a prime example, it simplifies into the error backpropagation algorithm. Those procedures are still more complicated than ours: they require implementing a specialized auxiliary dynamical system, holding intermediate states in memory, and alternating between inference and gradient computation phases.

\subsection{A general class of dynamics leads to the same optimal control and parameter update} \label{sec:general_dynamics}
Our principle is agnostic to the dynamics $\dot{\bpsi}(t)$ of the controller; the only assumption made is that its value $\psi\css$ at the steady state is optimal, i.e.~$\psi\css = \psi\opt$. We use the subscripts $\css$ to denote equilibrium states of the controlled dynamics and $\opt$ to indicate optimality on the least-control objective of Eq.~\ref{eqn:least_control_problem}, and the superscript $\fss$ for a free (not controlled) equilibrium. In particular, the parameter update only depends on information available at the controlled equilibrium state $(\bphi_*, \bpsi_*)$. Here, we capitalize on this important flexibility by designing different simple feedback controller dynamics that lead to the same optimal steady state. We consider feedback controllers which take the following general form:
\begin{equation}\label{eqn:generalized_dynamics}
    \dot{\bphi} = \ff + Q(\bphi, \btheta) \bu, \quad\text{and}\quad \dot{\bu} = -\nabla_{\bphi}L(\bphi) - \alpha \bu,
\end{equation}
where $\bu$ is a controller defined on the output units and $Q(\bphi, \btheta)$ determines the influence of the controller on the entire system. We note that here and throughout our paper we omit the dependence on time (above, of $\bphi(t)$ and $\bu(t)$) from our equations to unclutter the presentation. In \eqref{eqn:generalized_dynamics}, the mapping $Q(\bphi, \btheta)$ can either be a direct mapping from the output controller $\bu$ to the entire system, or an indirect one leveraging internal feedback mechanisms in the system (c.f. Figure \ref{fig:intuition}C). The controller $\bu$ may be seen as a leaky integral controller with leak rate $\alpha$, that uses the output error $\nabla_{\bphi}L$ to drive the network towards a minimal-loss state (c.f. Section \ref{sec_app:proofs_LCP}). Returning to our example of a neural network, $\bu$ leads the network into a state configuration where the output neurons $\by$ are at the target value $y^\mathrm{target}$, when $\alpha=0$.

We use leaky integral control in \eqref{eqn:generalized_dynamics} for explanatory purposes and as we use it in our experiments. We remark however that our theory goes beyond this form of output control; in particular, it generalizes to any controller satisfying $\alpha \bu\css + \nabla_{\bphi} L(\bphi\css) = 0$ at equilibrium. For example, the proportional control $\bu = -\beta \nabla_\bphi L(\bphi)$ satisfies this condition for $\alpha = \beta^{-1}$.

Theorem \ref{theorem:columnspace} provides general conditions for the dynamics \eqref{eqn:generalized_dynamics} to converge to an optimal control signal needed for obtaining first-order parameter updates in Theorem \ref{theorem:first_order_updates}.

\begin{theorem}\label{theorem:columnspace}
    Let $(\bphi\css, \bu\css)$ be a steady state of the generalized dynamics \eqref{eqn:generalized_dynamics}. Assume that $\partial_{\bphi}{f}(\bphi\css, \btheta)$ is invertible and that the following column space condition is satisfied at equilibrium: 
    \begin{align}\label{eqn:columnspace_condition}
        \mathrm{col}\left [ Q(\bphi\css, \btheta) \right ] = \mathrm{row} \left [ [0, \mathrm{Id}] \, \spder{\bphi}{f}(\bphi\css, \btheta)^{-1} \right]
    \end{align}
    where the matrix $[0, \mathrm{Id}]$ has $|\phi|$ columns and as many rows as the system has output units. Then, $\bpsi\css = Q(\bphi\css, \btheta)\bu\css$ is an optimal control $\bpsi\opt$ for the least-control objective \eqref{eqn:least_control_problem}, in the limit $\alpha \rightarrow 0$.
\end{theorem}

Interestingly, Theorem \ref{theorem:columnspace} shows that many different feedback mappings from the output controller $\bu$ towards the system state $\bphi$ are possible, as long as the mapping $Q(\bphi\css, \btheta)$ satisfies the column space condition \eqref{eqn:columnspace_condition} at equilibrium, in the limit $\alpha \rightarrow 0$. Furthermore, the control $\bpsi$ need not be instantaneous and can have its own temporal dynamics $\dot{\bpsi}$ alongside the output controller dynamics $\dot{\bu}$, as the column space condition \eqref{eqn:columnspace_condition} is only defined at equilibrium. We exploit this property and construct a general-purpose algorithm that computes an optimal control $\bpsi\opt$:
\begin{align}\label{eqn:deq_inversion}
\dot{\bphi} = \ff + \bpsi, \quad\text{and}\quad \dot{\bpsi} = \spder{\bphi}{\ff^\top}\bpsi + \bu
\end{align}
with $\bu$ the leaky integral controller defined in \eqref{eqn:generalized_dynamics}. This \textit{inversion dynamics} has two special properties. First, it satisfies the conditions of Theorem~\ref{theorem:columnspace} by construction (c.f. Section~\ref{sec_app:columnspace}). Second, as defined above, $\dot{\bpsi}$ resembles the second-phase dynamics of the (recurrent) backpropagation algorithm (c.f.~Section~\ref{sec_app:rbp}), or in case of feedforward networks, the layerwise backpropagation of error signals in the error backpropagation algorithm. Intuitively, in this case we find an  optimally-controlled equilibrium by \emph{simultaneously} (in a single phase) running the first- and second-phase dynamics of (recurrent) backpropagation, with the critical difference that here the dynamical equations of both phases interact through the relation $\dot{\bphi} = \ff+\bpsi$.

\subsection{The least-control principle solves the original learning problem}
\label{sec:solve_original_problem}
Iterating our parameter update minimizes the least-control objective $\HH(\btheta)$ and not directly the original objective of learning \eqref{eq:original_optimization_problem}. The latter is however the objective of ultimate interest, used to measure the performance of the system after learning, when it is no longer under the influence of a controller. We now show that there is a close link between the two objectives and identify a wide range of conditions under which least-control solutions also solve the original learning problem \eqref{eq:original_optimization_problem}.

Proposition \ref{proposition:loss_equivalence_0} shows the intuitive result that if the optimal control is minimized to zero, the equilibria of the controlled and free dynamics coincide and hence the loss $L$ is minimized at the free equilibrium as well (c.f.~Figure~\ref{fig:intuition}.B).
\begin{proposition}\label{proposition:loss_equivalence_0}
    Assuming $L$ is convex on the system output, we have that the optimal control $\bpsi\opt$ is equal to 0 iff. the free equilibrium $\bphi\fss$ minimizes $L$.
\end{proposition}

As $\bpsi\opt= -f(\bphi\opt, \btheta)$, reaching $\bpsi\opt = 0$ can only be done if the model is powerful enough to perfectly fit the controlled equilibria for all data points. We show in Proposition \ref{proposition:overparameterization}, that overparameterization, i.e., having more parameters $\btheta$ than system states $\bphi$, indeed helps solve the original learning problem \eqref{eq:original_optimization_problem}. The condition $\spder{\btheta}{f}(\bphi\opt, \btheta)$ being of full row rank, that is needed in this proposition, can only be satisfied when the dimensions satisfy $|\btheta| \geq |\bphi|$, with $\bphi$ the concatenated states of all data points, hence providing another perspective on why overparameterization helps.
\begin{proposition}\label{proposition:overparameterization}
    Assuming $L$ is convex on the system output, a local minimizer $\btheta$ of the least control objective $\HH$ is a global minimizer of the original learning problem \eqref{eq:original_optimization_problem}, under the sufficient condition that $\spder{\btheta}{f}(\bphi\opt, \btheta)$ has full row rank.
\end{proposition}

Fully minimizing the amount of control to zero is not always possible, which highlights the need to understand the relation between the least-control objective $\HH(\btheta) = \frac{1}{2}\snorm{\bpsi\opt}^2$ and the loss $L(\bphi\fss)$ from the original learning problem \eqref{eq:original_optimization_problem}. Proposition \ref{proposition:relation_H_C} shows that we can do so, as minimizing $\snorm{\bpsi\opt}^2$ indirectly minimizes the loss $L(\bphi\fss)$, under some regularity and strong convexity assumptions.

\begin{proposition}\label{proposition:relation_H_C}
    If $\snorm{f}^2$ is $\mu$-strongly convex, $L$ is $M$-Lipschitz continuous and the minimum of $L$ is 0, then
    \begin{equation*}
        L(\bphi\fss) \leq \frac{\sqrt{\mu}}{\sqrt{2}M}\snorm{\bpsi\opt} = \frac{\sqrt{\mu}}{\sqrt{2}M}\snorm{f(\bphi\opt, \btheta)}.
    \end{equation*}
\end{proposition}

\textbf{Approximate equilibria.} Our least-control theory requires the system to be at equilibrium, and that the loss is minimized at the controlled equilibrium. This is almost never the case in practice, for instance because the dynamics are run for a finite amount of time, the column space condition for $Q$ is not perfectly satisfied, or because $\alpha$ is not zero in the output controller $\bu$ dynamics. The update we take therefore does not strictly follow the gradient of the least-control objective $\HH(\btheta)$. We formally show in Section~\ref{sec_app:approximate_equilibria} that our update resulting from an approximate optimally-controlled equilibrium is close to the gradient $\nabla_{\btheta} \HH$. More precisely, under some regularity conditions on $f$, the error in the gradient estimation is upper bounded by the distance between $\bphi\opt$ and its estimate. Additionally, if we reach equilibrium but $\alpha$ is non-zero the update $-\spder{\btheta}{f}(\bphi\css, \btheta)^\top \psi\css$ follows the gradient of an objective very similar to the least-control one $\HH$, under strict conditions on the feedback mapping $Q$.

\subsection{The least-control principle as constrained energy minimization} \label{sec:constrained_energy}

We now provide a dual view of our least-control principle from the perspective of energy minimization. This will allow designing a second general-purpose dynamics for computing the least control $\bpsi\opt$, and establishing a link to another principle for learning known as the free-energy principle. To arrive at an energy function, we rewrite the least-control objective \eqref{eqn:least_control_problem} by substituting the constraint $\bpsi = -f(\bphi, \btheta)$ into this objective, leading to the following constrained optimization problem: 
\begin{equation} \label{eqn:least_violation_physics}
    \bphi\opt = \argmin{}_{\! \bphi} \, \tfrac{1}{2}\snorm{\ff}^2 \quad \mathrm{s.t.} \enspace \nabla_{\bphi}L(\bphi) = 0
\end{equation}
The role of optimal control can therefore be reinterpreted as finding the state that is the closest to equilibrium under the free dynamics, among the states that minimize the loss function.

Next, we introduce the augmented energy $F(\bphi, \btheta, \beta) := \frac{1}{2}\snorm{\ff}^2 + \beta L(\bphi)$, which adds a nudging potential to the equilibrium energy $\snorm{\ff}^2$, and take the \emph{perfect control limit}, $\beta \to \infty$, where this potential dominates the energy. As we show in Section~\ref{sec_app:energy_based_version}, minimizing the augmented energy with respect to the state variable $\bphi$ leads to fulfilling our objective \eqref{eqn:least_violation_physics}. Interestingly, when $\beta\to\infty$ and $f(\bphi, \btheta)$ implements a feedforward neural network, we recover the free-energy function which governs the predictive coding model of \citet{whittington_approximation_2017}, which was obtained from an entirely different route based on variational expectation maximization for a probabilistic latent variable model \citep{friston_free-energy_2009} (c.f. Section \ref{sec_app:free_energy}).

\textbf{Energy-based dynamics as optimal control.} 
We leverage this connection to design a second general-purpose dynamics that computes an optimal control, by using the gradient flow on $F$:
\begin{align}\label{eqn:generalized_energy_dynamics}
    \dot{\phi} = -\partial_{\bphi}\ff^\top \ff - \beta \nabla_{\bphi}L(\bphi)
\end{align}
Here, the control $\bpsi$ is implicitly contained in the dynamics of $\bphi$. Part of the dynamics, $-\beta \nabla_{\bphi}L(\bphi)$, plays the role of an infinitely-strong proportional controller on the system output; another part is in charge of optimally sending the teaching signal from the output back to the rest of the system. We refer to Section \ref{sec_app:energy_based_version} for more details.

\textbf{Two opposing limits: perfect control vs.~weak nudging.} The update $\partial_{\btheta}f^\top f$ that we obtain in Theorem~\ref{theorem:first_order_updates} has been extensively studied in the weak nudging limit $\beta \rightarrow 0$ \citep{carreira-perpinan_distributed_2014,whittington_approximation_2017,dold_lagrangian_2019}, which sits at the opposite end of the $\beta$-spectrum of the perfect control limit studied here. A seminal result shows that as $\beta \to 0$ the gradient $\nabla_{\btheta} L(\bphi^*)$ of the original objective function is recovered \citep{scellier_equilibrium_2017,scellier_deep_2021}. However, this update is known to be sensitive to noise as the value of $f$ can be very small at the weakly-nudged equilibrium \citep{zucchet_contrastive_2022}. Our least-control principle works at the opposite perfect control end of the spectrum ($\beta \rightarrow \infty$), and it is therefore more resistant to noise. One of our main contributions is to show that using the same update, but instead evaluated at an optimally-controlled state $\bphi\opt$, still performs principled gradient-based optimization, however now on the least-control objective $\HH(\btheta)$. 

\section{Applications of the least-control principle}
The least-control principle applies to a general class of dynamical systems at equilibrium. By contrast, the majority of previous work on local learning focused on designing circuits and rules tailored towards feedforward neural networks \citep{kording_supervised_2001,guerguiev_towards_2017,whittington_approximation_2017,sacramento_dendritic_2018,akrout_deep_2019,meulemans_theoretical_2020,podlaski_biological_2020,pozzi_attention-gated_2020,payeur_burst-dependent_2021}. We now demonstrate the generality of our principle by applying it to two problems of interest in neuroscience and machine learning: deep and recurrent neural network learning,
and meta-learning.\footnote{Source code for all experiments is available at \url{github.com/seijin-kobayashi/least-control}} In both cases, our principle leads to activity-dependent local learning rules by leveraging Theorem \ref{theorem:first_order_updates}, while making use of simple and flexible optimal controllers to feed back credit assignment information (c.f.~Theorem~\ref{theorem:columnspace}). We test our learning and meta-learning systems on standard benchmarks and find that they perform competitively when compared to conventional gradient-based methods.

\subsection{The least-control principle for feedforward and recurrent neural networks}
\label{section:RNN}

As a first demonstration of learning according to our principle we consider feedforward neural networks with multiple layers of hidden neurons, the current preferred architecture for a large array of perceptual problems. Moreover, prior work on biologically-plausible alternatives to backpropagation has mostly focused on feedforward networks, making them a natural first choice to study our principle.

Motivated by the massive recurrence of cortical networks, we then turn to equilibrium recurrent neural networks as a second application of our principle, aiming to demonstrate its applicability beyond feedforward models. While there is debate over the precise functional roles of recurrent processing in the brain, there is strong experimental evidence that these connections are not limited to transmitting learning signals and play an active role in information processing \citep[e.g.,][]{gilbert_top-down_2013,manita_top-down_2015,marques_functional_2018,kirchberger_essential_2021}. From the viewpoint of machine learning, it has been shown that equilibrium RNNs are often more parameter-efficient than feedforward networks, while reaching state-of-the-art performance \citep[][]{liao_bridging_2016,linsley_stable_2020,bai_multiscale_2020}.

More concretely, we consider a generic neural network driven by a fixed input $\bx$ assumed to converge to an equilibrium state, obeying the following free dynamics (e.g. Figure \ref{fig:intuition}.C):
\begin{equation}\label{eqn:rnn_dynamics}
    \dot{\bphi} = \ff = -\bphi + W \sigma (\bphi) + U \bx 
\end{equation}
with $\bphi$ the neural activities, $W$ the internal synaptic weight matrix, $U$ the input weight matrix and $\sigma$ the activation function. Note that this notation is general enough to encompass feedforward and recurrent network architectures; while freely chosen $W$ and $U$ yield a vanilla RNN, structuring $W$ and $U$ appropriately (with block-diagonal structure) allows recovering standard feedforward fully-connected or convolutional deep neural networks. We point to Section \ref{sec_app:rnn} for full architectural details.  
The learning problem then consists of optimizing the network weights $\btheta = \{W, U\}$ on a loss $L(\bphi\fss)$ at the equilibrium activity $\bphi\fss$. This loss is only measured on the output neurons, a subset of $\phi$.
  
We now add a control signal $\bpsi$ to the dynamics \eqref{eqn:rnn_dynamics}. Theorem \ref{theorem:first_order_updates} guarantees that if the controlled steady state $\bpsi\css$ is optimal, the following updates minimize the least-control objective $\HH(\btheta)$:
\begin{align} \label{eqn:rnn_parameter_updates}
    \Delta W = \bpsi\css \sigma(\bphi\css)^\top, \quad \mathrm{and}  \quad \Delta U = \bpsi\css \bx^\top\!.
\end{align}
To showcase the flexibility of our principle, we use Theorem \ref{theorem:columnspace} to design various simple feedback circuits that compute these control signals, as we detail next. We visualize those circuits in Figure~\ref{fig_app:circuits}.

\subsubsection{A simple controller with direct linear feedback}
To instantiate the least-control principle in its simplest form, we project the output controller $\bu$ onto the hidden neurons $\bphi$ with direct linear feedback weights $Q$, which is a generalization of a recent feedforward network learning method \citep{meulemans_minimizing_2022} to equilibrium neural networks. Broadcasting output errors directly \citep{nokland_direct_2016} may be seen as one of the simplest possible ways of avoiding the weight transport problem \citep{grossberg_competitive_1987,crick_recent_1989,lillicrap_random_2016}. This results in the following dynamics:
\begin{align}\label{eqn:rnn_linear_q}
    \dot{\bphi} = -\bphi + W\sigma(\bphi) + U \bx + Q \bu, \quad \mathrm{and} \quad \dot{\bu} = - \nabla_{\bphi} L(\bphi) - \alpha \bu
\end{align}
with $\bu$ a simple leaky integral output controller (c.f.  Section \ref{sec:general_dynamics}), and $\nabla_{\bphi} L$ the output error. 

Theorem \ref{theorem:columnspace} guarantees that $Q\bu\css$ is an optimal control in the limit of $\alpha \rightarrow 0$, if the feedback weights $Q$ satisfy the column space condition $\mathrm{col}[Q] = \mathrm{row}\left[[0, \mathrm{Id}](\Id - W\sigma'(\bphi\css))^{-1}\right]$ for all data samples. For a linear feedback mapping, this condition cannot be exactly satisfied for multiple samples as the row space is data-dependent, whereas $Q$ is not. Still, we learn the feedback weights $Q$ to approximately satisfy this column space condition, by using a local Hebbian rule that operates simultaneously to the learning of the other weights, inspired by recent work on feedforward networks \citep{meulemans_minimizing_2022, meulemans_credit_2021} (details in Section \ref{sec_app:rnn}).
The resulting update provably finds feedback weights $Q$ that satisfy the column space condition for one sample, and we empirically verify that the control signal it gives in the multiple samples regime is close to optimal (c.f. Section~\ref{sec_app:rnn}).

Despite the simplicity of this linear feedback controller, we show that the training procedure described above is still powerful enough to learn both a two-hidden-layer (each of 256 neurons) fully-connected feedforward network and an equilibrium RNN with a fully-connected recurrent layer of 256 neurons on the MNIST digit classification task \cite{lecun_mnist_1998} (c.f.~Table~\ref{table:mnist}, LCP-LF). Notably, it performs almost as well as backpropagation (BP) and recurrent backpropagation (RBP), the current methods of choice for equilibrium RNN training. We use fixed point iterations to find the steady states $\bphi\css$ and $\bu\css$ in Eq. (\ref{eqn:rnn_linear_q}) for computational efficiency (c.f. Section \ref{sec_app:rnn} for more simulation details). We also observe that the empirical performance improves by changing the parameter updates to $\Delta W = \sigma'(\bphi\css) Q\bu\css \sigma(\bphi\css)^\top$ and $\Delta U = \sigma'(\bphi\css) Q\bu\css \bx^\top$ when the column space condition is not perfectly satisfied, corroborating the findings of \citet{meulemans_minimizing_2022, meulemans_credit_2021} for feedforward neural networks.

\begin{SCtable}[50][h]
    \begin{tabular}{llll}
        \toprule
        Method & FF & RNN  \\
        \midrule
        LCP\,-\,LF & 97.73{$^{\pm0.07}$} & {97.70$^{\pm0.11}$} \\
        LCP\,-\,DI &{98.11$^{\pm0.07}$} & {97.58$^{\pm0.12}$} \\
        LCP\,-\,DI (KP)&{98.14$^{\pm0.09}$} & {97.75$^{\pm0.11}$} \\
        LCP\,-\,EBD& {98.00$^{\pm0.03}$}&{97.60$^{\pm0.15}$} \\
        \midrule
        BP/RBP &{98.29$^{\pm0.14}$} &  {97.87$^{\pm0.19}$} \\ 
        \bottomrule
    \end{tabular}
    \caption{MNIST test set classification accuracy (\%)
    for a feedforward network (2x256 neurons) and an RNN (256 neurons) trained by resp.~backpropagation (BP) and recurrent backpropagation (RBP), and by the least-control principle (LCP) with linear feedback (LF), energy-based dynamics (EBD), dynamic inversion (DI) and with learned feedback weights (DI KP). Mean $\pm$ std computed over 3 seeds.}
    \label{table:mnist}
\end{SCtable}

\subsubsection{Learning with general-purpose optimal control dynamics}
\label{sec:rnn-general-purpose-experiments}

In the previous section we designed and tested a simple error broadcast circuit and showed that we can learn it such that the conditions of Theorem~\ref{theorem:columnspace} are approximately met. This approximate feedback is likely insufficient for more complex tasks which require harnessing depth or strongly-recurrent dynamics  \citep{bartunov_assessing_2018}. This leads us to consider two feedback circuits with more detailed architectures, which have the capacity to steer the network to an (exact) optimally-controlled equilibrium.

The first circuit is based on our general-purpose dynamic inversion  \eqref{eqn:deq_inversion}. Applied to RNNs it yields:
\begin{align} \label{eqn:rnn_dynamic_inversion}
    \dot{\bphi} = - \bphi + W\sigma(\bphi) + U \bx + \bpsi, \quad \text{and} \quad
    \dot{\bpsi} = -\bpsi + \sigma'(\bphi) S \bpsi + \bu.
\end{align}
The second circuit exploits the dual energy minimization view of our principle presented in Section~\ref{sec:constrained_energy}. Taking the energy-descending dynamics $\dot{\bphi} = - \partial_{\bphi} F(\bphi, \btheta,\beta)^\top$ of Eq. \ref{eqn:generalized_energy_dynamics}, we obtain
\begin{align}
\label{eqn:rnn_energy_dynamics}
    \dot{\bphi} = - \bpsi + \sigma'(\bphi) S \bpsi + \bu, \quad \text{and} \quad \bpsi = \bphi - W\sigma(\bphi) - U \bx, 
\end{align}
where we use the output controller $\bu$ of Eq. \ref{eqn:generalized_dynamics} as a generalization of the proportional control $-\beta \nabla_{\by} L$ to make the two circuits more directly comparable. Above, we introduce decoupled feedback weights $S$ to avoid sharing weights between $\bphi$ and $\bpsi$, and we learn $S$ to align with $W^\top$ by using the Kolen-Pollack local learning rule \citep{kolen_back-propagation_1994, akrout_deep_2019}, which simply transposes the weight update $\Delta W$ \eqref{eqn:rnn_parameter_updates} and adds weight decay to both update rules: $\Delta S = \sigma(\bphi\css) \bpsi\css^\top - \gamma S$.

Both circuits implement optimal control through neural dynamics, avoiding explicit matrix inverses and phases. In biological terms, they lead to different interpretations and implementation possibilities, highlighting the flexibility of our learning principle. The dynamic inversion circuit \eqref{eqn:rnn_dynamic_inversion} fits naturally with dendritic error implementations, where $\bpsi$ can be linked to feedback dendritic compartments whose signals steer plasticity \citep{kording_supervised_2001,guerguiev_towards_2017,sacramento_dendritic_2018,richards_dendritic_2019,payeur_burst-dependent_2021,meulemans_credit_2021,meulemans_minimizing_2022,mikulasch_dendritic_2022}. In particular, by leveraging burst-multiplexing circuits \citep{payeur_burst-dependent_2021,kording_supervised_2001} or interneuron microcircuits \citep{sacramento_dendritic_2018}, all information for the dynamics and weight updates can be made locally available (cf.~Section~\ref{sec_app:rnn}). On the other hand, the energy-minimizing dynamics \eqref{eqn:rnn_energy_dynamics} naturally leads to an implementation based on prediction and error neuron populations, characteristic of predictive coding circuits \citep{rao_predictive_1999,bastos_canonical_2012,whittington_approximation_2017,keller_predictive_2018}. We provide a high-level discussion of the two circuit alternatives (\ref{eqn:rnn_dynamic_inversion}, \ref{eqn:rnn_energy_dynamics}) in Section~\ref{sec_app:rnn}.

\begin{wrapfigure}[15]{r}{0.41\textwidth}
    \vspace{-0.25cm}
    \centering
    \includegraphics{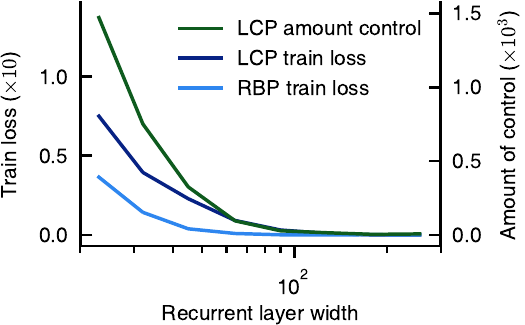}
    \caption{Training loss $L$ and the amount of control $\snorm{\psi\opt} = \sqrt{2\calH}$ after training untill convergence on MNIST with various recurrent layer widths.}
    \label{fig:overparam}
\end{wrapfigure}
First, we test our optimally-controlled neural networks by training a fully-connected equilibrium RNN on MNIST with circuit \eqref{eqn:rnn_dynamic_inversion}, while varying the recurrent layer width, cf.~Figure~\ref{fig:overparam}. This analysis shows that our theoretical results of Section~\ref{sec:solve_original_problem} translate to practice, as for strongly overparameterized networks, both $\HH$ and $L$ are minimized to zero, whereas for underparameterized networks, minimizing the least-control objective $\HH(\btheta)$ leads to a significant decrease in $L$. Moreover, our feedforward (2x256 hidden neurons) and recurrent (256 hidden neurons) networks achieve competitive test performance compared to backpropagation (recurrent backpropagation, for RNNs), both when the feedback weights are trained (LCP\,-\,DI KP) and when they are tied ($S=W^\top$; LCP\,-\,DI and LCP-EBD), irrespective of whether we use the energy-based \eqref{eqn:rnn_energy_dynamics} or dynamic inversion \eqref{eqn:rnn_dynamic_inversion} optimal control dynamics, cf.~Table~\ref{table:mnist}.

\begin{wraptable}[15]{r}{0.55\textwidth}
\vspace{-0.3cm}
    \centering
    \caption{C10: CIFAR-10 test accuracy (\%) of a convolutional feedforward (FF) and equilibrium RNN (RNN), learned by dynamic inversion (DI; Kollen-Pollack learning without weight sharing: KP) and by standard/recurrent backpropagation (BP/RBP). INR: the peak signal-to-noise ratio (in dB) of an implicit neural representation learned on a natural image dataset \cite{fathony_multiplicative_2020}. Mean $\pm$ std computed over 3 seeds. }
    \label{table:cifar}
    \begin{tabular}{llll} \toprule
        Method  & C10 (FF) & C10 (RNN)  & INR  \\ \midrule
        BP/RBP&77.58$^{\pm0.14}$  & 80.14$^{\pm0.20}$ & 25.47$^{\pm4.16}$ \\ 
        LCP-DI&77.28$^{\pm0.10}$ &  80.26$^{\pm0.17}$  & 25.11$^{\pm2.90}$ \\ 
        LCP-KP&77.16$^{\pm0.10}$ &  --  & -- \\ 
        \bottomrule
    \end{tabular}
\end{wraptable}
Next, we test whether the circuits and rules obtained through our principle scale to more complex architectures and tasks and move to the CIFAR-10 image classification benchmark \citep{krizhevsky_learning_2009}. We study two distinct models: a standard feedforward convolutional neural network, and a deep equilibrium model (DEQ) consisting of a recurrent block of convolutions \cite{bai_multiscale_2020}. The latter can be interpreted as an efficient infinite-depth residual neural network with weight-sharing \cite{he_deep_2016, chen_neural_2018}. Full architectural details are provided in Sections~\ref{sec_app:training_deq}~and~\ref{sec_app:conv-net-details}; we note that to showcase our least-control principle as a scalable learning method in the DEQ experiments, we make a number of design choices that include using automatic differentiation to compute certain derivatives. For our feedforward network experiments we make no such concessions, and we further investigate a weight-transport-free variant using Kollen-Pollack learning. We show on Table~\ref{table:cifar} that the networks trained by our principle are able to match those trained with backpropagation. Interestingly, we observe that the DEQ models considered here are not overparameterized enough to not need control, but that this does not impair final predictive performance (c.f.~Table~\ref{table:cifar_detailed}).

Lastly, we turn to implicit neural representations \citep{sitzmann_implicit_2020}, yet another task for which recurrent connections are known to be particularly beneficial \cite{huang_implicit2_2021}. The network there has to represent an image in its weights so that it is able to map pixel coordinates to their color values. Briefly, we once again find that our principle performs competitively to RBP. We refer full details in Section~\ref{sec_app:rnn} for more details.

\subsection{The least-control principle for meta-learning}\label{section:metalearning}
The generality of the least-control principle enables considering learning paradigms other than supervised learning. Here, we show that it can be applied to meta-learning, a framework that aims at improving a learning algorithm by leveraging shared structure in encountered tasks \citep{thrun_learning_1998}. This framework is gaining traction in neuroscience, having been featured in recent work aiming at understanding which systems \citep{behrens_what_2018,wang_prefrontal_2018,wang_meta-learning_2021} and rules \citep{zucchet_contrastive_2022} could support meta-learning in the brain.

\textbf{Meta-learning complex synapses.}
Following \cite{zucchet_contrastive_2022, rajeswaran_meta-learning_2019}, we model the learning algorithm for training a neural network with synaptic weights $\bphi$ on some task $\tau$ as the minimization of a data-driven loss $L^\text{train}_\tau(\bphi)$, regularized by a quadratic term $\lambda ||\bphi - \bomega||^2$. This regularizer can be interpreted as a model of a complex synapse; similar models have been introduced in studies of continual learning and long-term memory \citep{fusi_cascade_2005,ziegler_synaptic_2015,benna_computational_2016,zenke_continual_2017,kirkpatrick_overcoming_2017}. In the context of meta-learning, the fast weights $\bphi$ quickly learn how to solve a task and $\lambda$ determines how strongly $\bphi$ is pulled towards the slow weights $\bomega$, which consolidate the knowledge that has been previously acquired. The capabilities of the learning algorithm are measured by evaluating $L^\mathrm{eval}_\tau(\bphi\fss_\tau)$, the performance of the learned neural network on data from task $\tau$ that was held out during learning. Meta-learning finally corresponds to improving the learning performance on many tasks by adjusting the meta-parameters $\btheta = \bomega$, leading to the following problem:
\begin{equation}\label{eqn:metalearning_bo}
	\min_{\btheta} \, \EE_\tau \! \left [ L^\mathrm{eval}_\tau(\bphi\fss_\tau) \right ] \quad \mathrm{s.t.} \enspace \bphi_\tau\fss = \argmin_{\bphi} L^\mathrm{learn}_\tau(\bphi) + \frac{\lambda}{2} \snorm{\bphi - \bomega}^2.
\end{equation}
Above, the parameters $\bphi$ no longer represent neural activities and $\btheta$ synaptic weights as in the previous section; $\bphi$ and $\btheta = \bomega$ now designate synaptic weights and meta-parameters, respectively. 

\begin{wraptable}[14]{r}{0.37\textwidth}
    \vspace{-0.5cm}
    \label{tab:omniglot}
    \centering
    \begin{tabular}{ll}
        \toprule
        Method & Test accuracy\\
        \midrule
        FOMAML \cite{finn_model-agnostic_2017}  & 89.40$^{\pm0.50}$ \\
        T1-T2 \cite{luketina_scalable_2016} & 89.72$^{\pm 0.43}$\\
        Reptile \cite{nichol_first-order_2018} & 89.43$^{\pm0.14}$\\
        \midrule
        LCP (ours) &  91.00$^{\pm 0.24}$ \\
        \midrule
        iMAML \cite{rajeswaran_meta-learning_2019} &94.46$^{\pm0.42}$\\
        CML \cite{zucchet_contrastive_2022} & 94.16$^{\pm0.12}$\\
        \bottomrule
    \end{tabular}
    \caption{Meta-learning neural networks with complex synapses on Omniglot 20-way 1-shot learning tasks. Mean $\pm$ std test set classification accuracy (\%) computed over 3 seeds.}
\end{wraptable}
\textbf{Meta-parameter updates.} Using the first-order conditions $-\nabla_{\bphi} L^\mathrm{train}_\tau(\bphi_\tau\fss) - \lambda (\bphi_\tau\fss - \btheta) = 0$ associated to $\bphi_\tau\fss$ being a minimizer in the previous equation yields an optimization problem on which the least-control principle can be applied.
Our principle prescribes running the controlled dynamics
\begin{equation}
    \begin{split}
        \dot{\bphi} &= -\nabla_{\bphi} L^\mathrm{learn}_\tau(\bphi) - \lambda (\bphi - \bomega) + \bu,\\
        \dot{\bu} &= -\alpha \bu - \nabla_{\bphi} L^\mathrm{eval}_\tau(\bphi)
    \end{split}
\end{equation}
or any other process that finds the same fixed points $(\bphi\css^\tau, \bu\css^\tau)$. The theory we developed in Section \ref{sec:general_dynamics} guarantees that $\bu\css^\tau$ is an optimal control as long as $\alpha \rightarrow 0$ and the Hessian of the evaluation loss $\partial_{\bphi}^2 L^\mathrm{eval}_\tau$ at the equilibrium $\bphi\css^\tau$ is invertible. Hence, we do not need complex dynamics to compute the gradient, highlighting the flexibility and simplicity of learning according to the least-control principle. The resulting single-phase meta-learning rule is local to every synapse (c.f. Section~\ref{sec_app:meta_learning}):
\begin{align}
    \Delta \bomega = \lambda \bu\css^\tau.
\end{align}
\textbf{Experimental results.} Our meta-learning algorithm stands out from competing methods as being both principled, first-order and single-phase. On the one hand, existing first-order and single-phase methods either approximate meta-gradients (e.g., FOMAML \cite{finn_model-agnostic_2017} and T1-T2 \cite{luketina_scalable_2016}) or rely on heuristics (e.g. Reptile \cite{nichol_first-order_2018}), and therefore do not provably minimize a meta-objective function. On the other hand, full meta-gradient descent methods require two phases to backpropagate through learning \cite{finn_model-agnostic_2017}, invert an intractable Hessian \cite{lorraine_optimizing_2020, rajeswaran_meta-learning_2019}, or learn a slightly modified task \cite{zucchet_contrastive_2022}, and often require second derivatives \cite{finn_model-agnostic_2017, lorraine_optimizing_2020, rajeswaran_meta-learning_2019}. Here, we compare our rule to competing first-order methods (FOMAML, T1-T2 and Reptile) and meta-gradient approximation algorithms that use the same synaptic model (implicit MAML and contrastive meta-learning; iMAML and CML) on the Omniglot few-shot visual classification benchmark \cite{lake_human-level_2015}. Each Omniglot task consists of distinguishing 20 classes after seeing only 1 sample per class; we choose this 20-way 1-shot regime as the gap between first- and second-order methods is reputedly the largest compared to other variants of this benchmark. We find that our meta-learning rule significantly outperforms other single-phased first-order methods, while still lagging behind iMAML, a second-order and two-phased method, and CML, a two-phased first-order method that can approximate the meta-gradient arbitrarily well. These findings are particularly interesting since our synaptic model \eqref{eqn:metalearning_bo} is in the underparameterized regime: there is a separate set of weights $\bphi_\tau$ per task $\tau$ but only one set of consolidated weights $\omega$ for the whole task distribution. The strong performance observed here is another demonstratation that least-control solutions can perform well even when overparameterization arguments (cf.~Section~\ref{sec:solve_original_problem}) do not hold.

\section{Discussion}

\textbf{Control theories of neural network learning.} The relationship between control theory and neural network learning is an old and intricate one. It is by now well known that backpropagation, the method of choice for computing gradients in deep learning, can be seen as the discrete-time counterpart of the adjoint method developed in control theory for continuous-time systems \citep{bryson_applied_1969,baydin_automatic_2018}. One of the early papers realizing this connection formulated neural network learning as an optimization problem under equilibrium constraints, as in the original problem statement \eqref{eq:original_optimization_problem}, and rederived the recurrent backpropagation algorithm \citep{lecun_theoretical_1988}. More recent work has explored using feedback neural control circuits which influence the neural dynamics of feedforward \citep{meulemans_credit_2021,meulemans_minimizing_2022} and recurrent \citep{gilra_predicting_2017,deneve_brain_2017,alemi_learning_2018} networks to deliver learning signals. Different from such approaches, our least-control principle is rooted in \emph{optimal} control: instead of setting the control from an ansatz, we derived it from first principles. This allowed us to arrive at a simple local rule for gradient-based learning, and to identify general sufficient conditions for the feedback control signals.

\textbf{Predictive coding from the least-control perspective.} We have shown that least-control problems correspond to free-energy minimization problems. This dual view of least-control problems allows connecting to influential theories of neural information processing based on predictive coding \citep{rao_predictive_1999,friston_free-energy_2009,whittington_approximation_2017,keller_predictive_2018,whittington_theories_2019}. Our results elucidate predictive coding theories in at least three ways. First, Theorem \ref{theorem:first_order_updates} justifies the parameter update taken in supervised predictive coding networks \citep{whittington_approximation_2017} as gradient-based optimization of our least-control objective. This connection offers a novel view of the predictive coding dynamics as an optimal control algorithm in charge of controlling hidden neurons, while `clamping' output neurons to the respective targets can be seen as a form of output control. Previous arguments for gradient-based learning were only known for the weak nudging limit ($\beta \to 0$) where learning signals are vanishingly small \citep{scellier_equilibrium_2017}, which is problematic in noisy systems such as the cortex \citep{rusakov_noisy_2020}; our novel theoretical and experimental findings characterize the opposite perfect control ($\beta \to \infty$) end of the spectrum. Second, Theorem \ref{theorem:columnspace} broadens the class of neural network circuits which may be used to implement predictive coding, beyond the precisely-constructed microcircuits that have been proposed so far \citep{whittington_approximation_2017,sacramento_dendritic_2018}. Finally, our experiments on recurrent neural networks using free-energy-minimizing dynamics to assign credit are to the best of our knowledge the first of the kind, and serve to validate the effectiveness of predictive coding circuits beyond feedforward architectures.

\textbf{Limitations.} A limiting factor of our theory is that in its current form it is only applicable to equilibrium systems. While the study of neural dynamics at equilibrium is an old endeavor in neuroscience and neural network theory \citep{hopfield_neurons_1984,cohen_absolute_1983}, it remains unclear how strong an assumption it is when modeling the networks of the brain. Extending our principle to out-of-equilibrium dynamic is an exciting direction for future work. Another limitation concerning the breadth of problem~\eqref{eq:original_optimization_problem} is that we did not allow the loss function $L$ to directly depend on the parameters $\btheta$. We discuss extensions to this more general case in Section \ref{sec_app:first_order_gradients}, however not all loss functions $L$ in this more general class satisfy the conditions needed to obtain a local parameter update. Lastly, it should be noted that although our principle requires only one relaxation phase, its computational cost often exceeds the combined cost of performing the two phases of recurrent backpropagation (cf.~Section~\ref{sec_app:comp_cost}).

\textbf{Concluding remarks.} We have presented a new theory that enables learning equilibrium systems with local gradient-based rules. Our learning rules are driven by changes in activity generated by an optimal control, in charge of feeding back credit assignment information into the system as it evolves. Taken together, our results push the boundaries of what is possible with single-phase, activity-dependent local learning rules, and suggest an alternative to direct gradient-based learning which yields high performance in practice.

\begin{ack}
This research was supported by an Ambizione grant (PZ00P3\_186027) from the Swiss National Science Foundation and an ETH Research Grant (ETH-23 21-1) awarded to João Sacramento. Seijin Kobayashi was supported by the Swiss National Science Foundation (SNF) grant CRSII5\_173721. Johannes von Oswald was funded by the Swiss Data Science Center (J.v.O. P18-03). We thank Angelika Steger, Benjamin Scellier, Walter Senn, Il Memming Park, Blake A.~Richards, Rafal Bogacz, Yoshua Bengio, and members of the Senn, Bogacz and Richards labs for discussions.
\end{ack}

\bibliographystyle{unsrtnat}
\bibliography{references}

\begin{thebibliography}{22}
\providecommand{\natexlab}[1]{#1}
\providecommand{\url}[1]{\texttt{#1}}
\expandafter\ifx\csname urlstyle\endcsname\relax
  \providecommand{\doi}[1]{doi: #1}\else
  \providecommand{\doi}{doi: \begingroup \urlstyle{rm}\Url}\fi

\bibitem[Dontchev and Rockafellar(2009)]{dontchev_implicit_2009}
Asen~L. Dontchev and R.~Tyrrell Rockafellar.
\newblock \emph{Implicit functions and solution mappings}.
\newblock Springer, 2009.

\bibitem[Zintgraf et~al.(2019)Zintgraf, Shiarlis, Kurin, Hofmann, and
  Whiteson]{zintgraf_fast_2019}
Luisa Zintgraf, Kyriacos Shiarlis, Vitaly Kurin, Katja Hofmann, and Shimon
  Whiteson.
\newblock Fast context adaptation via meta-learning.
\newblock In \emph{International {Conference} on {Machine} {Learning}}, 2019.

\bibitem[Lee et~al.(2019)Lee, Maji, Ravichandran, and
  Soatto]{lee_meta-learning_2019}
Kwonjoon Lee, Subhransu Maji, Avinash Ravichandran, and Stefano Soatto.
\newblock Meta-learning with differentiable convex optimization.
\newblock In \emph{Proceedings of the {IEEE}/{CVF} {Conference} on {Computer}
  {Vision} and {Pattern} {Recognition}}, 2019.

\bibitem[Zhao et~al.(2020)Zhao, Kobayashi, Sacramento, and von
  Oswald]{zhao_meta-learning_2020}
Dominic Zhao, Seijin Kobayashi, João Sacramento, and Johannes von Oswald.
\newblock Meta-learning via hypernetworks.
\newblock In \emph{Workshop on {Meta}-{Learning} at {NeurIPS}}, 2020.

\bibitem[Nocedal and Wright(2006)]{nocedal_numerical_2006}
Jorge Nocedal and Stephen~J. Wright.
\newblock \emph{Numerical optimization}.
\newblock Springer, 2006.

\bibitem[Bertsekas(2014)]{bertsekas_constrained_2014}
Dimitri~P. Bertsekas.
\newblock \emph{Constrained optimization and {Lagrange} multiplier methods}.
\newblock Academic Press, 2014.

\bibitem[Friston et~al.(2006)Friston, Kilner, and Harrison]{friston_free_2006}
Karl Friston, James Kilner, and Lee Harrison.
\newblock A free energy principle for the brain.
\newblock \emph{Journal of Physiology-Paris}, 100\penalty0 (1-3):\penalty0
  70--87, 2006.

\bibitem[Neal and Hinton(1998)]{neal_view_1998}
Radford~M. Neal and Geoffrey~E. Hinton.
\newblock A view of the {EM} algorithm that justifies incremental, sparse, and
  other variants.
\newblock In \emph{Learning in {Graphical} {Models}}, pages 355--368. Springer
  Netherlands, 1998.

\bibitem[Tzikas et~al.(2008)Tzikas, Likas, and
  Galatsanos]{tzikas_variational_2008}
Dimitris~G. Tzikas, Aristidis~C. Likas, and Nikolaos~P. Galatsanos.
\newblock The variational approximation for {Bayesian} inference.
\newblock \emph{IEEE Signal Processing Magazine}, 25\penalty0 (6):\penalty0
  131--146, 2008.

\bibitem[Bogacz(2017)]{bogacz_tutorial_2017}
Rafal Bogacz.
\newblock A tutorial on the free-energy framework for modelling perception and
  learning.
\newblock \emph{Journal of Mathematical Psychology}, 76:\penalty0 198--211,
  2017.

\bibitem[Goldman et~al.(2010)Goldman, Compte, and Wang]{goldman_neural_2010}
Mark~S. Goldman, Albert Compte, and Xiao-Jing Wang.
\newblock Neural integrator models.
\newblock \emph{Encyclopedia of Neuroscience}, pages 165--178, 2010.

\bibitem[Särkkä and Solin(2019)]{sarkka_applied_2019}
Simo Särkkä and Arno Solin.
\newblock \emph{Applied stochastic differential equations}.
\newblock Cambridge University Press, 2019.

\bibitem[Lu and Shiou(2002)]{lu_inverses_2002}
Tzon-Tzer Lu and Sheng-Hua Shiou.
\newblock Inverses of 2 × 2 block matrices.
\newblock \emph{Computers \& Mathematics with Applications}, 43\penalty0
  (1):\penalty0 119--129, 2002.

\bibitem[Magee and Grienberger(2020)]{magee_synaptic_2020}
Jeffrey~C. Magee and Christine Grienberger.
\newblock Synaptic plasticity forms and functions.
\newblock \emph{Annual Review of Neuroscience}, 43:\penalty0 95--117, 2020.

\bibitem[Shepherd(2009)]{shepherd_dendrodendritic_2009}
Gordon~M. Shepherd.
\newblock Dendrodendritic synapses: past, present, and future.
\newblock \emph{Annals of the New York Academy of Sciences}, 1170\penalty0
  (1):\penalty0 215--223, 2009.

\bibitem[Hennequin et~al.(2017)Hennequin, Agnes, and
  Vogels]{hennequin_inhibitory_2017}
Guillaume Hennequin, Everton~J. Agnes, and Tim~P. Vogels.
\newblock Inhibitory plasticity: balance, control, and codependence.
\newblock \emph{Annual Review of Neuroscience}, 40:\penalty0 557--579, 2017.

\bibitem[Kingma and Ba(2015)]{kingma_adam_2015}
Diederik~P. Kingma and Jimmy Ba.
\newblock Adam: {A} method for stochastic optimization.
\newblock In \emph{International {Conference} on {Learning} {Representations}},
  2015.

\bibitem[Loshchilov and Hutter(2017)]{loshchilov_sgdr_2017}
Ilya Loshchilov and Frank Hutter.
\newblock {SGDR}: {Stochastic} gradient descent with restarts.
\newblock In \emph{International {Conference} on {Learning} {Representations}},
  2017.

\bibitem[Ioffe and Szegedy(2015)]{ioffe_batch_2015}
Sergey Ioffe and Christian Szegedy.
\newblock Batch normalization: accelerating deep network training by reducing
  internal covariate shift.
\newblock In \emph{International {Conference} on {Machine} {Learning}}, 2015.

\bibitem[Wu and He(2018)]{wu_group_2018}
Yuxin Wu and Kaiming He.
\newblock Group normalization.
\newblock In \emph{Proceedings of the {European} {Conference} on {Computer}
  {Vision}}, 2018.

\bibitem[Liao et~al.(2018)Liao, Xiong, Fetaya, Zhang, Yoon, Pitkow, Urtasun,
  and Zemel]{liao_reviving_2018}
Renjie Liao, Yuwen Xiong, Ethan Fetaya, Lisa Zhang, KiJung Yoon, Xaq Pitkow,
  Raquel Urtasun, and Richard Zemel.
\newblock Reviving and improving recurrent back-propagation.
\newblock In \emph{International {Conference} on {Machine} {Learning}}, 2018.

\bibitem[Zucchet and Sacramento(2022)]{zucchet_beyond_2022}
Nicolas Zucchet and João Sacramento.
\newblock Beyond backpropagation: implicit gradients for bilevel optimization.
\newblock \emph{arXiv preprint arXiv:2205.03076}, 2022.

\end{thebibliography}


\begin{thebibliography}{88}
\providecommand{\natexlab}[1]{#1}
\providecommand{\url}[1]{\texttt{#1}}
\expandafter\ifx\csname urlstyle\endcsname\relax
  \providecommand{\doi}[1]{doi: #1}\else
  \providecommand{\doi}{doi: \begingroup \urlstyle{rm}\Url}\fi

\bibitem[Felleman and van Essen(1991)]{felleman_distributed_1991}
D.~J. Felleman and D.~C. van Essen.
\newblock Distributed hierarchical processing in the primate cerebral cortex.
\newblock \emph{Cerebral Cortex}, 1\penalty0 (1):\penalty0 1--47, 1991.

\bibitem[Douglas and Martin(2004)]{douglas_neuronal_2004}
Rodney~J. Douglas and Kevan~A.C. Martin.
\newblock Neuronal circuits of the neocortex.
\newblock \emph{Annual Review of Neuroscience}, 27\penalty0 (1):\penalty0
  419--451, 2004.

\bibitem[Liao and Poggio(2016)]{liao_bridging_2016}
Qianli Liao and Tomaso Poggio.
\newblock Bridging the gaps between residual learning, recurrent neural
  networks and visual cortex.
\newblock \emph{arXiv preprint arXiv:1604.03640}, 2016.

\bibitem[Nayebi et~al.(2018)Nayebi, Bear, Kubilius, Kar, Ganguli, Sussillo,
  DiCarlo, and Yamins]{nayebi_task-driven_2018}
Aran Nayebi, Daniel Bear, Jonas Kubilius, Kohitij Kar, Surya Ganguli, David
  Sussillo, James~J. DiCarlo, and Daniel~L. Yamins.
\newblock Task-driven convolutional recurrent models of the visual system.
\newblock \emph{Advances in Neural Information Processing Systems}, 31, 2018.

\bibitem[Kubilius et~al.(2019)Kubilius, Schrimpf, Kar, Rajalingham, Hong,
  Majaj, Issa, Bashivan, Prescott-Roy, Schmidt, Nayebi, Bear, Yamins, and
  DiCarlo]{kubilius_brain-like_2019}
Jonas Kubilius, Martin Schrimpf, Kohitij Kar, Rishi Rajalingham, Ha~Hong, Najib
  Majaj, Elias Issa, Pouya Bashivan, Jonathan Prescott-Roy, Kailyn Schmidt,
  Aran Nayebi, Daniel Bear, Daniel~L. Yamins, and James~J. DiCarlo.
\newblock Brain-like object recognition with high-performing shallow recurrent
  {ANNs}.
\newblock In \emph{Advances in {Neural} {Information} {Processing} {Systems}},
  2019.

\bibitem[van Bergen and Kriegeskorte(2020)]{van_bergen_going_2020}
Ruben~S. van Bergen and Nikolaus Kriegeskorte.
\newblock Going in circles is the way forward: the role of recurrence in visual
  inference.
\newblock \emph{Current Opinion in Neurobiology}, 65:\penalty0 176--193, 2020.

\bibitem[Linsley et~al.(2020)Linsley, Karkada~Ashok, Govindarajan, Liu, and
  Serre]{linsley_stable_2020}
Drew Linsley, Alekh Karkada~Ashok, Lakshmi~Narasimhan Govindarajan, Rex Liu,
  and Thomas Serre.
\newblock Stable and expressive recurrent vision models.
\newblock In \emph{Advances in {Neural} {Information} {Processing} {Systems}},
  2020.

\bibitem[Bai et~al.(2019)Bai, Kolter, and Koltun]{bai_deep_2019}
Shaojie Bai, J.~Zico Kolter, and Vladlen Koltun.
\newblock Deep equilibrium models.
\newblock \emph{Advances in Neural Information Processing Systems}, 2019.

\bibitem[Bai et~al.(2020)Bai, Koltun, and Kolter]{bai_multiscale_2020}
Shaojie Bai, Vladlen Koltun, and J.~Zico Kolter.
\newblock Multiscale deep equilibrium models.
\newblock \emph{Advances in Neural Information Processing Systems}, 2020.

\bibitem[Huang et~al.(2021)Huang, Bai, and Kolter]{huang_implicit2_2021}
Zhichun Huang, Shaojie Bai, and J.~Zico Kolter.
\newblock Implicit$^{\textrm{2}}$: implicit layers for implicit
  representations.
\newblock In \emph{Advances in {Neural} {Information} {Processing} {Systems}},
  2021.

\bibitem[Grossberg(1987)]{grossberg_competitive_1987}
Stephen Grossberg.
\newblock Competitive learning: {From} interactive activation to adaptive
  resonance.
\newblock \emph{Cognitive science}, 11\penalty0 (1):\penalty0 23--63, 1987.

\bibitem[Crick(1989)]{crick_recent_1989}
Francis Crick.
\newblock The recent excitement about neural networks.
\newblock \emph{Nature}, 337:\penalty0 129--132, 1989.

\bibitem[Lillicrap and Santoro(2019)]{lillicrap_backpropagation_2019}
Timothy~P. Lillicrap and Adam Santoro.
\newblock Backpropagation through time and the brain.
\newblock \emph{Current Opinion in Neurobiology}, 55:\penalty0 82--89, 2019.

\bibitem[Körding and König(2001)]{kording_supervised_2001}
Konrad~P Körding and Peter König.
\newblock Supervised and unsupervised learning with two sites of synaptic
  integration.
\newblock \emph{Journal of Computational Neuroscience}, 11\penalty0
  (3):\penalty0 207--215, 2001.

\bibitem[Lee et~al.(2015)Lee, Zhang, Fischer, and Bengio]{lee_difference_2015}
Dong-Hyun Lee, Saizheng Zhang, Asja Fischer, and Yoshua Bengio.
\newblock Difference target propagation.
\newblock In \emph{Joint {European} {Conference} on {Machine} {Learning} and
  {Knowledge} {Discovery} in {Databases}}, 2015.

\bibitem[Lillicrap et~al.(2016)Lillicrap, Cownden, Tweed, and
  Akerman]{lillicrap_random_2016}
Timothy~P. Lillicrap, Daniel Cownden, Douglas~B. Tweed, and Colin~J. Akerman.
\newblock Random synaptic feedback weights support error backpropagation for
  deep learning.
\newblock \emph{Nature Communications}, 7\penalty0 (1):\penalty0 13276, 2016.

\bibitem[Sacramento et~al.(2018)Sacramento, Costa, Bengio, and
  Senn]{sacramento_dendritic_2018}
João Sacramento, Rui~P. Costa, Yoshua Bengio, and Walter Senn.
\newblock Dendritic cortical microcircuits approximate the backpropagation
  algorithm.
\newblock In \emph{Advances in {Neural} {Information} {Processing} {Systems}},
  2018.

\bibitem[Roelfsema and Holtmaat(2018)]{roelfsema_control_2018}
Pieter~R. Roelfsema and Anthony Holtmaat.
\newblock Control of synaptic plasticity in deep cortical networks.
\newblock \emph{Nature Reviews Neuroscience}, 19\penalty0 (3):\penalty0
  166--180, 2018.

\bibitem[Whittington and Bogacz(2019)]{whittington_theories_2019}
James C.~R. Whittington and Rafal Bogacz.
\newblock Theories of error back-propagation in the brain.
\newblock \emph{Trends in Cognitive Sciences}, 23\penalty0 (3):\penalty0
  235--250, 2019.

\bibitem[Richards and Lillicrap(2019)]{richards_dendritic_2019}
Blake~A. Richards and Timothy~P. Lillicrap.
\newblock Dendritic solutions to the credit assignment problem.
\newblock \emph{Current Opinion in Neurobiology}, 54:\penalty0 28--36, 2019.

\bibitem[Richards et~al.(2019)Richards, Lillicrap, Beaudoin, Bengio, Bogacz,
  Christensen, Clopath, Costa, de~Berker, Ganguli, Gillon, Hafner, Kepecs,
  Kriegeskorte, Latham, Lindsay, Miller, Naud, Pack, Poirazi, Roelfsema,
  Sacramento, Saxe, Scellier, Schapiro, Senn, Wayne, Yamins, Zenke, Zylberberg,
  Therien, and Kording]{richards_deep_2019}
Blake~A. Richards, Timothy~P. Lillicrap, Philippe Beaudoin, Yoshua Bengio,
  Rafal Bogacz, Amelia Christensen, Claudia Clopath, Rui~Ponte Costa, Archy
  de~Berker, Surya Ganguli, Colleen~J. Gillon, Danijar Hafner, Adam Kepecs,
  Nikolaus Kriegeskorte, Peter Latham, Grace~W. Lindsay, Kenneth~D. Miller,
  Richard Naud, Christopher~C. Pack, Panayiota Poirazi, Pieter Roelfsema, João
  Sacramento, Andrew Saxe, Benjamin Scellier, Anna~C. Schapiro, Walter Senn,
  Greg Wayne, Daniel Yamins, Friedemann Zenke, Joel Zylberberg, Denis Therien,
  and Konrad~P. Kording.
\newblock A deep learning framework for neuroscience.
\newblock \emph{Nature Neuroscience}, 22\penalty0 (11):\penalty0 1761--1770,
  2019.

\bibitem[Lillicrap et~al.(2020)Lillicrap, Santoro, Marris, Akerman, and
  Hinton]{lillicrap_backpropagation_2020}
Timothy~P. Lillicrap, Adam Santoro, Luke Marris, Colin~J. Akerman, and Geoffrey
  Hinton.
\newblock Backpropagation and the brain.
\newblock \emph{Nature Reviews Neuroscience}, 21\penalty0 (6):\penalty0
  335--346, 2020.

\bibitem[Payeur et~al.(2021)Payeur, Guerguiev, Zenke, Richards, and
  Naud]{payeur_burst-dependent_2021}
Alexandre Payeur, Jordan Guerguiev, Friedemann Zenke, Blake~A. Richards, and
  Richard Naud.
\newblock Burst-dependent synaptic plasticity can coordinate learning in
  hierarchical circuits.
\newblock \emph{Nature Neuroscience}, 24\penalty0 (7):\penalty0 1010--1019,
  2021.

\bibitem[Scellier and Bengio(2017)]{scellier_equilibrium_2017}
Benjamin Scellier and Yoshua Bengio.
\newblock Equilibrium propagation: bridging the gap between energy-based models
  and backpropagation.
\newblock \emph{Frontiers in Computational Neuroscience}, 11, 2017.

\bibitem[Bellec et~al.(2020)Bellec, Scherr, Subramoney, Hajek, Salaj,
  Legenstein, and Maass]{bellec_solution_2020}
Guillaume Bellec, Franz Scherr, Anand Subramoney, Elias Hajek, Darjan Salaj,
  Robert Legenstein, and Wolfgang Maass.
\newblock A solution to the learning dilemma for recurrent networks of spiking
  neurons.
\newblock \emph{Nature Communications}, 11\penalty0 (1):\penalty0 3625, 2020.

\bibitem[Martin et~al.(2000)Martin, Grimwood, and Morris]{martin2000synaptic}
Stephen~J. Martin, Paul~D. Grimwood, and Richard~G.M. Morris.
\newblock Synaptic plasticity and memory: an evaluation of the hypothesis.
\newblock \emph{Annual Review of Neuroscience}, 23\penalty0 (1):\penalty0
  649--711, 2000.

\bibitem[Amos(2019)]{amos_differentiable_2019}
Brandon Amos.
\newblock \emph{Differentiable optimization-based modeling for machine
  learning}.
\newblock {PhD} {Thesis}, Carnegie Mellon University, 2019.

\bibitem[Gould et~al.(2021)Gould, Hartley, and Campbell]{gould_deep_2021}
Stephen Gould, Richard Hartley, and Dylan~John Campbell.
\newblock Deep declarative networks.
\newblock \emph{IEEE Transactions on Pattern Analysis and Machine
  Intelligence}, 2021.

\bibitem[El~Ghaoui et~al.(2021)El~Ghaoui, Gu, Travacca, Askari, and
  Tsai]{el_ghaoui_implicit_2021}
Laurent El~Ghaoui, Fangda Gu, Bertrand Travacca, Armin Askari, and Alicia Tsai.
\newblock Implicit deep learning.
\newblock \emph{SIAM Journal on Mathematics of Data Science}, 3\penalty0
  (3):\penalty0 930--958, 2021.

\bibitem[Meulemans et~al.(2022)Meulemans, Farinha, Cervera, Sacramento, and
  Grewe]{meulemans_minimizing_2022}
Alexander Meulemans, Matilde~Tristany Farinha, Maria~R. Cervera, João
  Sacramento, and Benjamin~F. Grewe.
\newblock Minimizing control for credit assignment with strong feedback.
\newblock \emph{arXiv preprint arXiv:2204.07249}, 2022.

\bibitem[Werbos(1974)]{werbos_beyond_1974}
Paul Werbos.
\newblock \emph{Beyond regression: new tools for prediction and analysis in the
  behavioral sciences}.
\newblock Ph.{D}. thesis, Harvard University, 1974.

\bibitem[Rumelhart et~al.(1986)Rumelhart, Hinton, and
  Williams]{rumelhart_learning_1986}
David~E. Rumelhart, Geoffrey~E. Hinton, and Ronald~J. Williams.
\newblock Learning representations by back-propagating errors.
\newblock \emph{Nature}, 323\penalty0 (6088):\penalty0 533--536, 1986.

\bibitem[Almeida(1989)]{almeida_backpropagation_1989}
Luís~B. Almeida.
\newblock Backpropagation in perceptrons with feedback.
\newblock In Rolf Eckmiller and Christoph v.d. Malsburg, editors, \emph{Neural
  {Computers}}, pages 199--208. Springer Berlin Heidelberg, 1989.

\bibitem[Pineda(1989)]{pineda_recurrent_1989}
Fernando~J. Pineda.
\newblock Recurrent backpropagation and the dynamical approach to adaptive
  neural computation.
\newblock \emph{Neural Computation}, 1\penalty0 (2):\penalty0 161--172, 1989.

\bibitem[Zucchet et~al.(2022)Zucchet, Schug, von Oswald, Zhao, and
  Sacramento]{zucchet_contrastive_2022}
Nicolas Zucchet, Simon Schug, Johannes von Oswald, Dominic Zhao, and João
  Sacramento.
\newblock A contrastive rule for meta-learning.
\newblock \emph{Advances in Neural Information Processing Systems}, 2022.

\bibitem[Whittington and Bogacz(2017)]{whittington_approximation_2017}
James C.~R. Whittington and Rafal Bogacz.
\newblock An approximation of the error backpropagation algorithm in a
  predictive coding network with local {Hebbian} synaptic plasticity.
\newblock \emph{Neural Computation}, 29\penalty0 (5):\penalty0 1229--1262,
  2017.

\bibitem[Friston(2009)]{friston_free-energy_2009}
Karl Friston.
\newblock The free-energy principle: a rough guide to the brain?
\newblock \emph{Trends in Cognitive Sciences}, 13\penalty0 (7):\penalty0
  293--301, 2009.

\bibitem[Carreira-Perpinan and Wang(2014)]{carreira-perpinan_distributed_2014}
Miguel Carreira-Perpinan and Weiran Wang.
\newblock Distributed optimization of deeply nested systems.
\newblock In \emph{Artificial {Intelligence} and {Statistics}}, 2014.

\bibitem[Dold et~al.(2019)Dold, Kungl, Sacramento, Petrovici, Schindler, Binas,
  Bengio, and Senn]{dold_lagrangian_2019}
Dominik Dold, Akos~F. Kungl, João Sacramento, Mihai~A. Petrovici, Kaspar
  Schindler, Jonathan Binas, Yoshua Bengio, and Walter Senn.
\newblock Lagrangian dynamics of dendritic microcircuits enables real-time
  backpropagation of errors.
\newblock In \emph{Computational and {Systems} {Neuroscience} ({Cosyne})},
  2019.

\bibitem[Scellier(2021)]{scellier_deep_2021}
Benjamin Scellier.
\newblock \emph{A deep learning theory for neural networks grounded in
  physics}.
\newblock {PhD} {Thesis}, Université de Montréal, 2021.

\bibitem[Guerguiev et~al.(2017)Guerguiev, Lillicrap, and
  Richards]{guerguiev_towards_2017}
Jordan Guerguiev, Timothy~P. Lillicrap, and Blake~A. Richards.
\newblock Towards deep learning with segregated dendrites.
\newblock \emph{eLife}, 6:\penalty0 e22901, 2017.

\bibitem[Akrout et~al.(2019)Akrout, Wilson, Humphreys, Lillicrap, and
  Tweed]{akrout_deep_2019}
Mohamed Akrout, Collin Wilson, Peter Humphreys, Timothy~P. Lillicrap, and
  Douglas~B. Tweed.
\newblock Deep learning without weight transport.
\newblock In \emph{Advances in {Neural} {Information} {Processing} {Systems}},
  2019.

\bibitem[Meulemans et~al.(2020)Meulemans, Carzaniga, Suykens, Sacramento, and
  Grewe]{meulemans_theoretical_2020}
Alexander Meulemans, Francesco Carzaniga, Johan Suykens, João Sacramento, and
  Benjamin~F. Grewe.
\newblock A theoretical framework for target propagation.
\newblock In \emph{Advances in {Neural} {Information} {Processing} {Systems}},
  2020.

\bibitem[Podlaski and Machens(2020)]{podlaski_biological_2020}
William~F. Podlaski and Christian~K. Machens.
\newblock Biological credit assignment through dynamic inversion of feedforward
  networks.
\newblock In \emph{Advances in {Neural} {Information} {Processing} {Systems}},
  2020.

\bibitem[Pozzi et~al.(2020)Pozzi, Bohte, and
  Roelfsema]{pozzi_attention-gated_2020}
Isabella Pozzi, Sander Bohte, and Pieter Roelfsema.
\newblock Attention-gated brain propagation: how the brain can implement
  reward-based error backpropagation.
\newblock \emph{Advances in Neural Information Processing Systems}, pages
  2516--2526, 2020.

\bibitem[Gilbert and Li(2013)]{gilbert_top-down_2013}
Charles~D. Gilbert and Wu~Li.
\newblock Top-down influences on visual processing.
\newblock \emph{Nature Reviews Neuroscience}, 14\penalty0 (5):\penalty0
  350--363, 2013.

\bibitem[Manita et~al.(2015)Manita, Suzuki, Homma, Matsumoto, Odagawa, Yamada,
  Ota, Matsubara, Inutsuka, Sato, Ohkura, Yamanaka, Yanagawa, Nakai, Hayashi,
  Larkum, and Murayama]{manita_top-down_2015}
Satoshi Manita, Takayuki Suzuki, Chihiro Homma, Takashi Matsumoto, Maya
  Odagawa, Kazuyuki Yamada, Keisuke Ota, Chie Matsubara, Ayumu Inutsuka,
  Masaaki Sato, Masamichi Ohkura, Akihiro Yamanaka, Yuchio Yanagawa, Junichi
  Nakai, Yasunori Hayashi, Matthew~E. Larkum, and Masanori Murayama.
\newblock A top-down cortical circuit for accurate sensory perception.
\newblock \emph{Neuron}, 86\penalty0 (5):\penalty0 1304--1316, 2015.

\bibitem[Marques et~al.(2018)Marques, Nguyen, Fioreze, and
  Petreanu]{marques_functional_2018}
Tiago Marques, Julia Nguyen, Gabriela Fioreze, and Leopoldo Petreanu.
\newblock The functional organization of cortical feedback inputs to primary
  visual cortex.
\newblock \emph{Nature Neuroscience}, 21\penalty0 (5):\penalty0 757--764, 2018.

\bibitem[Kirchberger et~al.(2021)Kirchberger, Mukherjee, Schnabel, van Beest,
  Barsegyan, Levelt, Heimel, Lorteije, van~der Togt, Self, and
  Roelfsema]{kirchberger_essential_2021}
Lisa Kirchberger, Sreedeep Mukherjee, Ulf~H. Schnabel, Enny~H. van Beest, Areg
  Barsegyan, Christiaan~N. Levelt, J.~Alexander Heimel, Jeannette A.~M.
  Lorteije, Chris van~der Togt, Matthew~W. Self, and Pieter~R. Roelfsema.
\newblock The essential role of recurrent processing for figure-ground
  perception in mice.
\newblock \emph{Science Advances}, 7\penalty0 (27):\penalty0 eabe1833, 2021.

\bibitem[Nøkland(2016)]{nokland_direct_2016}
Arild Nøkland.
\newblock Direct feedback alignment provides learning in deep neural networks.
\newblock In \emph{Advances in {Neural} {Information} {Processing} {Systems}},
  2016.

\bibitem[Meulemans et~al.(2021)Meulemans, Tristany~Farinha, Garcia~Ordonez,
  Vilimelis~Aceituno, Sacramento, and Grewe]{meulemans_credit_2021}
Alexander Meulemans, Matilde Tristany~Farinha, Javier Garcia~Ordonez, Pau
  Vilimelis~Aceituno, João Sacramento, and Benjamin~F. Grewe.
\newblock Credit assignment in neural networks through deep feedback control.
\newblock In \emph{Advances in {Neural} {Information} {Processing} {Systems}},
  2021.

\bibitem[LeCun(1998)]{lecun_mnist_1998}
Yann LeCun.
\newblock The {MNIST} database of handwritten digits.
\newblock \emph{Available at http://yann. lecun. com/exdb/mnist}, 1998.

\bibitem[Bartunov et~al.(2018)Bartunov, Santoro, Richards, Marris, Hinton, and
  Lillicrap]{bartunov_assessing_2018}
Sergey Bartunov, Adam Santoro, Blake Richards, Luke Marris, Geoffrey~E. Hinton,
  and Timothy~P. Lillicrap.
\newblock Assessing the scalability of biologically-motivated deep learning
  algorithms and architectures.
\newblock In \emph{Advances in neural information processing systems}, 2018.

\bibitem[Kolen and Pollack(1994)]{kolen_back-propagation_1994}
John~F. Kolen and Jordan~B. Pollack.
\newblock Back-propagation without weight transport.
\newblock In \emph{Proceedings of 1994 {IEEE} {International} {Conference} on
  {Neural} {Networks} ({ICNN}’94)}, 1994.

\bibitem[Mikulasch et~al.(2022)Mikulasch, Rudelt, Wibral, and
  Priesemann]{mikulasch_dendritic_2022}
Fabian~A. Mikulasch, Lucas Rudelt, Michael Wibral, and Viola Priesemann.
\newblock Dendritic predictive coding: {A} theory of cortical computation with
  spiking neurons.
\newblock \emph{arXiv preprint arXiv:2205.05303}, 2022.

\bibitem[Rao and Ballard(1999)]{rao_predictive_1999}
Rajesh P.~N. Rao and Dana~H. Ballard.
\newblock Predictive coding in the visual cortex: a functional interpretation
  of some extra-classical receptive-field effects.
\newblock \emph{Nature Neuroscience}, 2\penalty0 (1):\penalty0 79--87, 1999.

\bibitem[Bastos et~al.(2012)Bastos, Usrey, Adams, Mangun, Fries, and
  Friston]{bastos_canonical_2012}
Andre~M. Bastos, W.~Martin Usrey, Rick~A. Adams, George~R. Mangun, Pascal
  Fries, and Karl~J. Friston.
\newblock Canonical microcircuits for predictive coding.
\newblock \emph{Neuron}, 76\penalty0 (4):\penalty0 695--711, 2012.

\bibitem[Keller and Mrsic-Flogel(2018)]{keller_predictive_2018}
Georg~B. Keller and Thomas~D. Mrsic-Flogel.
\newblock Predictive processing: a canonical cortical computation.
\newblock \emph{Neuron}, 100\penalty0 (2):\penalty0 424--435, 2018.

\bibitem[Fathony et~al.(2020)Fathony, Sahu, Willmott, and
  Kolter]{fathony_multiplicative_2020}
Rizal Fathony, Anit~Kumar Sahu, Devin Willmott, and J.~Zico Kolter.
\newblock Multiplicative filter networks.
\newblock September 2020.

\bibitem[Krizhevsky and Hinton(2009)]{krizhevsky_learning_2009}
Alex Krizhevsky and Geoffrey Hinton.
\newblock Learning multiple layers of features from tiny images.
\newblock Technical report, 2009.

\bibitem[He et~al.(2016)He, Zhang, Ren, and Sun]{he_deep_2016}
Kaiming He, Xiangyu Zhang, Shaoqing Ren, and Jian Sun.
\newblock Deep residual learning for image recognition.
\newblock In \emph{Proceedings of the {IEEE} {Conference} on {Computer}
  {Vision} and {Pattern} {Recognition}}, 2016.

\bibitem[Chen et~al.(2018)Chen, Rubanova, Bettencourt, and
  Duvenaud]{chen_neural_2018}
Ricky T.~Q. Chen, Yulia Rubanova, Jesse Bettencourt, and David~K. Duvenaud.
\newblock Neural ordinary differential equations.
\newblock In \emph{Advances in {Neural} {Information} {Processing} {Systems}},
  2018.

\bibitem[Sitzmann et~al.(2020)Sitzmann, Martel, Bergman, Lindell, and
  Wetzstein]{sitzmann_implicit_2020}
Vincent Sitzmann, Julien Martel, Alexander Bergman, David Lindell, and Gordon
  Wetzstein.
\newblock Implicit neural representations with periodic activation functions.
\newblock In \emph{Advances in {Neural} {Information} {Processing} {Systems}},
  2020.

\bibitem[Thrun and Pratt(1998)]{thrun_learning_1998}
Sebastian Thrun and Lorien Pratt.
\newblock \emph{Learning to learn}.
\newblock Springer US, 1998.

\bibitem[Behrens et~al.(2018)Behrens, Muller, Whittington, Mark, Baram,
  Stachenfeld, and Kurth-Nelson]{behrens_what_2018}
Timothy E.~J. Behrens, Timothy~H Muller, James C.~R. Whittington, Shirley Mark,
  Alon~B. Baram, Kimberly~L Stachenfeld, and Zeb Kurth-Nelson.
\newblock What is a cognitive map? {Organizing} knowledge for flexible
  behavior.
\newblock \emph{Neuron}, 100\penalty0 (2):\penalty0 490--509, 2018.

\bibitem[Wang et~al.(2018)Wang, Kurth-Nelson, Kumaran, Tirumala, Soyer, Leibo,
  Hassabis, and Botvinick]{wang_prefrontal_2018}
Jane~X. Wang, Zeb Kurth-Nelson, Dharshan Kumaran, Dhruva Tirumala, Hubert
  Soyer, Joel~Z. Leibo, Demis Hassabis, and Matthew Botvinick.
\newblock Prefrontal cortex as a meta-reinforcement learning system.
\newblock \emph{Nature Neuroscience}, 21\penalty0 (6):\penalty0 860--868, 2018.

\bibitem[Wang(2021)]{wang_meta-learning_2021}
Jane~X. Wang.
\newblock Meta-learning in natural and artificial intelligence.
\newblock \emph{Current Opinion in Behavioral Sciences}, 38:\penalty0 90--95,
  2021.

\bibitem[Rajeswaran et~al.(2019)Rajeswaran, Finn, Kakade, and
  Levine]{rajeswaran_meta-learning_2019}
Aravind Rajeswaran, Chelsea Finn, Sham Kakade, and Sergey Levine.
\newblock Meta-learning with implicit gradients.
\newblock In \emph{Advances in {Neural} {Information} {Processing} {Systems}},
  2019.

\bibitem[Fusi et~al.(2005)Fusi, Drew, and Abbott]{fusi_cascade_2005}
Stefano Fusi, Patrick~J. Drew, and Larry~F. Abbott.
\newblock Cascade models of synaptically stored memories.
\newblock \emph{Neuron}, 45\penalty0 (4):\penalty0 599--611, 2005.

\bibitem[Ziegler et~al.(2015)Ziegler, Zenke, Kastner, and
  Gerstner]{ziegler_synaptic_2015}
Lorric Ziegler, Friedemann Zenke, David~B. Kastner, and Wulfram Gerstner.
\newblock Synaptic consolidation: from synapses to behavioral modeling.
\newblock \emph{Journal of Neuroscience}, 35\penalty0 (3):\penalty0 1319--1334,
  2015.

\bibitem[Benna and Fusi(2016)]{benna_computational_2016}
Marcus~K. Benna and Stefano Fusi.
\newblock Computational principles of synaptic memory consolidation.
\newblock \emph{Nature Neuroscience}, 19\penalty0 (12):\penalty0 1697--1706,
  2016.

\bibitem[Zenke et~al.(2017)Zenke, Poole, and Ganguli]{zenke_continual_2017}
Friedemann Zenke, Ben Poole, and Surya Ganguli.
\newblock Continual learning through synaptic intelligence.
\newblock In \emph{International {Conference} on {Machine} {Learning}}, 2017.

\bibitem[Kirkpatrick et~al.(2017)Kirkpatrick, Pascanu, Rabinowitz, Veness,
  Desjardins, Rusu, Milan, Quan, Ramalho, Grabska-Barwinska, Hassabis, Clopath,
  Kumaran, and Hadsell]{kirkpatrick_overcoming_2017}
James Kirkpatrick, Razvan Pascanu, Neil Rabinowitz, Joel Veness, Guillaume
  Desjardins, Andrei~A. Rusu, Kieran Milan, John Quan, Tiago Ramalho, Agnieszka
  Grabska-Barwinska, Demis Hassabis, Claudia Clopath, Dharshan Kumaran, and
  Raia Hadsell.
\newblock Overcoming catastrophic forgetting in neural networks.
\newblock \emph{Proceedings of the National Academy of Sciences of the United
  States of America}, 114\penalty0 (13):\penalty0 3521--3526, 2017.

\bibitem[Finn et~al.(2017)Finn, Abbeel, and Levine]{finn_model-agnostic_2017}
Chelsea Finn, Pieter Abbeel, and Sergey Levine.
\newblock Model-agnostic meta-learning for fast adaptation of deep networks.
\newblock In \emph{International {Conference} on {Machine} {Learning}}, 2017.

\bibitem[Luketina et~al.(2016)Luketina, Berglund, Greff, and
  Raiko]{luketina_scalable_2016}
Jelena Luketina, Mathias Berglund, Klaus Greff, and Tapani Raiko.
\newblock Scalable gradient-based tuning of continuous regularization
  hyperparameters.
\newblock In \emph{International {Conference} on {Machine} {Learning}}, 2016.

\bibitem[Nichol et~al.(2018)Nichol, Achiam, and
  Schulman]{nichol_first-order_2018}
Alex Nichol, Joshua Achiam, and John Schulman.
\newblock On first-order meta-learning algorithms.
\newblock \emph{arXiv preprint arXiv:1803.02999}, 2018.

\bibitem[Lorraine et~al.(2020)Lorraine, Vicol, and
  Duvenaud]{lorraine_optimizing_2020}
Jonathan Lorraine, Paul Vicol, and David Duvenaud.
\newblock Optimizing millions of hyperparameters by implicit differentiation.
\newblock In \emph{International {Conference} on {Artificial} {Intelligence}
  and {Statistics}}, 2020.

\bibitem[Lake et~al.(2015)Lake, Salakhutdinov, and
  Tenenbaum]{lake_human-level_2015}
Brenden~M. Lake, Ruslan Salakhutdinov, and Joshua~B. Tenenbaum.
\newblock Human-level concept learning through probabilistic program induction.
\newblock \emph{Science}, 350\penalty0 (6266):\penalty0 1332--1338, 2015.

\bibitem[Bryson and Ho(1969)]{bryson_applied_1969}
A.~E. Bryson and Y.~C. Ho.
\newblock \emph{Applied optimal control: optimization, estimation, and
  control}.
\newblock Blaisdell Pub. Co., 1969.

\bibitem[Baydin et~al.(2018)Baydin, Pearlmutter, Radul, and
  Siskind]{baydin_automatic_2018}
Atilim~Gunes Baydin, Barak~A. Pearlmutter, Alexey~Andreyevich Radul, and
  Jeffrey~Mark Siskind.
\newblock Automatic differentiation in machine learning: a survey.
\newblock \emph{Journal of Marchine Learning Research}, 18:\penalty0 1--43,
  2018.

\bibitem[LeCun(1988)]{lecun_theoretical_1988}
Yann LeCun.
\newblock A theoretical framework for back-propagation.
\newblock In \emph{Proceedings of the 1998 {Connectionist} {Models} {Summer}
  {School}}, 1988.

\bibitem[Gilra and Gerstner(2017)]{gilra_predicting_2017}
Aditya Gilra and Wulfram Gerstner.
\newblock Predicting non-linear dynamics by stable local learning in a
  recurrent spiking neural network.
\newblock \emph{eLife}, 6:\penalty0 e28295, 2017.

\bibitem[Denève et~al.(2017)Denève, Alemi, and Bourdoukan]{deneve_brain_2017}
Sophie Denève, Alireza Alemi, and Ralph Bourdoukan.
\newblock The brain as an efficient and robust adaptive learner.
\newblock \emph{Neuron}, 94\penalty0 (5):\penalty0 969--977, 2017.

\bibitem[Alemi et~al.(2018)Alemi, Machens, Deneve, and
  Slotine]{alemi_learning_2018}
Alireza Alemi, Christian Machens, Sophie Deneve, and Jean-Jacques Slotine.
\newblock Learning nonlinear dynamics in efficient, balanced spiking networks
  using local plasticity rules.
\newblock In \emph{Proceedings of the {AAAI} {Conference} on {Artificial}
  {Intelligence}}, 2018.

\bibitem[Rusakov et~al.(2020)Rusakov, Savtchenko, and
  Latham]{rusakov_noisy_2020}
Dmitri~A. Rusakov, Leonid~P. Savtchenko, and Peter~E. Latham.
\newblock Noisy synaptic conductance: {Bug} or a feature?
\newblock \emph{Trends in Neurosciences}, 43\penalty0 (6):\penalty0 363--372,
  2020.

\bibitem[Hopfield(1984)]{hopfield_neurons_1984}
J.~J. Hopfield.
\newblock Neurons with graded response have collective computational properties
  like those of two-state neurons.
\newblock \emph{Proceedings of the National Academy of Sciences}, 81\penalty0
  (10):\penalty0 3088--3092, 1984.

\bibitem[Cohen and Grossberg(1983)]{cohen_absolute_1983}
Michael~A. Cohen and Stephen Grossberg.
\newblock Absolute stability of global pattern formation and parallel memory
  storage by competitive neural networks.
\newblock \emph{IEEE Transactions on Systems, Man, and Cybernetics},
  SMC-13\penalty0 (5):\penalty0 815--826, 1983.

\bibitem[Kingma and Ba(2015)]{kingma_adam_2015}
Diederik~P. Kingma and Jimmy Ba.
\newblock Adam: {A} method for stochastic optimization.
\newblock In \emph{International {Conference} on {Learning} {Representations}},
  2015.

\end{thebibliography}

\newpage


\appendix
\setcounter{page}{1}
\setcounter{figure}{0} \renewcommand{\thefigure}{S\arabic{figure}}
\setcounter{theorem}{0} \renewcommand{\thetheorem}{S\arabic{theorem}}
\setcounter{remark}{0} \renewcommand{\theremark}{S\arabic{remark}}
\setcounter{definition}{0} \renewcommand{\thedefinition}{S\arabic{definition}}
\setcounter{section}{0}
\renewcommand{\thesection}{S\arabic{section}}
\setcounter{table}{0} \renewcommand{\thetable}{S\arabic{table}}

\renewcommand \thepart{}
\renewcommand \partname{}
\renewcommand \thepart{Supplementary Materials}
\addcontentsline{toc}{section}{} 
\part{} 
\textbf{Alexander Meulemans*, Nicolas Zucchet*, Seijin Kobayashi*, Johannes von Oswald,\\João Sacramento}
\parttoc 

\newpage
\section{Notation}\label{sec_app:notation}

In this section, we provide an extensive list of the notations we use. In particular, we detail what the different quantities precisely mean for learning recurrent neural networks and for meta-learning.

\subsection{General notations}
\begin{tabularx}{\textwidth}{lL}
    \toprule
    Notation & Meaning\\
    \midrule
    $\pder{\bx}{} = \partial_\bx$ & Partial derivative of a function. The derivative of a scalar function is a row vector of size $|x|$.\\
    $\der{\bx}{} = \mathrm{d}_\bx$ & Total derivative of a function. The derivative of a scalar function is a row vector of size $|\bx|$.\\
    $\nabla_\bx$ & Gradient of a scalar function (column vector). We have $\nabla_\bx \,\cdot = (\mathrm{d}_\bx \,\cdot)^\top\!$. \\
    $\dot{\bx}$ & Time derivative of $\bx$ (column vector).\\
    $\snorm{\,\cdot\,}$ & Euclidean (L2) norm.\\
    $\min_\bx$ & Minimum of a function with respect to $\bx$.\\
    $\argmin_\bx$ & Set of $\bx$ that (locally) minimize a function.\\
    $\Id$ & Identity matrix (size defined by the context).\\ 
    $\mathrm{col}[A]$ & Column space of $A$, that is $\mathrm{Im}[A]$.\\
    $\mathrm{row}[A]$ & Row space of $A$, that is $\mathrm{Im}[A^\top]$.\\
    \bottomrule
\end{tabularx}

\subsection{Notations for the least-control principle in general}
\begin{tabularx}{\textwidth}{lL}
    \toprule
    Notation & Meaning\\
    \midrule
    $\bphi$ & Dynamical parameters of the system. \\
    $\btheta$ & Parameters of the system. \\
    $f(\bphi, \btheta)$ & Vector-field of the (free) dynamics on $\bphi$. \\
    $L$ & Learning loss that the system should minimize. \\
    $\by=h(\bphi)$ & Output units (i.e. units on which the loss is evaluated) of the system.\\
    $\HH$ & Least-control objective. \\
    $\psi$ & Controller whose goal is to help the system reaching a loss-minimizing state. \\
    $\bu$ & Controller on the output units. \\
    $Q(\bphi, \btheta)$ & General feedback mapping from output control $\bu$ to the entire system control $\bpsi$. We have $\bpsi = Q(\bphi, \btheta)\bu$ at equilibrium. \\
    $\bphi\fss$ & Equilibrium of the free dynamics.\\
    $\bphi\opt$, $\bpsi\opt$, $\bu\opt$ & Optimal state, control and output control according for the least-control objective (c.f. Eq.~\ref{eqn:least_control_problem}). \\
    $\bphi\css$, $\bpsi\css$, $\bu\css$ & Equilibrium of the controlled dynamics. The least-control principle aims for $\bpsi\css = \bpsi\opt$, that is to find an optimal control through the steady state of controlled dynamics.\\
    $\alpha$ & Leakage strength in a leaky integral controller ($\alpha = 0$ yields pure integral control).\\
    $\beta$ & Strength of the proportional control, or of nudging in the energy-based view. \\
    \bottomrule
\end{tabularx}

\subsection{Notations for recurrent neural networks}

\begin{tabularx}{\textwidth}{lL}
    \toprule
    Notation & Meaning\\
    \midrule
    $\phi$ & Neural activities of all the neurons in the network (column vector). \\
    $\btheta = \{W, S\}$ & Learnable synaptic connection weights of the network. \\
    $f(\bphi, \btheta)$ & Dynamics of the neural network. \\
    $\by = D\bphi$ & The binary matrix $D$ selects output neurons $\bu$ from the rest of them $\bphi$.\\
    $\sigma$ & Non-linear activation function. \\
    $x$ & Input of the neural network. \\
    $y^\mathrm{true}$ & Desired output corresponding to an input $x$. \\
    $L$ & Loss measuring the discrepancy between outputs of the network $y$ (that depends on $x$) and target output $y^\mathrm{true}$, averaged over all the pairs $(x, y^\mathrm{true})$ from a data set.\\
    $\gamma$ & Weight decay for the Kollen-Pollack learning-rule. \\
    \bottomrule
\end{tabularx}

\subsection{Notations for meta-learning}
\begin{tabularx}{\textwidth}{lL}
    \toprule
    Notation & Meaning\\
    \midrule
    $\bphi$ & Parameters (weights and biases) of the neural network that is learned.\\
    $\btheta = \{\bomega, \blambda \}$ & Parameters of the learning algorithm, here the consolidated synaptic weights $\bomega$, and eventually, the strength of attraction to those weights $\blambda$.\\
    $f(\bphi, \btheta)$ & Dynamics of the learning algorithm. \\
    $\tau$ & Index of the task considered.\\
    $L^\mathrm{learn}_\tau(\bphi) $ & Data-driven loss that is combined with a prior-term to obtain the loss minimized by the learning algorithm. \\
    $L^\mathrm{eval}_\tau(\bphi)$ & Loss measuring the performance of the output of the learning algorithm.\\
    $\bpsi = \bu$ & Controller on all the units (they are all output units in this context). \\
    \bottomrule
\end{tabularx}

\newpage
\section{Proofs for the least-control principle}
\label{sec_app:proofs_LCP}

The purpose of this section is to prove all the theoretical results behind the least-control principle.  Recall that our principle advocates to solve the following constrained optimization problem
\begin{equation}
    \label{eqn_app:LCP}
    \min_{\btheta} \min_{\bphi, \bpsi} \, \frac{1}{2}\snorm{\bpsi}^2 \quad\mathrm{s.t.} \enspace \ff + \bpsi = 0,\enspace \pder{\by}{L}(h(\bphi)) = 0,
\end{equation}
that is minimizing the amount of control needed to reach a loss-minimizing state. Finding an optimal control $\bpsi\opt$ consists in solving the least-control problem:
\begin{equation}
    \label{eqn_app:constrained_optim}
    (\bphi\opt, \bpsi\opt) = \argmin_{\bphi, \bpsi} \frac{1}{2}\snorm{\bpsi}^2 \quad\mathrm{s.t.} \enspace \ff + \bpsi = 0,\enspace \pder{\by}{L}(h(\bphi)) = 0.
\end{equation}
Most of the proofs that we present here rely on interpreting the problem above as a constrained minimization problem, and using the associated first-order stationarity condition to manipulate $\bpsi\opt$ and $\bphi\opt$. Compared to the main text, we introduced the notation $y=h(\phi)$ in this formulation. It is here to select output units from the rest and clarifies mathematical manipulations.

\subsection{Constrained optimization and Lagrange multipliers}
\label{sec_app:lagrange_multipliers}

Throughout the proofs for the least-control principle, we characterize the optimal control $\bpsi\opt$ and state $\bphi\opt$ through the stationarity conditions that they verify, that are known as the KKT conditions. In particular, this allows considering $(\bphi\opt, \bpsi\opt)$ as an implicit function of $\theta$, and computing derivatives through the implicit function theorem \citeS{dontchev_implicit_2009}. We provide a quick review on constrained optimization in Section \ref{sec_app:review_constrained_optim} for the reader who is new to the topic.

We now state what are the KKT conditions for the constrained optimization problem \eqref{eqn_app:constrained_optim}. To do so, we introduce the corresponding Lagrangian, which we denote as the LCP-Lagrangian:
\begin{equation}
    \LCPLag := \frac{1}{2} \snorm{\bpsi}^2 + \blambda^\top\left(\ff + \bpsi \right) + \pder{\by}{L}(h(\bphi)) \bmu,
\end{equation}
where $\blambda$ and $\bmu$ are the Lagrange multipliers associated to the constraints $f(\bphi, \btheta) + \bpsi=0$ and $\nabla_{\by}L(h(\bphi))=0$. Then, $(\bphi\opt, \bpsi\opt)$ verifies the KKT conditions associated to Eq.~\ref{eqn_app:constrained_optim} if there exists $(\lambda\opt, \mu\opt)$ such that
\begin{equation}
    \label{eqn_app:KKT_LCP}
    \left\{
    \begin{split}
        \pder{\bphi}{\calL}(\bphi\opt, \bpsi\opt, \blambda\opt, \bmu\opt, \btheta) &= {\blambda\opt}^\top \pder{\bphi}{f}(\bphi\opt, \btheta) + {\bmu\opt}^\top \pder{\by^2}{^2L}(h(\bphi\opt)) \pder{\bphi}{h}(\bphi\opt) = 0 \\
        \pder{\bpsi}{\calL}(\bphi\opt, \bpsi\opt, \blambda\opt, \bmu\opt, \btheta) &= {\bpsi\opt}^\top + {\blambda\opt}^\top = 0 \\
        \pder{\blambda}{\calL}(\bphi\opt, \bpsi\opt, \blambda\opt, \bmu\opt, \btheta) &= f(\bphi\opt, \btheta)^\top + {\bpsi\opt}^\top = 0 \\
        \pder{\bmu}{\calL}(\bphi\opt, \bpsi\opt, \blambda\opt, \bmu\opt, \btheta) &= \pder{\by}{L}(h(\bphi\opt)) = 0.
    \end{split}
    \right.
\end{equation}

In the proofs that follow, we use equivalently that $(\bphi\opt, \bpsi\opt)$ is optimal for Eq.~\ref{eqn_app:constrained_optim} and that it verifies the KKT conditions of Eq.~\ref{eqn_app:KKT_LCP}. We need to invoke sufficient conditions, such as the positive definiteness of the Hessian of the Lagrangian, to make it rigorous and to show that a state satisfying the KKT conditions is a local minimizer; we discuss this matter in more details in Section \ref{sec_app:review_constrained_optim}. However, we omit these considerations in the proofs to keep the arguments as concise as possible.

\subsection{Theorem \ref{theorem:first_order_updates}: first-order gradient}
\label{sec_app:first_order_gradients}

We first state the formal version of Theorem \ref{theorem:first_order_updates} and then prove it, by differentiating through the KKT conditions \eqref{eqn_app:KKT_LCP}. Note that we provide an alternate proof that leverages the energy-based formulation of the principle in Section~\ref{sec_app:energy_based_version}.
\begin{theorem}[Formal] 
    \label{theorem_app:first_order_updates}
    Let $(\bphi\opt, \bpsi\opt)$ an optimal control for the least-control problem \eqref{eqn_app:constrained_optim} and $(\blambda\opt, \bmu\opt)$ the Lagrange multipliers for which the KKT conditions are satisfied. Under the assumption that the Hessian of the LCP-Lagrangian $\partial^2_{\bphi, \bpsi, \blambda, \bmu}\LCPLagOpt$ is invertible, the least-control principle yields the following gradient for $\btheta$:
    \begin{align}\label{eqn_app:theta_lcp_gradient}
        \left ( \der{\btheta}{}\HH(\btheta) \right )^\top &= -\pder{\btheta}{f}(\bphi\opt, \btheta)^\top \bpsi\opt = \pder{\btheta}{f}(\bphi\opt, \btheta)^\top f(\bphi\opt, \btheta).
    \end{align}
\end{theorem}

\begin{proof}
    The proof relies on two ingredients: rewriting $\HH$ using the LCP-Lagrangian and then differentiating through this Lagrangian.
    
    We start by rewriting the least-control objective $\HH(\btheta)$ using the LCP-Lagrangian. Recall that $\HH$ is defined as
    \begin{equation}
        \HH(\btheta) = \frac{1}{2}\snorm{\bpsi\opt}^2.
    \end{equation}
    As $(\bphi\opt, \bpsi\opt)$ is an optimal control, there indeed exists Lagrange multipliers $(\blambda\opt, \bmu\opt)$ such that $\partial_{\bphi, \bpsi, \blambda, \bmu}\LCPLagOpt = 0$, as assumed in the statement of the theorem. We use this equality condition to implicitly define $(\bphi\opt, \bpsi\opt, \blambda\opt, \bmu\opt)$ as functions of $\btheta$. The implicit function theorem along with the assumption that the Hessian is invertible ensures that these functions are well defined and differentiable. As all the constraints are satisfied for $(\bphi\opt, \bpsi\opt)$, we have
    \begin{equation}
        \begin{split}
            \HH(\btheta) &= \LCPLagOpt\\
            &= \frac{1}{2} \snorm{\bpsi\opt}^2 + {\blambda\opt}^\top\left(f(\bphi\opt, \btheta) + \bpsi\opt \right) + \pder{\by}{L}(h(\bphi\opt)) \bmu\opt \\
            &= \frac{1}{2} \snorm{\bpsi\opt}^2
        \end{split}
    \end{equation}
    
    We can then calculate $\sder{\btheta}{}\HH(\btheta)$:
    \begin{equation}
        \begin{split}
            \der{\btheta}{}\HH(\btheta) &= \der{\btheta}{}\LCPLagOpt\\
            &= \pder{\btheta}{\calL} + \pder{\bphi}{\calL} \der{\btheta}{\bphi\opt} + \pder{\bpsi}{\calL} \der{\btheta}{\bpsi\opt} + \pder{\blambda}{\calL} \der{\btheta}{\blambda\opt} + \pder{\bmu}{\calL} \der{\btheta}{\bmu\opt}\\
            &= \pder{\btheta}{\calL} + 0 + 0 + 0 + 0\\
            &= {\blambda\opt}^\top \pder{\btheta}{f}(\bphi\opt, \btheta),
        \end{split}
    \end{equation}
    where we used the chain rule for the second equation, and the stationarity equations of the Lagrangian for the third equation.
    Through the stationarity equation $\partial_{\bpsi}\LCPLagOpt = 0$, we have that $\blambda\opt = -\bpsi\opt$. The constraint on $\bpsi$ additionally implies that $\bpsi\opt=-f(\bphi\opt, \btheta)$, which finishes the proof.
\end{proof}

\textbf{Extension to $\btheta$-dependent losses.} Throughout this work, we focus on loss functions $L(\bphi)$ that do not explicitly depend on $\btheta$. The influence of $\btheta$ on the loss only appears through the equilibrium condition $\ff = 0$. Still, we can generalize Theorem~\ref{theorem_app:first_order_updates} further by also considering loss functions that explicitly depend on $\btheta$, on top of their implicit dependence on $\btheta$ through $\bphi$. A simple example is when we add a weight regularization term $\snorm{\btheta}^2$ to the loss $L$ when training neural networks, or in meta-learning when $\btheta$ does not only impact the learning algorithm of the system but also its prediction \citeS{zintgraf_fast_2019,lee_meta-learning_2019,zhao_meta-learning_2020}. 

The proof of Theorem~\ref{theorem_app:first_order_updates} can be slightly adapted to hold for the more general case, the only difference is that the equality $\spder{\btheta}{\calL} = \blambda_*^\top \spder{\btheta}{f}(\bphi_*, \btheta)$ is no longer true. Instead, we have
\begin{equation}
    \begin{split}
        \der{\btheta}{}\HH(\btheta) &= \pder{\btheta}{\calL}(\bphi\opt, \bpsi\opt, \blambda\opt, \bmu\opt, \btheta) \\
        &= -{\bpsi\opt}^\top \pder{\btheta}{f}(\bphi\opt, \btheta) + {\bmu\opt}^\top \pder{\btheta \partial \by}{^2L}(\bphi\opt, \btheta).
    \end{split}
\end{equation}
The second derivative $\partial_{\btheta}\partial_{\by}L(h(\bphi\opt), \btheta)$ captures mixed dependencies of the loss $L$ on $\by$ and $\btheta$. Note that when the loss $L$ contains no terms mixing $h(\bphi)$ and $\btheta$, as in the weight decay example we mentioned earlier, this second-order derivative is zero. As for this weight decay example the resulting updates are the same as for the loss without any weight decay, this indicates that some $\btheta$-dependent losses are not captured in the least-control principle we introduced in the main text.

We now propose an extended formulation of the least-control principle to fully incorporate $\btheta$-dependent losses:
\begin{equation}
    \min_{\btheta} \min_{\bphi, \bpsi} \frac{1}{2} \snorm{\bpsi}^2 + L(\bphi, \btheta) ~~~\mathrm{s.t.} ~~ \ff + \bpsi = 0, ~~ \nabla_{\bphi} L(\bphi, \btheta) = 0.
\end{equation}
Intuitively, the least-control objective $\snorm{\bpsi}^2$ combined with the constraint that the loss is minimized w.r.t. $\bphi$ at the controlled equilibrium captures all the ways $\btheta$ influences the loss through $\bphi$, and the new term $L(\bphi, \btheta)$ takes care of the ways in which $\btheta$ influences the loss without influencing $\bphi$, while keeping the same stationarity conditions for the LCP-Lagrangian (the extra term $\nabla_{\by}L(h(\bphi\opt), \btheta)$ that now appears in $\partial_{\bphi}\calL$ vanishes). Using the same proof strategy as in Theorem \ref{theorem:first_order_updates}, we obtain the following gradient:
\begin{equation}\label{eqn_app:update_theta_loss}
    \begin{split}
        \der{\btheta}{}\calH(\btheta) &= \pder{\btheta}{\calL}(\bphi\opt, \bpsi\opt, \blambda\opt, \bmu\opt, \btheta)\\
        &=  -{\bpsi\opt}^\top \pder{\btheta}{f}(\bphi\opt, \btheta) + {\bmu\opt}^\top\frac{\partial^2 L}{\partial \btheta \partial \by}(\bphi\opt, \btheta) + \pder{\btheta}{L}(\bphi\opt, \btheta).
    \end{split}
\end{equation}
This generalized least-control principle now allows to fully consider terms in the loss that depend on $\btheta$, while keeping the nice properties of being single-phased, having no need to invert any Jacobian, and using no infinitesimal learning signals. However, now the learning rule is not always first-order anymore due the mixed second-order derivative. 

One special case of interest where the resulting learning rule is first-order again is when the $\btheta$-dependence of the loss is encapsulated in the decoder $h(\bphi, \btheta)$, i.e. when we have the loss $L(h(\bphi, \btheta))$. Applying Eq. \ref{eqn_app:update_theta_loss} to this loss results in:
\begin{equation}
    \begin{split}
        \der{\btheta}{}\calH(\btheta) &=  -{\bpsi\opt}^\top \pder{\btheta}{f}(\bphi\opt, \btheta) + {\bmu\opt}^\top\pder{\by^2}{^2L}(h(\bphi\opt, \btheta))\pder{\btheta}{h}(\bphi\opt, \btheta) + \pder{\by}{L}(h(\bphi\opt, \btheta)) \pder{\btheta}{h}(\bphi\opt, \btheta). \\
        &= -{\bpsi\opt}^\top \pder{\btheta}{f}(\bphi\opt, \btheta) + {\bmu\opt}^\top\pder{\by^2}{^2L}(h(\bphi\opt, \btheta))\pder{\btheta}{h}(\bphi\opt, \btheta) + 0
    \end{split}
\end{equation}
where the last term vanishes as $\spder{\by}{L}(h(\bphi\opt, \btheta)) = 0$ is enforced by the constraints of the least-control objective of Eq. \ref{eqn:least_control_problem}. When we use the inversion dynamics \eqref{eqn:deq_inversion} or the energy-based dynamics \eqref{eqn:generalized_energy_dynamics}, we can easily access the Lagrange multiplier $\bmu\opt$. For this, let us take a look at the equilibrium equation of the inversion dynamics \eqref{eqn:deq_inversion}, which is equivalent to the equilibrium condition of the energy-based dynamics \eqref{eqn:generalized_energy_dynamics} if $\spder{\bphi}{f(\bphi\opt, \btheta)}$ is invertible: 
\begin{equation}
    0 = \left(\pder{\bphi}{\ffdown}\right)^\top \bpsi\css + \pder{\bphi}{h(\bphi\css, \btheta)^\top} \bu\css
\end{equation}
Comparing this equilibrium condition with the KKT conditions of Eq. \ref{eqn_app:KKT_LCP}, and using the relation $\blambda\opt = - \bpsi\opt$, we have that 
\begin{align}
    \frac{\partial^2 L(h(\bphi\css, \btheta))}{\partial \by^2}\bmu\opt =  -\bu\css
\end{align}
when $\spder{\bphi}{h(\bphi\css, \btheta)}$ is full rank and in the limit of $\alpha \to 0$ such that $(\bphi\css, \bpsi\css) = (\bphi\opt, \bpsi\opt)$. 
Taking everything together, we arrive at the following simple first-order update rule for models with a $\btheta$-dependent decoder $h$, in the limit of $\alpha \to 0$: 
\begin{equation}\label{eqn_app:theta_decoder_update}
    \begin{split}
        \der{\btheta}{}\calH(\btheta)
        &= -\bpsi\css^\top \pder{\btheta}{f}(\bphi\css, \btheta) - \bu\css^\top 
        \pder{\btheta}{h(\bphi\css, \btheta)}
    \end{split}
\end{equation}
Note that this update only holds when the inversion dynamics \eqref{eqn:deq_inversion} or energy-based dynamics \eqref{eqn:generalized_energy_dynamics} are used. If other dynamics are used for finding the optimal control $\bpsi\opt$, one needs to investigate case-by-case how to compute the Lagrange multiplier $\bmu\opt$. 

\subsection{Theorem \ref{theorem:columnspace}: general controller dynamics}
\label{sec_app:columnspace}

Theorem \ref{theorem:columnspace} shows that we can find optimal control $\bpsi\opt$ and state $\bphi\opt$ for Eq.~\ref{eqn_app:constrained_optim} through dynamics exhibiting the following equilibrium equations
\begin{equation}\label{eqn_app:general_controller_dynamics}
    0= f(\bphi\css^\alpha, \btheta) + Q(\bphi\css^\alpha, \btheta) \bu\css^\alpha ~~\mathrm{and}~~ 0 = -\nabla_{\by}L(h(\bphi\css^\alpha)) - \alpha \bu\css^\alpha
\end{equation}
in the limit $\alpha \rightarrow 0$, when $Q(\bphi\css^\alpha, \btheta)$ satisfies the column space condition \eqref{eqn_app:columnspace_condition}. We prove that by showing that the limit satisfies the KKT conditions, and hence is optimal. We restate Theorem \ref{theorem:columnspace} below with the full technical specifications.
\begin{theorem}[Formal]\label{theorem_app:columnspace}
    Let $(\bphi\css^\alpha, \bu\css^\alpha)$ be a steady state of the generalized dynamics satisfying \eqref{eqn_app:general_controller_dynamics} such that $\bphi\css^\alpha$ admits a finite limit $\bphi\opt$ when $\alpha$ goes to 0. Assume that $\partial_{\bphi}{f}(\bphi\opt)$ and $\partial_{\by}^2{L}(\bphi\opt)$ are invertible, $\partial_{\bphi} h(\bphi\opt)$ is full rank, $Q$ is continuous, and that the following column space condition is satisfied: 
    \begin{align}\label{eqn_app:columnspace_condition}
        \mathrm{col} \left [ Q(\bphi\opt, \btheta) \right ] = \mathrm{row} \left [ \pder{\bphi}{h}(\bphi\opt) \left(\pder{\bphi}{f}(\bphi\opt, \btheta)\right)^{-1} \right]
    \end{align}
    Then, $\bu\css^\alpha$ converges to a finite limit $\bu\opt$ and $(\bphi\opt, Q(\bphi\opt, \btheta)\bu\opt)$ verifies the KKT conditions for the least-control problem \eqref{eqn_app:constrained_optim}.
\end{theorem}

\begin{proof}
    Note that there is a slight abuse of notation in the statement of the theorem as we use the superscript $\opt$ to denote the limits of the different quantities without yet knowing if they correspond to an optimal quantity for the least-control problem \eqref{eqn_app:constrained_optim}. Part of the proof is to actually show that those limits are optimal.
    
    First, we show that $\bu\css^\alpha$ has a finite limit when $\alpha$ goes to 0.
    The function $\bphi^\alpha\css$ admits a finite limit $\bphi\opt$ when $\alpha$ goes to zero by hypothesis. Taking $\alpha \rightarrow 0$ in the equilibrium equations \eqref{eqn_app:general_controller_dynamics} results in
    \begin{align} \label{eqn_app:theorem2_eq1}
        f(\bphi\opt, \btheta) - \lim_{\alpha \rightarrow 0} \frac{1}{\alpha} Q(\bphi\css^\alpha, \btheta) \nabla_{\by} L(h(\bphi\css^\alpha)) = 0,
    \end{align}
    as $f$ is continuous. The matrices $\partial_{\bphi} h(\bphi\opt)$ and $\partial_{\bphi}f(\bphi\opt, \btheta)^{-1}$ are full rank, so the column space condition implies that $Q(\bphi\opt, \btheta)$ is also full rank. We can therefore extract $Q(\bphi\css^\alpha, \btheta)$ out of the limit in the previous equation:
    \begin{equation}
        \lim_{\alpha \rightarrow 0} \frac{1}{\alpha} Q(\bphi\css^\alpha, \btheta) \nabla_{\by} L(h(\bphi\css^\alpha)) = Q(\bphi\opt, \btheta) \lim_{\alpha \rightarrow 0} \frac{1}{\alpha} \nabla_{\by}L(h(\bphi\css^\alpha)),
    \end{equation}
    so that $\bu\css^\alpha = -\alpha^{-1}\nabla_{\by} L(h(\bphi\css^\alpha))$ has a finite limit that we note $\bu\opt$, using once again the fact that $Q(\bphi\opt, \btheta)$ is full rank and \eqref{eqn_app:theorem2_eq1}. In particular, this implies that $\bphi\opt$ is feasible as $\nabla_{\by}L(h(\bphi\css^\alpha)) = O(\alpha)$ so $\nabla_{\by}L(h(\bphi\opt))=0$.
    
    We can now prove that the KKT conditions are satisfied at $(\bphi\opt, \bpsi\opt)$ with $\bpsi\opt := Q(\bphi\opt, \btheta)\bu\opt$. Note that we have $f(\bphi\opt, \btheta) + \bpsi\opt=0$ by taking the limit $\alpha \rightarrow 0$ in the equilibrium equation
    \begin{equation}
        f(\bphi\css^\alpha, \btheta) - \frac{1}{\alpha} Q(\bphi\css^\alpha, \btheta) \nabla_{\by} L(h(\bphi\css^\alpha)) = 0.
    \end{equation}
    As $\mathrm{col}\left [ Q(\bphi\opt, \btheta) \right ] = \mathrm{row} \left [ \partial_{\bphi} h(\bphi\opt) \partial_{\bphi}f(\bphi\opt, \btheta)^{-1} \right]$, and $\partial^2_{\by} L(h(\bphi\opt))$ is invertible, there exist a $\mu$ and $\mu\opt$, such that
    \begin{equation}
        \begin{split}
            \bpsi\opt &= Q(\bphi\opt, \btheta)\bu\opt\\
            &= \pder{\bphi}{f}(\bphi\opt, \btheta)^{-\top} \pder{\bphi}{h}(\bphi\opt)^\top \bmu\\
            &= \pder{\bphi}{f}(\bphi\opt, \btheta)^{-\top} \pder{\bphi}{h}(\bphi\opt)^\top \pder{\by^2}{^2L}(\bphi\opt) \bmu\opt.
        \end{split}
    \end{equation}
    Multiplying from the left with $\spder{\bphi}f(\bphi\opt, \btheta)^\top$ results in
    \begin{equation}\label{eqn_app:theorem2_eq3}
       -\pder{\bphi}{f}(\bphi\opt, \btheta)^\top \bpsi\opt + \pder{\bphi}{h}(\bphi\opt)^\top \pder{\by^2}{^2L}(\bphi\opt) \bmu\opt = 0.
    \end{equation}
    We now take $\blambda\opt := -\bpsi\opt$ and can easily check that $(\bphi\opt, \bpsi\opt, \blambda\opt, \bmu\opt)$ is a stationary point for the LCP-Lagrangian, i.e. it satisfies the KKT conditions \eqref{eqn_app:constrained_optim}. It follows that $\bpsi\opt$ is an optimal control and $\bphi\opt$ an optimally-controlled state.
\end{proof}

\subsection{Propositions \ref{proposition:loss_equivalence_0}, \ref{proposition:overparameterization}, \ref{proposition:relation_H_C}: the least-control principle solves the original learning problem}
\label{sec_app:solve_original_problem}

We here restate the results from Section~\ref{sec:solve_original_problem} and prove them.

\begin{proposition}
    Assuming $L$ is convex on the system output $\by$, we have that the optimal control $\bpsi\opt$ is equal to 0 iff. the free equilibrium $\bphi\fss$ minimizes $L$.
\end{proposition}
\begin{proof}
    First remark that, as $L$ is convex, $\bphi$ being a minimizer of $L(h(\bphi))$ is equivalent to $\nabla_{\by} L(h(\bphi)) = 0$. Then, the two conditions are equivalent to $\bphi\opt$ and $\bphi\fss$ satisfying both $\ff=0$ and $\nabla_{\by}L(h(\bphi)) = 0$.
\end{proof}

\begin{proposition}
    Assuming $L$ is convex on the system output $\by$, a local minimizer $\btheta$ of the least-control objective $\HH$ is a global minimizer of the original learning problem \eqref{eq:original_optimization_problem}, under the sufficient condition that $\spder{\btheta}{f}(\bphi\opt, \btheta)$ has full row rank.
\end{proposition}
\begin{proof}
    As $\btheta$ is a local minimizer of $\HH(\btheta)$, we have that
    \begin{equation}
        \der{\btheta}{} \HH(\btheta) = f(\bphi\opt, \btheta)\pder{\btheta}{f}(\bphi\opt, \btheta) = 0.
    \end{equation}
    The full row rank assumption of $\partial_{\btheta}{f}(\bphi\opt,\btheta)$ ensures that $f(\bphi\opt, \btheta) = 0$. We hence have that $\bphi\opt$ is a global minimizer of $L(h(\bphi))$, as $\nabla_{\by}L(h(\bphi\opt))=0$, and $L$ is convex, and satisfies $f(\bphi\opt, \btheta)=0$. It follows that $\btheta$ is a global minimizer for the original learning problem.
\end{proof}

\begin{remark}
    In the last two propositions, we assumed that the loss $L$ is convex on output units. This is for example the case in our supervised learning experiments, but not for meta-learning. If we relax this assumption, Propositions~\ref{proposition:loss_equivalence_0} and \ref{proposition:overparameterization} still hold, with the difference that global minimizers become local ones.
    
    Additionally, the result of those two propositions does not only hold in the strong control limit but also in the $\alpha > 0$ regime, under the strict feedback condition of Proposition \ref{prop_app:finite_alpha}. The proof is a direct combination of the proof of the two propositions above with the stationarity result of Proposition~\ref{prop_app:finite_alpha}.
\end{remark}

\begin{proposition}
    If $\frac{1}{2}\snorm{f}^2$ is $\mu$-strongly convex as a function of $\bphi$, $L\circ h$ is $M$-Lipschitz continuous and the minimum of $L$ is 0, then
    \begin{equation}
        L(h(\bphi\fss)) \leq \frac{\sqrt{\mu}}{\sqrt{2}M}\snorm{\psi\opt} \leq \frac{\sqrt{\mu}}{\sqrt{2}M}\snorm{f(\bphi\opt, \btheta)}.
    \end{equation}
\end{proposition}
\begin{proof}
    As $h(\bphi_*)$ minimizes $L$ and $L$ is Lipschitz-continuous,
    \begin{equation}
        L(h(\bphi\fss)) = L(h(\bphi\fss)) - L(h(\bphi\opt)) \leq M \snorm{\bphi\fss -\bphi\opt}.
    \end{equation}
    The strong convexity of $\frac{1}{2}\snorm{f}^2$ yields
    \begin{equation}
         \frac{1}{2}\snorm{f(\bphi\opt, \btheta)} = \frac{1}{2}\snorm{f(\bphi\opt, \btheta)}^2 - \frac{1}{2}\snorm{f(\bphi\fss, \btheta)}^2 \geq \frac{\mu}{2}\snorm{\bphi\opt-\bphi\fss}^2,
    \end{equation}
    as $\bphi\fss$ is a global minimizer of $\frac{1}{2}\snorm{f}^2$. Gathering the last two inequalities gives the desired result.
\end{proof}

We state and prove the remaining informal statements on the impact of approximate equilibria appearing in Section~\ref{sec:solve_original_problem} in Section~\ref{sec_app:approximate_equilibria}, as they require the energy-based view of the least control principle to be proved.

\subsection{The least-control principle as constrained energy minimization}
\label{sec_app:energy_based_version}

In Section~\ref{sec:constrained_energy}, we show that substituting the constraint $f(\bphi, \btheta) + \bpsi = 0$ back into the least-control problem implies that $\bphi\opt$ is a minimizer of
\begin{equation}\label{eqn_app:constrained_optim_energy}
    \min_{\bphi} \frac{1}{2}\snorm{f(\bphi, \btheta)}^2 \quad \mathrm{s.t.} \enspace \nabla_{\by}L(h(\bphi)) = 0.
\end{equation}
The same objective can be obtained from an energy-based perspective, by minimizing the augmented energy $F(\bphi, \btheta, \beta) = \frac{1}{2}\snorm{f(\bphi, \btheta)}^2 + \beta L(h(\bphi))$ and taking the limit $\beta \rightarrow +\infty$. We formalize this result by first showing that the limit of stationary points of $F$ when $\beta \rightarrow \infty$ satisfies the KKT conditions (Proposition \ref{prop_app:KKT_beta_infty}), and that the global minimizers coincide, under more restrictive assumption (Proposition \ref{prop_app:min_prop_control}). Note that the KKT conditions associated to \eqref{eqn_app:constrained_optim_energy} are different from the ones in Eq.~\ref{eqn_app:KKT_LCP}: the corresponding Lagrangian is 
\begin{equation}
    \calL(\bphi, \bmu, \btheta) := \frac{1}{2}\snorm{f(\bphi, \btheta)}^2 + \bmu^\top \left ( \nabla_{\by} L(h(\bphi)) \right )
\end{equation}
and KKT conditions are
\begin{equation}
    \label{eqn_app:KKT_energy_LCP}
    \left\{
    \begin{split}
        &\pder{\bphi}{\calL}(\bphi\opt, \bmu\opt, \btheta) = \pder{\bphi}{}\left [ \frac{1}{2}\snorm{f(\bphi\opt, \btheta)}^2 \right ] + {\bmu\opt}^\top \pder{\by^2}{^2L}(h(\bphi\opt)) \pder{\bphi}{h}(\bphi\opt) = 0 \\
        & \pder{\bmu}{\calL}(\bphi\opt, \bmu\opt, \btheta) = \nabla_{\by}L(h(\bphi\opt)) = 0 \\
    \end{split}
    \right. .
\end{equation}

\begin{proposition}
    \label{prop_app:KKT_beta_infty}
    Let $\bphi_*^\beta$ be a function of $\beta$ that admits a finite limit $\bphi\opt$ when $\beta$ goes to infinity and verifies
    \begin{equation}
        \pder{\bphi}{F}(\bphi\css^\beta, \beta, \btheta) = 0
    \end{equation}
    for every $\beta$.
    If we further assume that the loss Hessian $\partial^2_{\by} L(h(\bphi\opt))$ is invertible and $\spder{\bphi}{h}(\bphi\opt)$ is full rank, then $\bphi\opt$ satisfies the KKT conditions associated to Eq.~\ref{eqn_app:constrained_optim_energy}.
\end{proposition}

\begin{proof}
    We use the shorthand $E := \frac{1}{2}\snorm{f}^2$ for conciseness. For every $\beta$, $\bphi\css^\beta$ verifies
    \begin{equation}
        \pder{\bphi}{F}(\bphi\css^\beta, \btheta, \beta) = \pder{\bphi}{E}(\bphi\css^{\beta}, \btheta) + \beta \pder{\by}{L}(h(\bphi\css^\beta))\pder{\bphi}{h}(\bphi\css^\beta) = 0
    \end{equation}
    By taking $\beta$ to infinity in the equation above and using continuity of $\partial_\bphi E$, we obtain
    \begin{equation}
        \label{eqn:app_prop_conv_prop_proof_1}
        \pder{\bphi}{E}(\bphi\opt) = - \lim_{\beta \rightarrow \infty} \beta \pder{\by}{L}(h(\bphi\css^\beta)) \pder{\bphi}{h}(\bphi\css^\beta).
    \end{equation}
    This first implies that
    \begin{equation}
        \pder{\by}{L}(h(\bphi\opt)) \pder{\bphi}{h}(\bphi\opt) = \lim_{\beta \rightarrow \infty} \pder{\by}{L}(h(\bphi\css^\beta)) \pder{\bphi}{h}(\bphi\css^\beta) = 0
    \end{equation}
    as $\partial_{\bphi}h$, $\partial_{\by} L$  and $h$ are continuous, and then that $\nabla_{\by}L(h(\bphi\opt)) = 0$ as $\partial_{\bphi}h(\bphi\opt)$ is full rank.
    
    Eq. \ref{eqn:app_prop_conv_prop_proof_1}, along with $\partial_{\bphi}h(\bphi\opt)$ being full rank, implies that $\beta \spder{\by}{L}(h(\bphi\css^\beta))$ admits a finite limit, note it $\bmu$. As $\spder{\by}{^2L}(h(\bphi\opt))$ is invertible, there exists $\bmu\opt$ such that $\spder{\by}{^2L}(\bphi\opt)\bmu\opt = \bmu$. Put back into (\ref{eqn:app_prop_conv_prop_proof_1}) it yields
    \begin{equation}
        \pder{\bphi}{E}(\bphi\opt) + {\bmu\opt}^\top \pder{\by^2}{^2L}(h(\bphi\opt))\pder{\bphi}{h}(\bphi\opt) = 0,
    \end{equation}
    hence $\bphi\opt$ satisfies the KKT conditions associated to \eqref{eqn_app:constrained_optim_energy}.
\end{proof}

\begin{proposition}
    \label{prop_app:min_prop_control}
    Assume $L$ to positive and equal to 0 if and only if $\nabla_{\by}L = 0$. Let $\bphi\css^\beta$ be a global minimizer of the augmented energy $F$ for every $\beta$, such that $\bphi\css^\beta$ admits a finite limit $\bphi\opt$ when $\beta$ goes to infinity. Then the limit $\bphi\opt$ is a global minimizer of the constrained optimization problem \eqref{eqn_app:constrained_optim_energy}.
\end{proposition}
\begin{proof}
    First, remark that the condition on the minimizers of $L$ implies that solving
    \begin{equation}
        \min_{\bphi}\, \frac{1}{2}\snorm{f(\bphi, \btheta)}^2 \quad \mathrm{s.t.} \enspace \nabla_{\by}L(h(\bphi)) = 0
    \end{equation}
    is equivalent to solving
    \begin{equation}
        \min_{\bphi}\, \frac{1}{2}\snorm{f(\bphi, \btheta)}^2 \quad \mathrm{s.t.} \enspace L(h(\bphi)) = 0.
    \end{equation}
    We use the second characterization in the rest of the proof.
    
    We introduce the short hand
    \begin{equation}
        \begin{split}
            f\opt &:= \min_{\bphi}\, \frac{1}{2}\snorm{f(\bphi, \btheta)}^2 \quad \mathrm{s.t.} \enspace L(h(\bphi)) = 0.
        \end{split}
    \end{equation}
    
    For any $\bphi$ we have that
    \begin{equation}
        F(\bphi\css^\beta, \btheta, \beta) \leq F(\bphi, \btheta, \beta)
    \end{equation}
    by definition of $\bphi\css^\beta$, in particular for all the $\bphi$ verifying the constraint $L(h(\bphi)) = 0$. It follows that
    \begin{equation}
        \label{eqn:app_thm_penalty_en_proof_1}
        F(\bphi\css^\beta, \btheta, \beta) \leq \min_{L(h(\bphi))=0} F(\bphi, \btheta, \beta) = \min_{L(h(\bphi))=0} \frac{1}{2}\snorm{f(\bphi, \btheta)}^2 = f\opt.
    \end{equation}
    The last inequality can be rewritten as
    \begin{equation*}
        L(h(\bphi\css^\beta)) \leq \frac{f\opt- \snorm{f(\bphi\css^\beta, \btheta)}^2/2}{\beta}.
    \end{equation*}
    The upper bound converges to 0 as the numerator converges to a finite value by continuity of $\snorm{f}^2$. Since $L$ is positive and continuous and $h$ is continuous we obtain $L(h(\bphi\opt)) = 0$, i.e., $\bphi\opt$ is feasible.
    
    We now show that $\snorm{f(\bphi\opt, \btheta)}^2/2 = f\opt$. As $\bphi\opt$ is feasible we have
    \begin{equation}
        f\opt = \min_{L(h(\bphi))=0} \frac{1}{2}\snorm{f(\bphi, \btheta)}^2 \leq \frac{1}{2}\snorm{f(\bphi\opt, \btheta)}^2.
    \end{equation}
    This implies that
    \begin{equation}
        \limsup_{\beta \rightarrow \infty} F(\bphi\css^\beta, \btheta, \beta) = \frac{1}{2}\snorm{f(\bphi\opt, \btheta)}^2 + \limsup_{\beta \rightarrow \infty} \beta L(h(\bphi\css^\beta)) \geq f\opt
    \end{equation}
    where we used the positivity of $L$ for the last inequality. Because of (\ref{eqn:app_thm_penalty_en_proof_1}), we also have
    \begin{equation}
        \limsup_{\beta \rightarrow \infty} F(\bphi_*^\beta, \beta) \leq f\opt
    \end{equation}
    so that the combination of the last two equations yield
    \begin{equation}
        \limsup_{\beta \rightarrow \infty} \beta L(h(\bphi\css^\beta)) = 0
    \end{equation}
    and
    \begin{equation}
        \frac{1}{2}\snorm{f(\bphi\opt,\btheta)}^2 = f\opt.
    \end{equation}
    This finishes the proof.
\end{proof}
\begin{remark}
    Under the same assumptions, we can actually show that, for any $\beta$, $\beta' \in \mathbb{R}_+$, we have
    \begin{equation}
        0 \leq F(\bphi\css^\beta, \btheta, \beta) \leq F(\bphi\css^{\beta'}, \btheta, \beta') \leq \frac{1}{2}\snorm{f(\bphi\opt, \btheta)}^2.
    \end{equation}
    This can be obtained by combining Proposition \ref{prop_app:min_prop_control} and the fact that $\beta \mapsto F(\bphi\css^\beta, \btheta, \beta)$ is a non decreasing function as
    \begin{equation}
        \begin{split}
            \der{\beta}{}F(\bphi\css^\beta, \btheta, \beta) &= \pder{\beta}{F}(\bphi\css^\beta, \btheta, \beta) + \pder{\bphi}{F}(\bphi\css^\beta, \btheta, \beta) \der{\beta}{\bphi\css^\beta}\\
            &= L(\bphi\css^\beta) + 0\\
            & \geq 0.
        \end{split}
    \end{equation}
    The function $\beta \mapsto F(\bphi\css^\beta, \btheta, \beta)$ therefore converges to the least-control objective $\HH(\btheta)$ from below.
\end{remark}

In the energy-based view of the least-control principle, the least-control objective becomes
\begin{equation}
    \HH(\btheta) = \frac{1}{2}\snorm{f(\bphi\opt, \btheta)}^2.
\end{equation}
We use this formulation to provide an alternative proof to Theorem~\ref{theorem_app:first_order_updates} that might be insightful for the reader.
\begin{proof}[Alternative proof of Theorem~\ref{theorem_app:first_order_updates}]
    Theorem~\ref{theorem_app:first_order_updates} states that
    \begin{equation}
        \left (\der{\btheta}{}\mathcal{H}(\btheta) \right )^\top = \pder{\btheta}{f}(\bphi\opt, \btheta)^\top f(\bphi\opt, \btheta).
    \end{equation}
    We use the shorthand notation $E := \frac{1}{2}\snorm{f(\bphi, \btheta)}^2$. The chain rule yields
    \begin{equation}
        \der{\btheta}{}E(\bphi\opt, \btheta) = \der{\theta}{}E(\bphi\opt, \btheta) = \pder{\btheta}{E}(\bphi\opt, \btheta) + \pder{\bphi}{E}(\bphi\opt, \btheta)\der{\btheta}{\bphi\opt}.
    \end{equation}
    As $\nabla_{\bphi} L$ is independent of $\btheta$, we have that $\mathrm{d}_{\btheta}h(\bphi\opt) = 0$. This rewrites $\partial_{\bphi}h(\bphi\opt)\mathrm{d}_{\btheta}\bphi\opt = 0$. This can be combined to the KKT conditions to show that the indirect term is equal to $0$: right multiplying the condition
    \begin{equation}
        0 = \pder{\bphi}{E}(\bphi\opt, \btheta) + {\bmu\opt}^\top \pder{\by}{^2L}(h(\bphi\opt)) \pder{\bphi}{h}(\bphi\opt)
    \end{equation}
    by $\mathrm{d}_{\btheta}\bphi\opt$ gives
    \begin{equation}
        \begin{split}
            0 &= \pder{\bphi}{E}(\bphi\opt, \btheta) \der{\btheta}{\bphi\opt} + {\bmu\opt}^\top \pder{\by}{^2L}(h(\bphi\opt)) \pder{\bphi}{h}(\bphi\opt) \der{\btheta}{\bphi\opt}\\
            &= \pder{\bphi}{E}(\bphi\opt, \btheta) \der{\btheta}{\bphi\opt}.
        \end{split}
    \end{equation}
    This finishes the proof.
\end{proof}

\subsection{The parameter update is robust to approximate optimal control}
\label{sec_app:approximate_equilibria}

The purpose of this section is to formally prove that the first-order update prescribed by Theorem~\ref{theorem:first_order_updates} is robust to inaccurate approximations of the optimally-controlled state $\bphi\opt$. More precisely, we show that the quality of the gradient can be linked to the distance between the approximate and the true optimally-controlled state (Proposition~\ref{prop_app:inacurrate_optimal_control}) and that the update is minimizing a closely related objective when the perfect control limit is not exactly reached (Proposition~\ref{prop_app:finite_alpha}).

\begin{proposition}
    \label{prop_app:inacurrate_optimal_control}
    Let $\bphi_*$ be an optimally-controlled state and $\hat{\bphi}$ an approximation of it that is used to update the parameters $\btheta$. Then, if the conditions of Theorem~\ref{theorem_app:first_order_updates} apply and $\partial_{\btheta} \ff^\top \ff$ is $M$-Lipschitz continuous we have that
    \begin{equation}
        \label{eqn_app:inacurrate_optimal_control}
        \norm{f(\hat{\bphi}, \btheta)^\top\pder{\btheta}{f}(\hat{\bphi},\btheta) - \der{\btheta}{}\mathcal{H}(\btheta)} \leq M \snorm{\hat{\bphi} - \bphi_*}.
    \end{equation}
\end{proposition}
\begin{proof}
    We use Theorem~\ref{theorem_app:first_order_updates} and then the Lipschitz continuity assumption:
    \begin{equation}
        \begin{split}
            \norm{f(\hat{\bphi}, \btheta)^\top\pder{\btheta}{f}(\hat{\bphi},\btheta) - \der{\btheta}{}\mathcal{H}(\btheta)} &= \norm{f(\hat{\bphi}, \btheta)^\top\pder{\btheta}{f}(\hat{\bphi},\btheta) - f(\bphi\opt, \btheta)^\top\pder{\btheta}{f}(\bphi\opt,\btheta)} \\
            &\leq M \snorm{\hat{\bphi} - \bphi\opt}.
        \end{split}
    \end{equation}
\end{proof}

We empirically verify that the behavior predicted by Proposition~\ref{prop_app:inacurrate_optimal_control} holds in practice. To do so, we compare the estimated gradient $f(\hat{\phi},\theta)^\top \partial_\theta f(\hat{\phi}, \theta)$ for several approximations $\hat{\phi}$ of the optimally-controlled state $\phi_*$ to the true least-control gradient $\mathrm{d}_\theta \mathcal{H}(\theta)$ on a feedforward neural network learning problem. Figure~\ref{fig_app:test_convexity} shows the gradient estimation error as a function of the steady state approximation $\lVert \hat{\phi}-\phi_*\rVert$. The evolution is almost linear, suggesting that the inequality in Eq.~\ref{eqn_app:inacurrate_optimal_control} is close to being an equality for some constant $M$.
\begin{figure}
    \centering
    \includegraphics{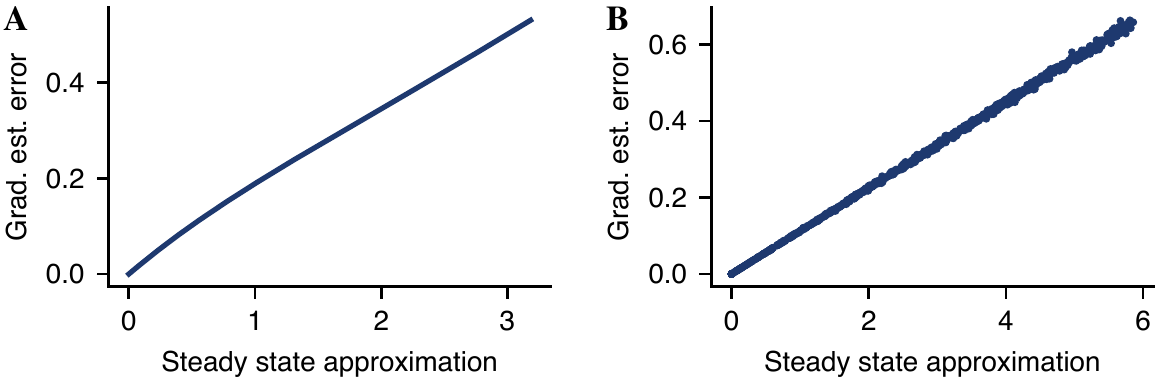}
    \caption{Empirical verification of Proposition~\ref{prop_app:inacurrate_optimal_control} on a randomly initalized feedforward neural network with 3 hidden layers of size 300. The different gradients are estimated using 100 data samples randomly drawn from the MNIST dataset. (A) We run energy-based dynamics for $\alpha=0$ with a time constant of 200 for 200 steps, compute the gradient estimation error at each step, and plot it as a function of the steady state approximation error $\lVert \hat{\phi} - \phi_* \rVert$. The optimally-controlled $\phi_*$ is computed by running the same dynamics until convergence. (B) After computing $\phi_*$, we perturb it with random noise drawn from $\mathcal{U}([0, 1])\times \mathcal{N}(0, 0.5)$, and plot the same relationship as in panel A.}
    \label{fig_app:test_convexity}
\end{figure}

\begin{proposition}
    \label{prop_app:finite_alpha}
    Let $(\bphi\css, \bu\css)$ verifying the steady-state equations
    \begin{equation}
        \label{eqn_app:steady_state_finite_alpha}
        f(\bphi\css, \btheta) + Q(\bphi\css, \btheta) \bu\css = 0, \quad \mathrm{and} \quad - \alpha \bu\css -\nabla_{\by}L(h(\bphi\css)) = 0.
    \end{equation}
    If 
    \begin{equation}
        \label{eqn_app:strict_Q_finite_alpha}
        Q(\bphi\css, \beta) = \spder{\bphi}{f}(\bphi\css,\btheta)^{-\top}\spder{\bphi}{h}(\bphi\css)^\top,
    \end{equation}
    then $\bphi\css$ is a stationary point of the augmented energy $F(\bphi, \btheta, \alpha) = \frac{1}{2}\snorm{f(\bphi, \btheta)}^2+\alpha^{-1}L(h(\bphi))$.
\end{proposition}

\begin{proof}
    We first use the formula for $Q$ and use the steady-state equation on $\bu\css$:
    \begin{equation}
        \begin{split}
            0 &= f(\bphi\css, \btheta) + Q(\bphi\css, \btheta) \bu\css\\
            &= f(\bphi\css, \btheta) - \spder{\bphi}{f}(\bphi\css,\btheta)^{-\top}\spder{\bphi}{h}(\bphi\css)^\top \bu\css\\
            &= f(\bphi\css, \btheta) + \alpha^{-1}\spder{\bphi}{f}(\bphi\css,\btheta)^{-\top}\spder{\bphi}{h}(\bphi\css)^\top \nabla_{\by}L(h(\bphi\css))
        \end{split}
    \end{equation}
    Left-multiplying by $\spder{\bphi}{f}(\bphi\css, \btheta)$ yields
    \begin{equation}
        \spder{\bphi}{f}(\bphi\css,\btheta)^{\top} f(\bphi\css, \btheta) + \alpha^{-1}\spder{\bphi}{h}(\bphi\css)^\top \nabla_{\by}L(h(\bphi\css)) = 0,
    \end{equation}
    which is exactly $\nabla_\bphi F(\bphi, \btheta, \alpha) = 0$.
\end{proof}

Proposition \ref{prop_app:finite_alpha} allows understanding what the update $\spder{\btheta}{f}(\bphi\css, \btheta)^\top f(\bphi\css, \btheta)$ is doing when the controlled dynamics reach an equilibrium, but $\alpha$ is non zero. In this case, we cannot guarantee that $Q(\bphi\css, \btheta)\bu\css$ is an optimal control. Still, when the strict condition on $Q$ of Eq.~\ref{eqn_app:strict_Q_finite_alpha} is satisfied (it is strict compared to the more general condition \eqref{eqn_app:columnspace_condition} of Theorem~\ref{theorem_app:columnspace}), the update prescribed by the least-control principle is minimizing $F(\bphi\css, \btheta, \alpha)$ as
\begin{equation}
    \begin{split}
        \der{\btheta}{} F(\bphi\css, \btheta, \alpha) & = \pder{\btheta}{F}(\bphi\css, \btheta, \alpha) + \pder{\bphi}{F}(\bphi\css, \btheta, \alpha) \der{\btheta}{\bphi\css}\\
        & = \pder{\btheta}{F}(\bphi\css, \btheta, \alpha) + 0\\
        & = \pder{\btheta}{f}(\bphi\css, \btheta)^\top f(\bphi\css, \btheta).
    \end{split}
\end{equation}
We showed in Section~\ref{sec_app:energy_based_version} that this objective is closely related to the least-control objective, and gets closer when $\alpha \rightarrow 0_+$. Note that the same conclusions apply if we use a proportional controller with strength $\beta$, as it verifies the steady-state equation (\ref{eqn_app:steady_state_finite_alpha}) for $\alpha = \beta^{-1}$.

\subsection{Quick review of constrained optimization}
\label{sec_app:review_constrained_optim}

We heavily rely on the Karush–Kuhn–Tucker (KKT) conditions and Lagrange multiplier method for characterizing constrained optimum. We state here some important properties that are throughout the proofs, and we refer the interested reader to \citeS{nocedal_numerical_2006, bertsekas_constrained_2014} for a comprehensive overview of the relevant theory. 

\textbf{KKT conditions.} Consider the following equality-constrained optimization problem 
\begin{align}
    \min_{\bx} f(\bx) ~~ \mathrm{s.t.} ~ g(\bx) = 0,
\end{align}
with $f$ some scalar function to be optimized and $g$ a vector-valued function defining equality constraints on the set of feasible states. The corresponding Lagrangian is defined as
\begin{align}
    \calL(\bx, \blambda) := f(\bx) + \blambda^\top g(\bx),
\end{align}
with $\blambda$ the so-called Lagrange multipliers. It can be shown, e.g. \citetS[Theorem 12.1]{nocedal_numerical_2006}, that local solutions $\bx\opt$ of the constrained optimization problem defined above satisfy the KKT conditions: there exists a Lagrange multiplier $\blambda\opt$ such that
\begin{equation}
    \left \{
    \begin{split}
        0&=\pder{\bx}{f}(\bx\opt) + \pder{\bx}{g}(\bx\opt)\blambda\opt \\
        0&=g(\bx)
    \end{split}
    \right .
\end{equation}
or equivalently such that $(\bx\opt, \blambda\opt)$ is a stationary point of the Lagrangian, i.e. $\spder{\bx\opt, \blambda\opt} \calL(\bx\opt, \blambda\opt) = 0$.

The KKT conditions are first-order stationary conditions for constrained optimization, in the same way $\spder{\bx}{f} = 0$ is a stationary condition for unconstrained minimization. On the other side, there exists sufficient conditions under which those stationary conditions give a local, constrained or unconstrained, minimum. For unconstrained optimization, it can be that the Hessian $\nabla^2_{\bx} f$ is positive definite. For constrained optimization there are several variants, but those conditions can more or less be understood as the Hessian $\partial^2_{\bx}\calL(\bx\opt, \blambda\opt)$ of the Lagrangian w.r.t.\ $\bx$ is positive definite (see e.g. \citetS[Theorem 12.5]{nocedal_numerical_2006}). 

\textbf{Differentiating through minima.} We now consider that both the objective $f$ and the constraints $g$ are dependent on some parameter $\by$, and aim to calculate how $f(\bx\opt(\by), \by)$ will react to a change in $\by$, that is we aim to compute the gradient associated to 
\begin{align}
    \min_{\by} f(\bx\opt(\by), \by) = \min_{\by} \min_{\bx} f(\bx, \by) ~~ \mathrm{s.t.} ~ g(\bx, \by) = 0.
\end{align}

For \textit{unconstrained} bilevel optimization (i.e.\ without the equality constraints $g(\bx, \by) = 0$), we differentiate through the inner minima $\bx\opt(\by)$ using the necessary stationary condition $\spder{\bx}{f}(\bx\opt(\by), \by) = 0$, combined with the implicit function theorem \citeS{dontchev_implicit_2009} to ensure that the implicit function $\bx\opt(\by)$ is well defined. This gives
\begin{equation}
    \begin{split}
        \der{\by}{}f(\bx\opt(\by), \by) & = \pder{\by}{f}(\bx\opt(\by), \by) + \pder{\bx}{f}(\bx\opt(\by), \by) \der{\by}{\bx\opt(\by)}\\
        &= \pder{\by}{f}(\bx\opt(\by), \by) + 0.
    \end{split}
\end{equation}
Note that $\sder{\by}{\bx\opt(\by)}$ can be obtained via the implicit function theorem by using the first-order condition $\spder{\bx}{f} (\bx\opt(\by), \by) = 0$ but we do not need to compute it.

A similar technique can be applied to differentiate through a \textit{constrained} minimum. We first remark that $\bx\opt$ satisfies the KKT conditions so there exists a Lagrange multiplier $\blambda\opt(\by)$ such that
\begin{equation}
    f(\bx\opt(\by), \by) = \mathcal{L}(\bx\opt(\by), \blambda\opt(\by), \by).
\end{equation}
We then use the implicit function theorem on the KKT condition $\partial_{\bx, \blambda} \calL(\bx\opt, \blambda\opt, \by) = 0$ to show that the functions $\bx\opt(\by)$ and $\blambda\opt(\by)$ are well defined and differentiable, under the assumption that the Hessian of the Lagrangian w.r.t. $\bx$ and $\blambda$ is invertible. Finally, we have:
\begin{equation}
    \begin{split}
        \der{\by}{}f(\bx\opt(\by), \by) &= \der{\by}{} \calL(\bx\opt(\by), \blambda\opt(\by), \by)\\
        & = \pder{\by}{\calL} + \pder{\bx}{\calL}\der{\by}{\bx\opt} + \pder{\blambda}{\calL}\der{\by}{\blambda\opt}\\
        & = \pder{\by}{\calL} + 0 + 0 \\
    \end{split}
\end{equation}
where the last inequality holds as $\bx\opt(\by)$ and $\blambda\opt(\by)$ are stationary points of the Lagrangian $\mathcal{L}$.

\newpage
\section{Implicit gradient, recurrent backpropagation and the link to the least-control principle} \label{sec_app:rbp}

The least-control principle aims at solving the learning problem
\begin{equation}
    \min_\theta L(h(\phi\fss)) \quad \mathrm{s.t.} \enspace f(\phi\fss, \theta) = 0
\end{equation}
indirectly, by minimizing the least-control objective. This leads to first-order updates, contrary to the direct minimization of the loss $L(\bphi\fss)$. In this section, we derive the implicit gradient associated to direct loss minimization, show how the recurrent backpropagation computes it, and finally highlight the similarities behind the least control principle and recurrent backpropagation.

\textbf{Implicit gradient.} The implicit gradient $\mathrm{d}_\btheta L(\phi\fss)$ can be calculated analytically using the implicit function theorem \citeS{dontchev_implicit_2009}, we here calculate it. The quantity $\bphi\fss$ can be characterized as an implicit function of $\theta$, which verifies
\begin{equation}
    f(\phi\fss, \theta) = 0
\end{equation}
for all $\theta$. The implicit function theorem guarantees that $\phi\fss$ is properly defined as an implicit function of $\theta$ if $\spder{\phi}{f}(\phi\fss, \theta)$ is invertible, and then
\begin{equation*}
    \der{\btheta}{\phi\fss} = - \left ( \pder{\phi}{f}(\phi\fss, \theta) \right )^{-1} \pder{\theta}{f}(\phi\fss, \theta).
\end{equation*}
We can then use the chain rule to obtain the implicit gradient:
\begin{equation}
    \begin{split}
        \der{\btheta}{}L(h(\phi\fss)) &= \pder{\by}{L}(h(\phi\fss)) \pder{\bphi}{h}(\phi\fss)\der{\btheta}{\phi\fss}\\
        &= -\pder{\by}{L}(\phi\fss) \pder{\bphi}{h}(h(\phi\fss)) \left ( \pder{\phi}{f}(\phi\fss, \theta) \right )^{-1} \pder{\theta}{f}(\phi\fss, \theta).
    \end{split}
\end{equation}

\textbf{Recurrent backpropagation.} The implicit gradient in computationally intractable for most of practical applications due to the matrix inversion. The key idea behind the recurrent backpropagation algorithm is that the column vector
\begin{equation}
    \delta^* := \left ( \pder{\phi}{f}(\phi\fss, \theta) \right )\!^{-\top} \pder{\bphi}{h}(\phi\fss)^\top\pder{\by}{L}(h(\phi\fss))^\top
\end{equation}
is the steady state of the dynamics
\begin{equation}
    \label{eqn_app:2nd_phase_rbp}
    \dot{\delta} = \left ( \pder{\phi}{f}(\phi\fss, \theta) \right )^\top \delta - \pder{\phi}{h}(\phi\fss)^\top\pder{\by}{L}(h(\phi\fss))^\top
\end{equation}
or alternatively the fixed point of the iterative procedure
\begin{equation}
    \delta_\mathrm{next} = \delta + \left ( \pder{\phi}{f}(\phi\fss, \theta) \right )^\top \delta - \pder{\phi}{h}(\phi\fss)^\top\pder{\by}{L}(h(\phi\fss))^\top.
\end{equation}
Once $\delta\fss$, or its estimation, is obtained through this dynamical procedure, it can then be used to estimate the implicit gradient with \begin{equation}
    \label{eqn_app:implicit_gradient_rbp}
    \der{\theta}{} L(\phi\fss) = -{\delta\fss}^\top \pder{\theta}{f}(\phi\fss, \theta)
\end{equation}

\textbf{Backpropagation.} When the computational graph encoded in $f$ is acyclic, that is we can find a permutation of the state such that $f(\phi, \theta)_i$ only depends on the states $\phi_j$ for $j \leq i$, the computation of $\delta\fss$ can be simplified. Indeed, instead of running the dynamics of Eq.~\ref{eqn_app:2nd_phase_rbp}, one can directly compute $\delta\fss$ through the fixed-point equation
\begin{equation}
    \left ( \pder{\phi}{f}(\phi\fss, \theta) \right )^\top \delta = \pder{\phi}{h}(\phi\fss)^\top\pder{\by}{L}(h(\phi\fss))^\top.
\end{equation}
as the Jacobian $\partial_\phi f(\phi\fss,\theta)$ is a lower triangular matrix. This is for example how $\delta^*$ is calculated for feedforward neural networks. In such settings, we can therefore interpret the dynamics \eqref{eqn_app:2nd_phase_rbp} to be the continuous time version of the backpropagation algorithm.

\textbf{Comparison with the least-control principle.} We now compare the implicit gradient, as computed with recurrent backpropagation, with the gradient of the least-control objective, which is equal to
\begin{equation}
    \label{eqn_app:lcp_dtheta_H}
    \der{\btheta}{\HH}(\btheta) = -{\psi\opt}^\top \pder{\theta}{f}(\phi\opt, \theta)
\end{equation}
as given by Theorem~\ref{theorem:first_order_updates}. In addition to the close resemblance between the gradients \eqref{eqn_app:implicit_gradient_rbp} and \eqref{eqn_app:lcp_dtheta_H}, the dynamical equations on $\delta$ and $\psi$ share some interesting similarities. Let us take the example of the inversion dynamics
\begin{equation}
    \begin{split}
        \dot{\phi} &= f(\phi,\btheta) + \psi\\
        \dot{\psi} &= \left ( \pder{\phi}{f}(\phi, \theta) \right )^\top \psi + \pder{\phi}{h}(\phi)^\top u\\
    \end{split}
\end{equation}
of Eq.~\ref{eqn:deq_inversion}. The $\psi$ dynamics are very similar to the one of $\delta$: they can be understood as multiplying the inverse of the Jacobian of $f$ with an error minimization signal, that is either $-\spder{\by}{L}$ or $u$, to yield $\delta\fss$ or $\psi\opt$. The important conceptual difference is that the controlled dynamics of the least-control principle simultaneously compute both the steady state $\bphi\css$ and optimal control $\bpsi\css$ in a single phase, whereas recurrent backpropagation uses two separate phases for computing $\bphi\fss$ and $\delta\fss$. This unique feature of the least-control principle is made possible by letting the control $\bpsi$ influence the dynamics and hence the steady state of $\bphi$.

\newpage
\section{On the relation between the least-control principle, free energy and predictive coding} \label{sec_app:free_energy}
Predictive coding \citep{whittington_approximation_2017, friston_free-energy_2009, rao_predictive_1999, keller_predictive_2018}\citeS{friston_free_2006} emerged as an important framework to study perception and sensory processing in the brain. In this framework, the brain is assumed to maintain a probabilistic model of the environment, which is constantly used to perform inference on its sensory inputs. Here, we expand in more detail on the relation between our least-control principle and predictive coding. More specifically, we discuss how the constrained energy minimization of Eq. \ref{eqn:least_violation_physics} relates to the free energy used in recent work on supervised predictive coding \citep{whittington_approximation_2017}, and how the new insights of the least-control principle, such as more flexible dynamics, translate to the predictive coding framework. 

\subsection{Supervised predictive coding}
The starting point of standard predictive coding theories is the specification of a probabilistic generative model from latent causes towards sensory inputs; this model is then used to infer the most likely causes underlying incoming sensory stimuli. Recent work \citep{whittington_approximation_2017} investigated supervised predictive coding, a form of predictive coding where the direction of causality is flipped, i.e., where one uses a discriminative model from sensory inputs towards latent causes, and trains this model by clamping its outputs to the correct teacher value and performs inference to propagate teaching signals throughout the network. \citet{whittington_approximation_2017} related the parameter updates resulting from supervised predictive coding network to error backpropagation \citep{werbos_beyond_1974, rumelhart_learning_1986} when the output is weakly nudged towards the teacher value instead of being clamped. In Section \ref{sec_app:pc_teacher_clamping} we show that the least-control principle relates the supervised predictive coding updates with standard teacher clamping to gradient updates on the least-control objective \eqref{eqn:least_control_problem} thereby removing the need for the weak-nudging limit. 

Supervised predictive coding uses a probabilistic latent variable model to compute the likelihood $p(y \mid x ;\btheta)$ of the output $y$ given the sensory input $x$. More specifically, it uses a hierarchical Gaussian model, a particular type of directed acyclic graphical model, with the following factorized distribution
\begin{equation} \label{eqn_app:pc_likelihood}
    \begin{split}
        p(z_1, \hdots, z_{L-1}, \by \mid \bx; \btheta) &= p(\by \mid z_{L-1})\prod_{l=1}^{L-1} p(\bphi_l \mid z_{l-1}) \\
        p(z_l \mid \bphi_{l-1}) &= \calN(z_l; ~W_l \sigma(z_{l-1}), \Sigma_l))\\
        p(\by \mid z_{L-1}) &= \calN(\by;~ W_L \sigma(z_{L-1}, \Sigma_{\by}))
    \end{split}
\end{equation}
with latent variables $z = \{z_l\}_{l=1}^{L-1}$, parameters $\btheta = \{W_l\}_{l=1}^L$, $\calN(v; \mu, \Sigma)$ a Gaussian distribution over a variable $v$ with mean $\mu$ and correlation matrix $\Sigma$, and where we take $\sigma(z_0) = \bx$. The goal of training this model is to change the model parameters $W_l$ to maximize the marginal likelihood for input-output pairs $(x,y)$:
\begin{equation}
    p(\by \mid \bx ; \btheta) = \int p(z, \by \mid \bx; \btheta)\mathrm{d}z 
\end{equation}
As optimizing this marginal directly is intractable, predictive coding uses the variational expectation maximization method \citepS{neal_view_1998, tzikas_variational_2008, bogacz_tutorial_2017, friston_free_2006}. This method uses a variational distribution $q(z ; \bphi)$ with parameters $\bphi$ to approximate the posterior $p(z \mid x, y)$ and then defines the free energy $\mathcal{F}$ as an upper bound on the negative log-likelihood: 
\begin{equation}
    \begin{split}
        \mathcal{F}(\bphi, \btheta) &= - \EE_q[\log p(y,z \mid x ;\btheta)] + \EE_q[\log q(z;\bphi)] \\
        &= -\log p(\by \mid \bx ; \btheta) + D_{\mathrm{KL}}(q(z;\bphi) \| p(z \mid x,y ;\btheta)) \geq -\log p(\by \mid \bx ; \btheta),
    \end{split}
\end{equation}
with $D_{\mathrm{KL}}$ the KL divergence which is always positive. Now the variational expectation maximization method proceeds in two alternating phases. In the expectation phase, it minimizes the free energy $\mathcal{F}$ w.r.t. $\bphi$ for the current parameter setting of $\btheta$, i.e. it finds a good approximation $q(z;\bphi)$ for the posterior. In the maximization phase, it minimizes the free energy w.r.t. $\btheta$ for the updated parameter setting of $\bphi$.

To be able to perform the two phases using simple neural dynamics and local update rules, predictive coding uses the Dirac delta distribution $q(z;\bphi) = \delta(z - \bphi)$ as the variational distribution \citepS{bogacz_tutorial_2017}. Now, the entropy term $\EE_q[\log p(y,z \mid x ;\btheta)]$ is independent of $\bphi$ and hence we ignore it in the expectation and maximization step.\footnote{Note however that this entropy term is infinite for the delta distribution, hence for making the derivation rigorous, certain limits need to be taken. For example, \citetS{friston_free_2006} uses a Laplace approximation instead of the delta distribution.} Taken together, this result in the following free energy (without entropy term): 
\begin{equation}\label{eqn_app:pc_free_energy}
    \begin{split}
        \mathcal{F}(\bphi, \btheta) &= - \log p(y,\bphi \mid x ;\btheta) \\
        &= \rho\frac{1}{2} \snorm{y - W\sigma(\bphi_{L-1})}^2 + \frac{1}{2} \sum_{l=1}^{L-1} \snorm{\bphi_l - W_l \sigma(\bphi_{l-1})}^2 + C
    \end{split}
\end{equation}
with $C$ a constant independent of $\bphi$ and $\btheta$, $\bx = \sigma(\bphi_0)$, and where we took for simplicity $\Sigma_l = \Id$ and $\Sigma_y = \rho^{-1} \Id$. For the expectation step, predictive coding interprets the parameters $\bphi = \{\bphi_l\}_{l=1}^{L-1}$ as a prediction neuron population and uses the gradient flow on $\mathcal{F}$ w.r.t. $\bphi$ as neural dynamics to find a $\bphi_*$ that minimizes $\mathcal{F}$. For the maximization step, predictive coding performs a gradient step on $\mathcal{F}(\bphi_*, \btheta)$ w.r.t. $\btheta$, evaluated at the steady-state neural activity $\bphi_*$. Thus,
\begin{equation}
    \label{eqn_app:pc_dynamis}
    \begin{split}
        \dot{\bphi} &= - \pder{\bphi}{\mathcal{F}}(\bphi, \btheta)^\top,\\
        \Delta \btheta &= - \pder{\btheta}{\mathcal{F}}(\bphi\css, \btheta)^\top,
    \end{split}
\end{equation}
with the change $\Delta \theta$ applied after the neural dynamics has converged, i.e., when $\dot{\bphi} = 0$.

\subsection{Equivalence between the free and augmented energies}
To link supervised predictive coding to the least-control principle, we further manipulate the free energy of Eq. \ref{eqn_app:pc_free_energy} to relate it to the augmented energy $F = \frac{1}{2} \snorm{\ff}^2 + \beta L$ introduced in Section \ref{sec:constrained_energy}. 
For this, we introduce a `ghost layer' $z_L$ in between $z_{L-1}$ and $\by$, that exists solely for the purpose of analyzing the free energy. The joint probability distribution is now given by
\begin{equation} \label{eqn_app:pc_ghost_likelihood}
    \begin{split}
        p(z, \by \mid \bx;\btheta) &= p(\by \mid z_{L})\prod_{l=1}^{L} p(z_l \mid z_{l-1}) \\
        p(z_l \mid z_{l-1}) &= \calN(z_l; ~W_l \sigma(z_{l-1}), \Id)\\
        p(\by \mid z_{L}) &= \calN(\by;~z_{L}, \beta^{-1} \Id)
    \end{split}
\end{equation}
We recover the same marginal likelihood $p(y\mid x;\btheta)$ as before if we have that $\rho^{-1} = \beta^{-1} + 1$, assuming that $\rho^{-1} > 1$. Using again the dirac delta distribution as variational distribution, we get the following augmented free energy

\begin{equation}\label{eqn_app:pc_augmented_energy}
\begin{split}
    \mathcal{F}(\bphi, \btheta, \beta) &= \log p(\bphi, \by \mid \bx ;\btheta) = \frac{1}{2} \sum_{l=1}^L \snorm{\bphi_l - W_l\sigma(\bphi_{l-1})}^2 + \beta \frac{1}{2} \snorm{\by - \bphi_L}^2 \\ 
    & = \frac{1}{2} \snorm{\bphi - W \sigma(\bphi) - U\bx}^2 + \beta L(D\bphi) \\
     W &= \begin{bmatrix} 0 & 0 & 0 & 0\\ W_2 & 0 & 0 & 0 \\ 0 & \ddots & 0 & 0 \\ 0 & 0 & W_L & 0 \end{bmatrix}, ~~~ U = \begin{bmatrix} W_1 \\ 0 \end{bmatrix}, ~~~ D = \begin{bmatrix} 0 & \Id \end{bmatrix},
\end{split}
\end{equation}
with $\bphi$ the concatenation of $\{\bphi_l\}_{l=1}^L$, $L$ the squared error loss and where we ignored the constant term $C$. Hence, we see that the free energy used in predictive coding is equivalent to the augmented energy $F = \frac{1}{2} \snorm{\ff}^2 + \beta L$ introduced in Section \ref{sec:constrained_energy}. Consequently, we can use the least-control principle to characterize predictive coding in the limit of $\beta \to \infty$. 

\subsection{Learning with teacher clamping results in gradients on the least-control objective}\label{sec_app:pc_teacher_clamping}
\citet{whittington_approximation_2017} connected supervised predictive coding to gradient-based optimization of the loss $L$, when the variance of the output layer $\by$ goes to infinity. In this limit, there is only an infinitesimal effect of the teaching signal on the rest of the network, arising by clamping $\by$ towards $\by^{\mathrm{true}}$, as the model `does not trust' the output layer $\by$ by construction. This \textit{weak nudging} limit is captured in the augmented energy of Eq. \eqref{eqn_app:pc_augmented_energy} as the limit of $\beta \to 0$, and can be linked to gradient-based optimization of the loss $L$ for more general energy functions beyond predictive coding \citep{scellier_equilibrium_2017}.

Equipped with the least-control principle, we can now relate predictive coding to gradient-based optimization on the least-control objective of Eq. \ref{eqn:least_control_problem}, without the need for infinite variance in the output layer. In Sections \ref{sec:constrained_energy} and \ref{sec_app:energy_based_version}, we show that when the neural dynamics minimize the augmented energy $F(\bphi, \btheta, \beta)$ for $\beta \to \infty$, the parameter update $\Delta \btheta$ of Eq. \ref{eqn_app:pc_dynamis} follows the gradient on the least-control objective \eqref{eqn:least_control_problem}. This perfect control limit of $\beta \to \infty$ corresponds to the conventional teacher clamping setting, where the output $\by$ has a finite variance (without loss of generality, $\rho = 1$ in this case in our derivation, as $\rho^{-1} = \beta^{-1} + 1$). 

\subsection{The optimal control interpretation of predictive coding leads to flexible dynamics.}
\label{sect_app:optimal_control_pc}
Now that we connected predictive coding to the least-control principle, we can further investigate predictive coding through the lens of optimal control. In predictive coding, the neural dynamics follow the gradient flow on the augmented free energy $\mathcal{F}$ of Eq. \eqref{eqn_app:pc_augmented_energy}: 
\begin{equation}\label{eqn_app:pc_optimal_control}
\begin{split}
    \dot{\bphi} &= - \pder{\bphi}{\mathcal{F}}(\bphi, \btheta, \beta)^\top\\
    &= -\bphi + W\sigma(\bphi) + U\bx + \sigma'(\bphi) W^\top \big(\bphi - W\sigma(\bphi) - U\bx \big) - \beta D^\top \nabla_{\by} L(D\bphi)
\end{split}
\end{equation}
with $\sigma'(\bphi) = \spder{\bphi}{\sigma(\bphi)}$. By comparing the above equation to $\dot{\bphi} = f(\phi, \theta) + \psi$, we have the optimal control $\psi = \sigma'(\bphi) W^\top \big(\bphi - W\sigma(\bphi) - U\bx \big) - \beta D^\top \nabla_{\by} L(D\bphi)$ For $\beta \to \infty$, the last term $\beta D^\top \nabla_{\by} L(D\bphi)$ corresponds to an infinitely fast proportional control that clamps the output $y=\bphi_L$ to the teacher $\by^{\mathrm{true}}$. The term $\sigma'(\bphi) W^\top \big(\bphi - W\sigma(\bphi) - U\bx \big)$ then optimally propagates the control at the output level towards the rest of the network. In predictive coding, the term $\bphi - W\sigma(\bphi) - U\bx $ is usually interpreted as an error neuron population $\epsilon$. Hence, in the optimal control view, the error neuron population $\epsilon$ optimally controls the prediction neuron population $\bphi$ via $\psi = \sigma'(\bphi) W^\top \epsilon - \beta D^\top \nabla_{\by} L(D\bphi)$, to reach a controlled network state that exactly matches the output target $\by^{\mathrm{true}}$, while having the smallest possible error signals $\snorm{\epsilon}^2 = \snorm{\bphi - W\sigma(\bphi) - U\bx }^2$. Note that at equilibrium, we have that $\psi = -\ff = \epsilon$. This optimal control view is further corroborated by isolating the error neurons $\epsilon$ at the equilibrium of the neural dynamics \eqref{eqn_app:pc_optimal_control}, leading to $\epsilon\css = \bpsi\css =  - (I-W\sigma'(\bphi\css))^{-\top} D^\top \beta  \nabla_{\by} L(D\bphi\css)$, which satisfies the column space condition of Theorem \ref{theorem:columnspace} with $\bu\css = -\beta  \nabla_{\by} L(D\bphi\css)$, confirming that the error neurons steady state $\bpsi\css$ is an optimal control.

Importantly, using the least-control principle, we can go beyond the gradient-flow dynamics on $F$ and generalize the predictive coding framework to more flexible neural dynamics. First, instead of clamping the output of the network towards the teacher value with an infinitely fast proportional controller, one can use more general output controllers that satisfy the equilibrium condition $\alpha \bu\css + \nabla_{\by} L(D\bphi) = 0$ for $\alpha \to 0$, such as an integral controller.
Next, as we identified the error neurons $\epsilon = \bphi - W\sigma(\bphi) - U\bx$ as being an optimal control $\epsilon\css = \bpsi\opt$ at equilibrium, we can use any dynamics satisfying Theorem \ref{theorem:columnspace} to compute these optimal error neurons and resulting neural dynamics $\dot{\bphi}$. For example, we can use the inversion dynamcis of Eq. \ref{eqn:deq_inversion} and \ref{eqn:rnn_dynamic_inversion} to dynamically compute the error neuron activity. Strikingly, at the steady state, the error neurons computed with these inversion dynamics are indistinguishable from the error neurons computed with the energy-based dynamics of Eq. \ref{eqn_app:pc_optimal_control}, even though the underlying dynamics and hence circuit implementation are completely different. This `circuit invariance' under the optimality condition of Theorem \ref{theorem:columnspace} opens new routes towards finding cortical circuits implementing predictive coding, which is a topic of high relevance in neuroscience \citep{keller_predictive_2018}. Furthermore, the least-control principle allows unifying predictive coding with other existing theories for learning in the brain \citep{meulemans_minimizing_2022, sacramento_dendritic_2018} by uncovering possible equivalences of the underlying credit assignment techniques.

\subsection{Beyond feedforward neural networks.}
The underlying theory of predictive coding, which starts from an acyclic graphical model, is tailored towards feedforward neural networks. However, starting from the augmented energy $F$ of Eq. \eqref{eqn_app:pc_augmented_energy}, we can generalize predictive coding to equilibrium recurrent neural networks with arbitrary synaptic connectivity matrices $W$ and $U$. Although the link with probabilistic latent variable models is not straightforwardly extendable to equilibrium RNNs, the predictive coding interpretation based on error and prediction neurons remains valid. To the best of our knowledge, our experiments testing this new setting of predictive coding for equilibrium RNNs are the first of their kind.

\newpage
\section{Remarks on the least-control principle when using multiple data points} \label{sec_app:multiple_datapoints}
Without loss of generality, we consider a single data point in the formulation of the least-control principle in Section \ref{section:principle}, for clarity of presentation. Here, we discuss in more detail how multiple data points can be incorporated in the least-control principle.

\subsection{Loss defined over finite sample of data points}\label{sec_app:finite_sample_loss}
In many learning problems, as in supervised learning, the loss $L$ is defined over a finite set of data samples: 
\begin{align}
    L(\bphi) = \sum_{b=1}^B L^b(\bphi).
\end{align}
As the losses associated to different samples are different, the corresponding teaching signal will impact the state differently. We therefore consider $\bphi$ to be a concatenation of all the sample-specific states $\{\bphi^b\}_{b=1}^B$. The loss is then equal to
\begin{equation}
    L(\phi) = \sum_{b=1}^B L^b(\bphi^b).
\end{equation}
Similarly, we define $f(\bphi, \btheta)$ as a concatenation of all $f(\bphi^b, \btheta)$, and $\bpsi$ as a concatenation of all the controls $\bpsi^b$. We can then apply the least-control principle on this concatenated quantities.

As the least-control objective of Eq. \ref{eqn:least_control_problem} can be rewritten as the sum $\sum_b \snorm{{\bpsi\opt}^b}^2$, and has constraints that do not interact in between different datapoints $b$, 
its gradient can also be rewritten as a sum over the datapoints: 
\begin{align}
    \der{\btheta}{\calH}(\btheta) = \sum_{b=1}^B  \der{\btheta}{\calH^b}(\btheta)
\end{align}
with $\calH^b(\btheta)$ the least-control objective for a single data point $b$. It follows that one can use standard stochastic or mini-batch optimization methods to minimize the least-control objective $\calH(\btheta)$.

\subsection{Loss defined by expectation over an infinite inputspace}
When there exists an infinite number of data samples (e.g. a continuous input space), the loss can be defined as an expectation over this infinity of data samples: 
\begin{align}
    L(\bphi) = \EE_b[L_b(\bphi)] 
\end{align}
One approach to incorporate this case in the least-control principle is to first sample a finite amount of data samples, and then apply the arguments of Section \ref{sec_app:finite_sample_loss}. 

Another approach is to define the least-control objective as an expectation over the infinity of samples: 
\begin{equation}
    \begin{split}
        \calH(\btheta) &= \EE_b[\calH^b(\btheta)] \\
        \calH^b(\btheta) &= \min_{\bphi^b, \bpsi^b} \snorm{\bphi^b}^2 \quad \text{s.t.} \enspace f(\bphi^b, \btheta) + \bpsi^b = 0, ~~ \nabla_yL(h(\bphi^b) = 0
    \end{split}
\end{equation}

Then, under standard regularity conditions for exchanging the gradient and expectation operator, we can sample gradients of this least-control objective: 
\begin{equation}
    \der{\btheta}{\calH}(\btheta) = \EE\left[ \der{\btheta}{\calH^b}(\btheta) \right]
\end{equation}
with $\sder{\btheta}{\calH^b(\btheta)}$ given by Theorem \ref{theorem:first_order_updates}.

\newpage
\section{The least-control principle for equilibrium RNNs}\label{sec_app:rnn}
Here, we provide additional details on the equilibrium recurrent neural network models, the derivation of the local learning rules and the various controller designs. 
\subsection{Equilibrium recurrent neural network model specifications} \label{sec_app:rnn_specification}
We use an equilibrium RNN with the following free dynamics: 
\begin{align}
    \label{eqn_app:rnn_free}
    \dot{\bphi} = \ff = - \bphi + W \sigma(\bphi) + U \bx,
\end{align}
with $\bphi$ the neural activities of the recurrent layer, W the recurrent synaptic weight matrix, U the input weight matrix and $\sigma$ the activation function. We evaluate the performance of this RNN at equilibrium, by selecting a set of output neurons and measuring its loss $L(D\bphi\fss)$, with $D=[0 ~ \Id]$ a fixed decoder matrix.

\textbf{Learning a decoder matrix.} In practice, we can achieve better performance if we learn a decoder matrix to map the recurrent activity $\bphi$ to the output $\by$, instead of taking a fixed decoder $D=[0 ~ \Id]$. To prevent the loss $L(D\bphi)$ from depending on the learned parameters $\btheta$, we augment the recurrent layer $\bphi$ with an extra set of output neurons, and include a learned decoder matrix $\tilde{D}$ inside the augmented recurrent weight matrix $\bar{W}$, leading to the following free dynamics:
\begin{equation}
\begin{split}
    \label{eqn_app:rnn_free_augmented}
    \dot{\bar{\bphi}} &= f(\bar{\bphi}, \btheta) = - \bar{\bphi} + \bar{W} \sigma(\bar{\bphi}) + \bar{U} \bx, \\
    \bar{W} &= \begin{bmatrix} W & 0 \\ \tilde{D} & 0 \end{bmatrix}, ~~~ \bar{U} = \begin{bmatrix} U \\ 0  \end{bmatrix},
    \end{split}
\end{equation}
with $\bar{\bphi}$ the concatenation of $\bphi$ with a set of output neurons and $\btheta = \{W, U, \tilde{D} \}$ the set of learned parameters. To map the augmented recurrent layer $\bar{\bphi}$ to the output $\by$, we again use a fixed decoder $D=[0 ~ \Id]$ that selects the `output neurons' in $\bar{\bphi}$. At equilibrium, we now have $\by  = D \bar{\bphi} = \tilde{D}\sigma( \bphi)$. Hence, we see that at equilibrium, the augmented system of Eq. \ref{eqn_app:rnn_free_augmented} with a fixed decoder $D$ and hence a loss $L(D\bar{\bphi})$ independent of $\btheta$ is equivalent to the original system of Eq. \ref{eqn_app:rnn_free} with a learned decoder $D$ and loss $L(D\sigma(\bphi))$. We can now use the least-control principle on the augmented system of Eq. \ref{eqn_app:rnn_free_augmented} to learn the weight matrices $\btheta = \{W, U, \tilde{D} \}$, which we show in the next section.

\textbf{Including biases.} In practice, we equip the neurons with a bias parameter $b$, leading to the following free dynamics
\begin{align}\label{eqn_app:rnn_bias_terms}
    \dot{\bphi} = \ff = - \bphi + W \sigma(\bphi) + U \bx + b,
\end{align}
with the set of learned parameters $\btheta = \{ W, U, b\}$. This can be extended to the augmented system of Eq. \ref{eqn_app:rnn_free_augmented} by adding the augmented bias parameters $\bar{b}$, which is a concatenation of the original bias $b$ and the decoder bias $\tilde{b}$. At equilibrium, we now have $\by = D \bar{\bphi} = \tilde{D} \sigma(\bphi) + \tilde{b} $. For ease of notation, we omit the bias terms in most of the derivations in this paper. 

\textbf{Link with feedforward neural networks.} When the recurrent weight matrix has a lower block-diagonal structure, the equilibrium of the RNN corresponds to a conventional deep feedforward neural network. More specifically, taking
\begin{align}
    W = \begin{bmatrix} 0 & 0 & 0 & 0\\ W_2 & 0 & 0 & 0 \\ 0 & \ddots & 0 & 0 \\ 0 & 0 & W_L & 0 \end{bmatrix}, ~~~ U = \begin{bmatrix} W_1 \\ 0 \end{bmatrix}, ~~~ D = \begin{bmatrix} 0 & \Id \end{bmatrix},
\end{align}
gives the following feedforward mappings at equilibrium
\begin{align}
    \phi_l = W_l \sigma(\phi_{l-1}) + b_l, ~~ 1\leq l \leq L
\end{align}
where we structure $\bphi$ into $L$ layers $\bphi_l$, and take $\sigma(\bphi_0) = \bx$ and $\by = D \bphi = \bphi_L$. Hence, we see that the equilibrium RNN model also includes the conventional deep feedforward neural network. Similarly, by further structuring $W$ with blocks of Toeplitz matrices, the equilibrium point corresponds with a convolutional feedforward neural network. 

\textbf{Fixed point iterations.} For computational efficiency, we use fixed point iterations for finding the equilibrium of the free dynamics \eqref{eqn_app:rnn_free}, which we need for the recurrent backpropagation baseline, and for evaluating the test performance of the trained equilibrium RNN. 
\begin{align}
    \nphi = W\sigma(\bphi) + U\bx
\end{align}

\begin{figure}
    \centering
    \includegraphics{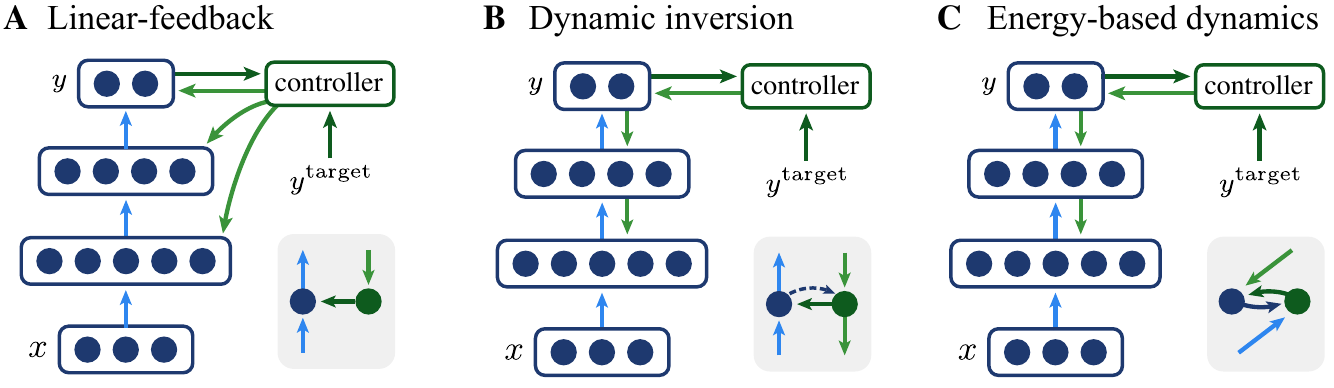}
    \caption{Comparison of the different types of controllers we use on a 2-hidden-layer feedforward neural network in a supervised learning setting. The light blue arrows represent the weights of the network that are used for inference, the light green ones the feedback weights (the $Q$ mapping of Eq.\ref{eqn:generalized_dynamics}), the dark green arrows identity mappings and the dark blue ones transformations through the activation function. Dashed arrows indicates the use of the activation function derivative. The grey rectangle shows how a state neuron $\phi$ and its corresponding control neuron $\psi$ are connected, along with how they are connected to the layers above and below. For example, a light blue bottom-up arrow on the left indicates that the state neurons receive bottom-up input from the state neurons of the layer below through the inference weights. Alternatively, a light green top-down arrow from the right to the left shows that the control signal from the layer above influence the state neurons through the feedback weights. (A) Linear feedback (LF). The output control signal $u$ is directly propagated to the hidden layers through the linear mapping $\psi = Qu$. (B) Dynamic inversion (DI). The control signal is propagated through the hidden layers by leveraging layer-wise feedback mechanisms $\dot{\psi} = -\psi + \sigma'(\bphi)S\psi + D^\top u$ and influences the state variable jointly with bottom up input: $\dot{\phi} = -\phi + W\sigma(\phi) + Ux + \psi$. (C) Energy-based dynamics (EBD). The control signal is here equal to the difference between the current state and the bottom-up input ($\psi = \phi - W\sigma(\phi)-Ux$) and is indirectly propagated to the remote layers through the state variables: $\dot{\phi} = -\psi + \sigma'(\phi)S\psi + D^\top u$.}
    \label{fig_app:circuits}
\end{figure}

\subsection{Deriving a local learning rule for equilibrium RNNs} \label{sec_app:rnn_local_updates}
The least-control principle prescribes to augment the RNN dynamics of Eq. \ref{eqn_app:rnn_free} with an optimal controller $\bpsi$ that leads the neural dynamics to a controlled equilibrium $\bphi\css$, where the output neurons $\by\css = D \bphi\css$ minimize the loss $L$, while being of minimum norm $\snorm{\bpsi\opt}$. This optimal control contains a useful learning signal, which we leverage in Theorem \ref{theorem:first_order_updates} to derive local first-order parameter updates $\Delta \btheta = - \sder{\btheta}{\calH(\btheta)} = {\bpsi\opt}^\top \spder{\btheta}{f(\bphi\opt, \btheta)} $, with $( \bpsi\opt,\bphi\opt)$ the minimizers of the least-control problem of Eq. \ref{eqn:least_control_problem}, i.e. the optimal control and optimally controlled state, respectively. In the next sections, we discuss in detail how to compute $( \bpsi\opt,\bphi\opt)$ for the equilibrium RNN using neural dynamics. Here, we apply the first-order update $\Delta \btheta$ to the equilibrium RNN to arrive at the local synaptic updates of Eq. \ref{eqn:rnn_parameter_updates}.

\textbf{Derivation of the local update rules.} To avoid using tensors in $\spder{\btheta}{f(\bphi\opt, \btheta)}$, we vectorize the weight matrices $W$ and $U$ and use Kronecker products to rewrite the free dynamics of Eq. \ref{eqn_app:rnn_free}:\footnote{For this, we use the following property: $\mathrm{vec}(ABC) = (C^\top \otimes A) \mathrm{vec}(B)$, with $\mathrm{vec}$ the vectorization operator that concatenates the columns of a matrix and $\otimes$ the Kronecker product.}
\begin{align}
    \dot{\bphi} = -\bphi + (\sigma(\bphi)^\top \otimes \Id ) \vec{W} + (\bx^\top \otimes \Id) \vec{U} + b
\end{align}
with $\otimes$ the Kronecker product, $\Id$ the identity matrix and $\vec{W}$ and $\vec{U}$ the vectorized form (concatenated columns) of $W$ and $U$, respectively. Taking $\btheta$ the concatenation of $\vec{W}$, $\vec{U}$ and $b$, we get 
\begin{align}
    \Delta \btheta = {\bpsi\opt}^\top \pder{\btheta}{f}(\bphi\opt, \btheta) = {\bpsi\opt}^\top \begin{bmatrix} (\sigma(\bphi)^\top \otimes \Id ) & (\bx^\top \otimes \Id) & \Id \end{bmatrix}
\end{align}
We can rewrite the resulting update in matrix form again, leading to the following parameter updates for the RNN:
\begin{align}
    \Delta W = \bpsi\opt \sigma(\bphi\opt)^\top, ~~~ \Delta U = \bpsi\opt \bx^\top, ~~~ \Delta b = \bpsi\opt.
\end{align}
Now assuming we have a controller $\bpsi$ that leads the dynamics to the optimally controlled state $(\bpsi\css, \bphi\css) = (\bpsi\opt, \bphi\opt)$, we arrive at the updates of Eq. \ref{eqn:rnn_parameter_updates}. Note that we can use these updates in any gradient-based optimizer. Applying a similar derivation to the augmented RNN dynamics of Eq. \ref{eqn_app:rnn_free_augmented}, we arrive at the following updates for $\btheta = \{W, U, b, \tilde{D}, \tilde{b}\}$:
\begin{equation}\label{eqn_app:rnn_parameter_updates_augmented}
    \Delta W = \bpsi\opt \sigma(\bphi\opt)^\top, ~~~ \Delta U = \bpsi\opt \bx^\top ~~~, \Delta \tilde{D} = \tilde{\bpsi}\opt \sigma(\bphi\opt)^\top, ~~~ \Delta b = \bpsi\opt, ~~~ \Delta \tilde{b} = \tilde{\bpsi}\opt
\end{equation}
with $\tilde{\bpsi}$ the control applied on the set of output neurons within $\bar{\bphi}$, and $\bpsi$ the control applied to the set of recurrent neurons $\bphi$ (i.e. $\bar{\bphi}$ without the output neurons).

\textbf{Controller dynamics.} Theorem \ref{theorem:columnspace} shows that we can compute an optimal control $\bpsi\opt$ by using any controller $\bpsi$ that leads the system to a controlled equilibrium 
\begin{align}
    0 = \ffdown + Q(\bphi\css, \btheta) \bu\css, ~~~ 0 = -\alpha \bu\css - \pder{\by}{L}(h(\bphi\css))^\top
\end{align}
in the limit of $\alpha \to 0$ and with $Q(\bphi\css, \btheta)$ satisfying the column space condition 
\begin{align}
    \mathrm{col}\left[Q(\bphi\css, \btheta)\right] =  \mathrm{row}\left[\pder{\bphi}{h}(\bphi\css) \left(\pder{\bphi}{f}(\bphi\css, \btheta)\right)^{-1} \right].
\end{align}
Under these conditions, we have that $Q(\bphi\css, \btheta) \bu \css$ is an optimal control $\bpsi\opt$. Applied to the equilibrium RNN setting, this result in the following column space condition.
\begin{align} \label{eqn_app:rnn_columnspace}
    \mathrm{col}\left[Q(\bphi\css, \btheta)\right] =  \mathrm{row}\left[D  \big(\Id - W\sigma'(\bphi\css) \big)^{-1} \right],
\end{align}
with $\sigma'(\bphi) := \spder{\bphi}{\sigma(\bphi)}$. In the next sections, we propose various control dynamics $\dot{\bpsi}$ that satisfy these conditions at equilibrium and hence lead to an optimal control $\bpsi\opt$.

\subsection{A controller with direct linear feedback} \label{sec_app:linear_feedback}
We start with the simplest controller, that broadcasts the output controller $\bu$ to the recurrent neurons $\bphi$ with a direct linear mapping $Q$, resulting in the following dynamics: 
\begin{align}\label{eqn_app:rnn_linear_q_dynamics}
    \tau \dot{\bphi} = - \bphi + W \sigma(\bphi) + U \bx + Q \bu, ~~~ \tau_u \dot{\bu} = - \pder{\by}{L}(D \bphi)^\top - \alpha \bu,
\end{align}
with $\tau$ and $\tau_u$ the time constants of the neural dynamics and output controller dynamics, respectively. We visualize an example of such a direct linear feedback controller on a feedforward neural network in Figure \ref{fig_app:circuits}.A. We assume without further discussion that the output controller $\bu$ which integrates the output error $\spder{\by}{L}$ is implemented by some neural integrator circuit, a topic which has been extensively studied in ref. \citepS{goldman_neural_2010}.

\subsubsection{A single-phase learning scheme for updating the feedback weights} 
Theorem \ref{theorem:columnspace} guarantees that this dynamics converges to the optimal control $\bpsi\opt$ if the columnspace condition of Eq. \ref{eqn_app:rnn_columnspace} is satisfied. As $Q$ is now a linear mapping, this condition cannot be exactly satisfied for all samples in the dataset, as the right-hand side of Eq. \ref{eqn_app:rnn_columnspace} is data-dependent, whereas $Q$ is not. Still, we learn the feedback weights $Q$ simultaneously with the other weights $W$ and $U$, to approximately fulfill this column space condition, without requiring separate phases. For this, we add noise to the dynamics and we adapt a recently proposed feedback learning rule for feedforward neural networks \citep{meulemans_minimizing_2022, meulemans_credit_2021} towards our RNN setup, leading to the following system dynamcis and simple anti-Hebbian plasticity dynamics:
\begin{align}\label{eqn_app:rnn_sde_phi}
\tau\dot{\bphi} &= -\bphi + W \sigma(\bphi) + U \bx + Q \bu + s \beps \\
\tau_Q \dot{Q} &= - s\beps \bu_{\mathrm{hp}}^\top - \gamma_Q Q \label{eqn_app:rnn_feedback_update_q}
\end{align}

Here, $\gamma_Q$ is a small weight-decay term, $\beps$ is noise from an Ornstein-Uhlenbeck process, i.e. exponentially filtered white noise with standard deviation $s$ and timeconstant $\tau_{\beps}$. As post-synaptic plasticity signal, we use a high-pass filtered version $\bu_{\mathrm{hp}}$ of the output control $\bu$, to extract the noise fluctuations. 
We model the plasticity with continuous dynamics instead of a single update at equilibrium, as the noise $\beps$ changes continuously. Building upon the work of \citet{meulemans_minimizing_2022, meulemans_credit_2021}, we show in Proposition \ref{prop_app:rnn_feedback_learning_q} that the feedback learning dynamics of Eq. \ref{eqn_app:rnn_feedback_update_q} drives the feedback weights to a setting that satisfies the column space condition of Eq. \ref{eqn_app:rnn_columnspace}, when trained on a single data point. 

\begin{proposition}\label{prop_app:rnn_feedback_learning_q}
Assuming (i) a separation of timescales $\tau \ll \tau_{\beps} \ll \tau_u \ll \tau_Q$, (ii) a small noise variance $s^2 \ll \snorm{\bphi}^2$, (iii) a high-pass filtered control signal $\bu_{\mathrm{hp}} = \bu - \bu_{\mathrm{ss}}$ with $\bu_{\mathrm{ss}}$ the steady state of the first moment of the output controller trajectory $\bu$, and (iv) stable dynamics, then, for a single fixed input $\bx$, the feedback learning dynamics of Eq. \ref{eqn_app:rnn_feedback_update_q} let the first moment $\EE[Q]$ converge towards a setting satisfying the column space condition of Theorem \ref{theorem:columnspace}:
\begin{align}
    \EE\left[Q_{\mathrm{ss}}^\top\right] \propto \frac{\partial^2 L}{\partial \by^2}(D\bphi_{\mathrm{ss}}) D (\Id- W \sigma'(\bphi_{\mathrm{ss}}))^{-1}
\end{align}
with $\bphi_{\mathrm{ss}}$ the steady state of the first moment of the state trajectory $\bphi$, and $\sigma'(\bphi) = \spder{\bphi}{\sigma(\bphi)}$. 
\end{proposition}

\begin{proof}
The general idea behind the feedback learning rule of Eq. \ref{eqn_app:rnn_feedback_update_q} is to correlate the noise fluctuations of the output controller $\bu$ with the noise $\beps$ inside the neurons that caused these fluctuations. As the injected noise propagates through the whole network, and the neural dynamics $\dot{\bphi}$ evolve on a faster timescale compared to the noise $\beps$, the noise fluctuations in the output contain useful information on the network transformation $D(\Id - W \sigma' )^{-1}$ which can be extracted by correlating these fluctuations with the noise $\beps$. We refer the interested reader to Appendix C of \citet{meulemans_minimizing_2022} for a detailed discussion of this type of feedback learning rule for feedforward neural networks. 

First, we define the neural fluctuations $\tphi:= \bphi -  \sphi$ and output controller fluctuations $\tu = \bu - \su$, with $\sphi$ and $\bu$ the steady state of the first moment of $\bphi$ and $\bu$, respectively. Using the assumption that the noise variance $s^2$ is much smaller than the neural activity $\snorm{\bphi}^2$ and that we have stable (hence contracting) dynamics, we perform a first-order Taylor approximation around $\sphi$ and $\bu$, which for simplicity we assume to be exact, leading to the following dynamics:
\begin{align}
    \tau \dot{\tphi} &= -\tphi + W \sigma'(\sphi) \tphi + Q \tu + \beps, \\
    \tau_u \dot{\tu} &= \frac{\partial^2 L}{\partial \by^2}(D\sphi) D \tphi - \alpha \tu
\end{align}
where we used that $-\sphi + W\sigma(\sphi) + U \bx + Q \su = 0$ and $- \partial_{\by} L(D\sphi) - \alpha \su = 0$ by definition. To compress notation, we introduce the variable $\bnu$ as a concatenation of $\tphi$ and $\tu$. 
\begin{equation}
\begin{split}
    \dot{\bnu} &= - A \bnu + B \beps \\
    A& := \begin{bmatrix} \frac{1}{\tau} (\Id - W \sigma'(\sphi)) & -\frac{1}{\tau} Q \\ \frac{1}{\tau_u} \partial^2_{\by} L(D\sphi) D & \frac{\alpha}{\tau_u} \Id  \end{bmatrix} ~~~ B := \begin{bmatrix} \frac{s}{\tau} \Id \\ 0  \end{bmatrix}
    \end{split}
\end{equation}
We solve this linear time-invariant stochastic differential equation using the method of variation of constants \citepS{sarkka_applied_2019}, while assuming that $\bnu =  0$ at time $t_0$, i.e. that the dynamics already have converged to the steady state at $t_0$.
\begin{equation}
\bnu (t) = \int_{t_0}^t \exp \left(-A(t-t')\right) B \beps(t') \mathrm{d} t' .
\end{equation}
Now we turn to the first moment of the feedback weight learning \eqref{eqn_app:rnn_feedback_update_q}, using $\tu = C \bnu $ with $C = [0 ~ \Id]$:
\begin{equation}
\EE \left[ \tu(t) \beps (t)^\top \right] = C \EE\left[\bnu (t) \beps (t) ^\top \right] = \frac{1}{2 \tau_{\beps}} C \int_{t_0}^t \exp \left( - \left(A + \frac{1}{\tau_{\beps}} \Id \right)(t - t')\right) B \mathrm{d}t'
\end{equation}
for which we used that $\EE\left[\beps (t) \beps (t')^\top\right] = \frac{1}{2 \tau_{\beps}} \exp ( \frac{1}{\tau_{\beps}} |t - t'|)$ for an Ornstein-Uhlenbeck process with time constant $\tau_{\beps}$. Now assuming that $t \gg t_0$ and that we have stable dynamics (i.e. $A$ has strictly positive eigenvalues), we can solve the integral of the matrix exponential:
\begin{equation}
    \EE \left[ \tu(t) \beps (t)^\top \right] = \frac{1}{2 \tau_{\beps}} C \left(A + \frac{1}{\tau_{\beps}} \Id \right)^{-1} B 
\end{equation}
We solve the inverse of this $2\times 2$ block matrix analytically (following e.g. \citetS{lu_inverses_2002}), leading to 
\begin{equation}
\begin{split}
    \EE \left[ \tu(t) \beps (t)^\top \right] &= \frac{1}{2 \tau_{\beps}} \frac{s}{\tau} \left[- \Delta ^{-1} \frac{1}{\tau_u} \frac{\partial^2 L}{\partial \by^2}(D\sphi) D \tau (\Id - W \sigma'(\sphi))^{-1}   \right] \\
    \Delta &:= \left(\frac{\alpha}{\tau_u} + \frac{1}{\tau_{\beps}}\right)\Id + \frac{1}{\tau_u} \frac{\partial^2 L}{\partial \by^2}(D\sphi) D \left(\frac{1}{\tau}(\Id - W \sigma'(\sphi))\right)^{-1}\frac{1}{\tau} Q
\end{split}
\end{equation}
Using the separation of timescales $\tau \ll \tau_{\beps} \ll \tau_u$, this simplifies to
\begin{equation}
\begin{split}
    \EE \left[ \tu(t) \beps (t)^\top \right] &\approx -\frac{s}{2\tau_u} \frac{\partial^2 L}{\partial \by^2}(D\sphi) D(\Id - W \sigma'(\sphi))^{-1}
\end{split}
\end{equation}
Remembering that $\tu = \bu_{\mathrm{hp}}$, we get the following plasticity dynamics using Eq. \ref{eqn_app:rnn_feedback_update_q}:
\begin{equation}
    \tau_Q \der{t}{}\EE \left[Q(t)^\top\right] := \tau_Q \der{t}{} \bar{Q}(t) = \frac{s^2}{2\tau_u} \frac{\partial^2 L}{\partial \by^2}(D\sphi) D(\Id - W \sigma'(\sphi))^{-1} - \gamma_Q \bar{Q}(t)
\end{equation}
with $\bar{Q}(t)$ the first moment of $Q(t)$. For $\gamma > 0$, these dynamics are stable and lead to 
\begin{align}
    \EE\left[Q_{\mathrm{ss}}^\top\right] = \frac{s^2}{2\tau_u \gamma_Q} \frac{\partial^2 L}{\partial \by^2}(D\bphi_{\mathrm{ss}}) D (\Id- W \sigma'(\bphi_{\mathrm{ss}}))^{-1}
\end{align}
thereby concluding the proof.
\end{proof}

\textbf{Isolating the postsynaptic noise $\beps$ with a multicompartment neuron model.} The feedback learning rule of Eq. \ref{eqn_app:rnn_feedback_update_q} and the corresponding Proposition \ref{prop_app:rnn_feedback_learning_q} represent the feedback learning idea in its simplest form. However, for this learning rule, the feedback synapses need some mechanism to isolate the postsynaptic noise $\beps$ from the (noisy) neural activity. \citet{meulemans_minimizing_2022, meulemans_credit_2021} solve this by considering a multicompartment model of the neuron, inspired by recent dendritic compartment models of the cortical pyramidal neuron \citep{sacramento_dendritic_2018, guerguiev_towards_2017, kording_supervised_2001}. Here, a feedback (apical) compartment integrates the feedback input $Q\bu$, a basal compartment integrates the recurrent and feedforward input $W\sigma(\bphi) + U\bx$ and a central (somatic) compartment combines the two other compartments together and transmits the output firing rate $\sigma(\bphi)$ to the other neurons. In this model, we can assume that a part of the noise $\beps$ enters through the feedback compartment, and change the feedback learning rule to
\begin{align}
\tau_Q \dot{Q} &= -\big(Q \bu +  s\beps^{\mathrm{fb}}\big) \bu_{\mathrm{hp}}^\top - \gamma_Q Q, \label{eqn_app:rnn_feedback_update_q_multicompartment}
\end{align}
where $\big(Q \bu +  s\beps^{\mathrm{fb}}\big)$ is the feedback compartment activity, and hence locally available for the synaptic updates. As $\EE[Q \bu \bu_{\mathrm{hp}}^\top] = Q\EE[ \bu_{\mathrm{hp}} \bu_{\mathrm{hp}}^\top] = Q\Sigma_u$ with $\Sigma_u$ the positive definite auto-correlation matrix of $\bu_{\mathrm{hp}}$, the result of Proposition \ref{prop_app:rnn_feedback_learning_q} can be adapted to $\EE\left[Q_{\mathrm{ss}}^\top\right] \propto M \partial_{\by}^2 L(D\bphi)_{\mathrm{ss}}) D (\Id- W \sigma'(\bphi_{\mathrm{ss}}))^{-1}$, with $M = (\Sigma_u + \gamma \Id)^{-1}$ a positive definite matrix. This new fixed point of $Q$ also satisfies the column space condition of Eq. \ref{eqn_app:rnn_columnspace}.

\textbf{Improving the parameter updates.} When the column space condition is not perfectly satisfied, as is the case with this direct linear feedback controller when multiple data points are used, the empirical performance can be improved by changing the parameter updates towards
\begin{align}\label{eqn:rnn_parameter_update_heuristic}
\Delta W = \sigma'(\bphi_*) Q\bu_* \sigma(\bphi_*)^\top, \quad \Delta U =  \sigma'(\bphi_*) Q\bu_* \bx^\top
\end{align}
This learning rule uses the local derivative of the nonlinearity $\sigma$ as a heuristic to prevent saturated neurons from updating their weights, which is known to improve performance in feedforward networks \citep{meulemans_minimizing_2022, meulemans_credit_2021}. 

\textbf{Alignment results.} We empirically tested the controller with direct linear feedback and trained feedback weights $Q$ on the MNIST digit recognition task. Figure \ref{fig:alignment} shows that the updates of Eq. \ref{eqn:rnn_parameter_update_heuristic} are approximately aligned to the gradients $-\nabla_{\btheta}\cal{H}(\btheta)$ on the least-control objective, indicating that the feedback weight learning for $Q$ is successful and the controller with direct linear feedback approximates the optimal control. 

\begin{SCfigure}[50][h]
    \centering
    \includegraphics{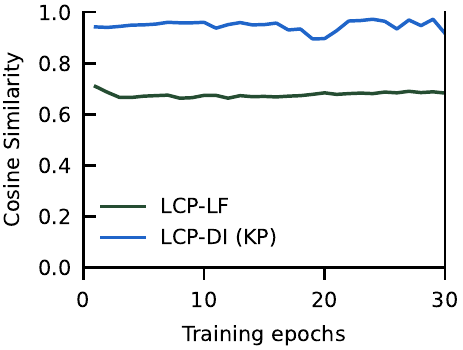}
    \caption{The linear feedback controller and inversion dynamics lead to approximate optimal control. Cosine similarity between the true gradient $-\nabla_{W} \calH(\btheta)$ and approximate updates $\Delta W$ of the least-control principle for training an equilibrium RNN on MNIST: a controller with trained linear feedback connections (LCP-LF) and a controller using dynamic inversion with trained feedback weights via Kolen-Pollack learning (LCP-DI KP). \vspace{0.3cm}}
    \label{fig:alignment}
\end{SCfigure}

\subsubsection{Simulation details} For computational efficiency, we run fixed-point iterations for finding the equilibrium of the controlled neural dynamics shown in Eq. \ref{eqn_app:rnn_linear_q_dynamics}, which in this case is equivalent to simulating the dynamics with the forward Euler method with a simulation stepsize equal to the timeconstant:
\begin{equation}\label{eqn_app:rnn_fixedpoint}
\begin{split}
    \nphi &= W\sigma(\bphi) + U\bx + Q \bu \\
    \nbu &= (1-\alpha) \bu - \pder{\by}{L}(D\bphi)
\end{split}
\end{equation}
We use the fixed points $\bphi\css$ and $\bu\css$ to update the parameters $\btheta =\{W, U, b, \tilde{D}, \tilde{b}\}$ according to Eq. \ref{eqn_app:rnn_parameter_updates_augmented} with $\bpsi_* = Q\bu_*$. For training the feedback weights, we simulate the stochastic differential equations of Eq. \ref{eqn_app:rnn_sde_phi} with the Euler-Maruyama method \citepS{sarkka_applied_2019}, starting from the fixed points computed in Eq. \ref{eqn_app:rnn_fixedpoint}:
\begin{equation} \label{eqn_app:rnn_sde_simulation}
\begin{split}
    \neps &= \beps + \frac{\Delta t}{\tau_{\beps}} \left( - \beps + \frac{1}{\sqrt{\Delta t}} \Delta \bxi \right) \\
    \nphi &= \bphi + \frac{\Delta t}{\tau} \left( -\bphi + W\sigma(\bphi) + U\bx + Q \bu + s \beps  \right)\\
    \nbu &= \bu + \frac{\Delta t}{\tau_u} \left(-\pder{\by}{L}(D\bphi) - \alpha \bu \right)
\end{split}
\end{equation}
with $\Delta t$ the simulation stepsize and $\Delta \bxi \sim \mathcal{N}(0, \Id)$ white Gaussian noise that is drawn independently for each simulation step. The factor $\frac{1}{\sqrt{\Delta t}}$ in front of $\Delta \bxi$ is due specific properties of Brownian motion in the simulation of stochastic differential equations. For computational efficiency, we accumulate the weight updates for $Q$ over the simulation interval, and apply them at the end: 
\begin{equation}
\label{eq:noisy_iter}
    \Delta Q = -\frac{1}{\tau_Q}\Big[\frac{1}{ s^2}\sum_{m} (Q \bu[m] + \beps[m])(\bu[m]-\bu\css)^\top - \gamma Q \Big]
\end{equation}
where $m$ is the simulation step and we use $\bu_{\mathrm{hp}} = \bu - \bu\css$, with $\bu\css$ the fixed point of Eq. \eqref{eqn_app:rnn_fixedpoint}, and where we scale by the noise variance $s^2$ to remove the dependence of the update magnitude on this hyperparameter. For all parameter updates, we use the Adam optimizer. 

\subsection{Using general-purpose dynamics for computing an optimal control}
In the main text (Section~\ref{sec:rnn-general-purpose-experiments}) we discussed the application of our two general-purpose optimal control methods to an equilibrium RNN. Here, we provide an expanded discussion on the neural network architectures each method yields. While we leave the development of more detailed biological models for future work, we comment on how these controlled RNNs may be implemented in cortical circuits, sketching a number of alternative circuit designs. We end this section with providing additional simulation and experimental details. 

\subsubsection{Dendritic error compartments for dynamic inversion}
In Section~\ref{sec:rnn-general-purpose-experiments} we applied our general-purpose dynamic inversion method \eqref{eqn:deq_inversion} to steer equilibrium RNNs towards optimally-controlled states. We restate here the resulting optimal control circuit \eqref{eqn:rnn_dynamic_inversion} and discuss afterwards two possible network implementations, where feedback error signals are represented in dendritic compartments. Assuming without loss of generality that the output controller is of the leaky-integral kind, the resulting network dynamics \eqref{eqn:rnn_dynamic_inversion} is:
\begin{equation}
\label{eqn_app:rnn-dyn-inv}
\begin{split}
    \dot{\bphi} &= - \bphi + W\sigma(\bphi) + U \bx + \bpsi,\\
    \dot{\bpsi} &= -\bpsi + \sigma'(\bphi) S \bpsi + D^\top \bu\\
    \dot{\bu} &= - \alpha u - \nabla_y L(y).
\end{split}
\end{equation}
where $y=D\phi$ is the output neuron state, a subset of the full neuron state $\phi$ determined by a fixed selector matrix $D$. Figure \ref{fig_app:circuits}.B shows an instantiation of this circuit on a feedforward neural network.

In Eq. \ref{eqn_app:rnn-dyn-inv}, the control variable $\psi$ follows its own dynamics $\dot{\bpsi}$, which serves the purpose of inverting $\spder{\bphi}{\ff^\top}$. This can be seen by inspecting the controlled steady state:
\begin{equation}
    \bpsi_* = (\Id -\sigma'(\bphi_*) S)^{-1} D^\top \bu_*.
\end{equation}
If the inverted factor appearing in the equation above would correspond to  $\spder{\bphi}{f(\bphi_*,\btheta)^{-\top}}$, the network would be at an optimally-controlled state, where the weight updates following from our theory assume the Hebbian form
\begin{equation}
\label{eqn_app:rnn-weight-updates}
    \begin{split}
    \Delta W &= \bpsi_* \sigma(\phi_*)^\top\\
    \Delta U &= \bpsi_* x^\top.    
    \end{split}
\end{equation}
For this to be the case, we need to satisfy the symmetry requirement $S = W^\top$. This corresponds to a form of weight transport, assuming that $S$ and $W$ are two distinct sets of synaptic connections.

We present now two alternative circuit implementations which allow maintaining $S \approx W^\top$ using local learning rules, so that the equations governing the change $\Delta W$, $\Delta U$ and $\Delta S$ of all synaptic weights only depend on quantities presumed to be available pre- and post-synaptically and at the same time. In both cases, we build upon previous theories of error-driven cortical learning and assume that our control signal $\bpsi$ corresponds to the apical dendritic activity of a cortical pyramidal cell, and $\bphi$ to its somatic activity. According to this view, apical dendrites encode neuron-specific credit assignment signals $\bpsi_*$ which are responsible for determining which neurons should potentiate or depress their connections, and by how much. We therefore extend previous models of apically-controlled synaptic plasticity  \citep{kording_supervised_2001,guerguiev_towards_2017,richards_dendritic_2019,payeur_burst-dependent_2021,roelfsema_control_2018}, for which there is some experimental evidence \citepS{magee_synaptic_2020}, to the case of recurrently-connected networks.

\textbf{Multiplexed codes.}  Interpreting our apical dendrite dynamics $\dot{\bpsi}$ in literal terms appears to require dendrodendritic synapses which bypass the soma, due to the apical-to-apical term $\bpsi \leftarrow S\bpsi$ occurring in \eqref{eqn_app:rnn-dyn-inv}. While such connections are known to exist in the brain \citepS{shepherd_dendrodendritic_2009}, they are rare. We thus highlight an alternative implementation, which builds on the interesting proposal that pyramidal cells might be capable of transmitting credit assignment signals alongside their intended outputs, by means of a multiplexed code \citep{payeur_burst-dependent_2021,kording_supervised_2001}. The basic premise of this theory is that certain synaptic connections are sensitive to high-frequency bursts of spikes only, which can be triggered by increasing apical activity, whereas other connections are tuned to filter out such events. Each connection type would then differentially reflect either apical or somatic activity, respectively; the required specific tuning to burst events or otherwise can be achieved for example by resorting to different profiles of short-term plasticity for each connection type. Here, we leverage on this idea and propose to replace the dendrodendritic (direct apical-to-apical) connections appearing in \eqref{eqn_app:rnn-dyn-inv} by common (axodendritic) apical-targetting synapses mediated through the soma (i.e., soma-to-apical $\bpsi \leftarrow \bphi$ synapses), tuned to be sensitive to burst events. This design enables keeping the circuit approximately symmetric ($S \approx W^\top$) through local burst-dependent plasticity rules, $\Delta S = \sigma(\bphi_*) \bpsi_* ^\top - \gamma S$ and $\Delta W = \bpsi_* \sigma(\phi_*)^\top - \gamma W$, a form of Kolen-Pollack learning \citep{kolen_back-propagation_1994,akrout_deep_2019}.

Interestingly, the resulting network model differs in a conceptually important -- and experimentally testable -- way from the multiplexing model originally proposed by \citet{payeur_burst-dependent_2021}: in our model, apical control signals should steer (thus, influence) the network state $\bphi$ towards a controlled equilibrium, and therefore should \emph{not} be filtered out from the somatic voltage dynamics $\dot{\bphi}$, while the model of \citet{payeur_burst-dependent_2021} fully multiplexes credit assignment information and the actual computations carried out by the network. Critically, our major functional requirement is that apical-targetting connections be capable of selectively transmitting burst events (e.g., through short-term facilitating plasticity), and we make no further claims about the short-term plasticity of the remaining connections, leaving open the possibility of using this mechanism for other computational purposes unrelated to multiplexing.

\textbf{Dendritic errors via lateral inhibition.} Our network dynamics \eqref{eqn_app:rnn-dyn-inv} is also compatible with an alternative proposal for how apical dendrites may come to encode credit assignment information, then used to steer synaptic plasticity within a neuron \citep{sacramento_dendritic_2018,dold_lagrangian_2019}. In that model, dendritic errors emerge thanks to an auxiliary population of neurons, modeled after a particular class of inhibitory interneurons which preferentially target apical dendrites. These interneurons are set up such that they learn to reproduce the activity of pyramidal neurons. Both pyramidal neurons and interneurons feedback to apical dendrites, but with opposite sign. \citet{sacramento_dendritic_2018} showed that when the interneurons perfectly track pyramidal neurons (sometimes referred to as a `tight' inhibitory-excitatory balance \citepS{hennequin_inhibitory_2017}) and an output error is generated, the remaining steady-state apical activity corresponds to a backpropagated error.

We briefly sketch here a proof-of-concept mapping of our dynamics \eqref{eqn_app:rnn-dyn-inv} onto this model, while leaving more realistic implementations to future work. We first replicate our recurrently-connected neurons with a population of interneurons, whose state we denote by $\delta$, and let both types of neurons feedback onto apical dendrites.  Now, we follow the supplementary analysis of ref.~\citep{sacramento_dendritic_2018} and assume that both populations have exactly the same connectivity and weights, a condition to ensure successful tracking and error propagation, which can be reached with the local weight updates for the interneurons introduced in ref.~\citep{sacramento_dendritic_2018}. Moreover, for simplicity, here we assume our neurons are linear and ignore factors involving the derivative of the transfer function $\spder{\bphi} \sigma$, although it is possible to take such factors into account with a more detailed construction. Taking $\dot{\bpsi} = - \bpsi + S \phi - S \delta + D^\top u$ as our apical dendrite dynamics and inspecting its steady-state yields $\bpsi_* = (\Id - S)^{-1} D^\top \bu_*$, as the interneuron activity perfectly tracks the pyramidal neuron activity $W\bphi + Ux$, but does not include the apical activity $\bpsi$. This is the optimally-controlled equilibrium we sought. As with the multiplexed code described above, approximate symmetry $S \approx W^\top$ can then be maintained with local rules via Kolen-Pollack learning.

\subsubsection{Error neurons for energy-based dynamics}
In the main text and in Sections~\ref{sec_app:energy_based_version}~and~\ref{sec_app:free_energy}, we have seen that the optimal control problem arising under our principle can be generically approached using a (constrained) energy-minimizing dynamics. We restate here the simple gradient-flow dynamics \eqref{eqn:rnn_energy_dynamics} presented in the main text:
\begin{equation}
    \begin{split}
    \dot{\bphi} &= - \bpsi + \sigma'(\bphi) S \bpsi + D^\top \bu,\\
    \bpsi &= \bphi - W\sigma(\bphi) - U \bx,
    \end{split}
\end{equation}
where $\bu$ can be any output controller which enforces the constraint that the loss $L$ is minimized with respect to the output neural activity $y$. The circuit corresponding to those dynamics is shown on Figure~\ref{fig_app:circuits}.C. We now briefly discuss how this system can be implemented using a predictive coding architecture featuring two distinct subpopulations of neurons, generalizing the model of \citet{whittington_approximation_2017} to the RNN setting. In this architecture, $\bpsi$ is interpreted as the activity of a population of prediction error neurons, while $\bphi$ corresponds to the activity of prediction neurons. Error neurons $\bpsi$ act as comparators, which measure deviations of the actual state of prediction neurons $\bphi$ from the expected state $\mu := W\sigma(\bphi) - U\bx$; in this model, such deviations dynamically adjust the predictions $\bphi$ towards an optimally-controlled steady-state in which the output loss is minimal (c.f. Section \ref{sec_app:free_energy}). The critical architectural difference to the models of \citet{whittington_approximation_2017} and \citet{rao_predictive_1999} is that in our case the expected state $\mu$ does not necessarily stem from different network layers. A hierarchically-structured connectivity matrix $W$ is a special case of our model; we allow for top-down, lateral and bottom-up contributions to the expected states $\mu$.

\subsubsection{Training deep equilibrium models with the least-control principle.} \label{sec_app:training_deq}
We use deep equilibrium models that consist out of three main modules: (i) an encoder $g(\bx, \btheta)$ that maps the input $\bx$ to the implicit layer $\bphi$, (ii) an implicit layer defined by $\ff = 0$ where $g(\bx, \btheta)$ is incorporated inside $\ff$, and (iii) a decoder $h(\bphi, \btheta)$ that maps the implicit layer towards the output $\by$. 

When $h(\bphi, \btheta)$ is independent of $\btheta$, we can use the inversion dynamics of Eq. \ref{eqn:deq_inversion} or the energy-based dynamics of Eq. \ref{eqn:generalized_energy_dynamics} to compute an optimal control $\bpsi\opt$ and use the first-order parameter updates of Theorem \ref{theorem:first_order_updates}. To harvest the full power of deep equilibrium models, it is beneficial to train the decoder $h(\bphi, \btheta)$ as well. However, this poses the problem that the loss $L(h(\bphi, \btheta))$ is now dependent on $\btheta$ and hence Theorem \ref{theorem:first_order_updates} cannot be applied out-of-the-box. There are two general approaches to solve this challenge: (i) similar to the vanilla equilibrium RNNs, we can extend the state $\bphi$ with an extra set of output neurons to absorb the $\btheta$-dependent decoder $h(\bphi, \btheta)$ inside of the implicit layer; and (ii) we can use the extended version of the least-control principle derived in Section \ref{sec_app:first_order_gradients} that can handle $\btheta$-dependent losses, to derive the parameter updates. 

\textbf{Augmenting the implicit layer.} We can use an augmented state $\bar{\bphi}$ that concatenates the system states $\bphi$ with an extra set of output neurons $\tphi$. Next we augment the implicit layer $f(\bphi, \btheta)$ towards 
\begin{equation} \label{eqn_app:augmented_system_dynamics}
    \begin{split}
        \bar{f}(\bar{\bphi}, \btheta) = \begin{bmatrix} f(\bphi, \btheta) \\ h(\bphi, \btheta) - \tphi \end{bmatrix}.
    \end{split}
\end{equation}
Now, we can use a fixed linear decoder $D = [0 ~ \Id]$ to select the output states $\tphi$ for the loss $L(D\bar{\bphi})$, which now is independent of $\btheta$. This augmented system is then controlled by $\bar{\bpsi}$ which has the following dynamics when using the dynamic inversion algorithm of Eq. \ref{eqn:deq_inversion}: 
\begin{align}
    \dot{\bar{\bpsi}} = \pder{\bar{\bphi}}{\bar{f}}(\bar{\bphi}, \btheta)^\top \bar{\bpsi} + D^\top u
\end{align}
Combining this with the system dynamics of Eq. \ref{eqn_app:augmented_system_dynamics} leads to the following set of equilibrium conditions
\begin{equation} \label{eqn_app:fixed_points_augmented_deq}
\begin{split}
    0 &=f(\bphi_*, \btheta) + \bpsi_* \\
    0 &= - \tilde{\bphi}_* + h(\bphi_*, \btheta) + \tilde{\bpsi}_* \\
    0&= \pder{\bphi}{f}(\bphi_*, \btheta)^\top\bpsi_* + \pder{\bphi}{h}(\bphi_*, \btheta)^\top \tilde{\bpsi}_*\\
    0&= -\tilde{\bpsi}_* + u_* \\
    0&= \alpha u_* + \pder{y}{L}(\bphi_*)
\end{split}
\end{equation}
with $\tilde{\bpsi}$ the control applied to the augmented output neurons $\tilde{\bphi}$. Grouping the parameters $\btheta$ into the encoder parameters $\btheta_g$, the implicit module parameters $\btheta_f$, and the decoder parameters $\btheta_h$, we get the following parameter gradients: 
\begin{equation}\label{eqn_app:deq_updates_htheta}
    \der{\btheta_f}{}\calH(\btheta) = \bpsi_*^\top \pder{\btheta_f}{f}(\bphi_*, \btheta), \quad \der{\btheta_g}{}\calH(\btheta) = \bpsi_*^\top \pder{\btheta_g}{f}(\bphi_*, \btheta), \quad \der{\btheta_h}{}\calH(\btheta) = \tilde{\bpsi}_*^\top \pder{\btheta_h}{h}(\bphi_*, \btheta)
\end{equation}
Note that due to the general treatment, if the implicit module, encoder or decoder consist of multiple computational layers, the above partial derivatives do not result in a local update of presynaptic activity times postsynaptic control signal, but require instead to backpropagate the credit assignment signal (the control) through the computational layers within the module. The updates however remain first-order updates which is important for computational efficiency, and they remain local in time in contrast to recurrent backpropagation (c.f. Section \ref{sec_app:rbp}). When a design requirement is to have local updates in space, e.g. for biological plausibility, we can augment the system state $\bphi$ with states representing each computational layer in the modules, such that each layer receives a dedicated control, making the resulting updates local in time and space.

\textbf{Using the least-control principle for $\btheta$-dependent losses.} In Section \ref{sec_app:first_order_gradients}, we extend the least-control principle to $\btheta$-dependent losses, and among others derive a first-order update when the decoder $h(\bphi, \btheta)$ is $\btheta$-dependent, which we restate below. 
\begin{equation}
    \begin{split}
        \der{\btheta}{}\calH(\btheta)
        &= -\bpsi\css^\top \pder{\btheta}{f}(\bphi\css, \btheta) - \bu\css^\top
        \pder{\btheta}{h}(\bphi\css, \btheta)
    \end{split}
\end{equation}
This first-order learning rule equals the gradient when the inversion dynamics \eqref{eqn:deq_inversion} or energy-based dynamics \eqref{eqn:generalized_energy_dynamics} are used for computing the optimal control $\bpsi\opt$, in the limit of $\alpha \to 0$.

As the strategy of augmenting the implicit layer to have a loss independent of $\btheta$ only uses the theory explained in the main manuscript, we use this strategy for our experiments.
\subsubsection{Simulation details}
\textbf{Equilibrium RNNs.} We use the following fixed point iterations for finding the equilibrium of the inversion dynamics of Eq. \ref{eqn:rnn_dynamic_inversion}:
\begin{equation}
    \begin{split}
        \nphi &= W\sigma(\bphi) + U\bx + \bpsi \\
        \npsi &= \sigma'(\bphi)S\bphi + D^\top \bu \\
        \nbu &= (1-\alpha)\bu - \pder{\by}{L}(D\bphi)^\top
    \end{split}
\end{equation}

The fixpoint iterations terminates for a given input when either the relative change in the norm of the dynamic states $\Phi=\{\phi\}$ or $\Phi=\{\phi, \psi, \bu\}$  is less than some threshold $\epsilon$, i.e. $\frac{\|\Phi_t-\Phi_{t+1}\|^2}{\|\Phi_t\|\|\Phi_{t+1}\|}\leq \epsilon$, or when some iteration budget \texttt{max\_steps} is exhausted.

When we learn the feedback weights $S$ with Kolen-Pollack learning (LCP-DI KP), we use the following weight updates for $S$ and $W$: 
\begin{equation}
    \label{eq:kp_iter}
    \Delta W = \bpsi\css \sigma(\bphi)^\top - \gamma W, ~~~\Delta S = \sigma(\bphi) \bpsi\css^\top - \gamma S
\end{equation}
with $\gamma$ the weight decay parameter. When we do not train the feedback weights $S$ (LCP-DI), we fix them towards the transpose of the recurrent weights $S=W^\top$ and we use standard weight update for $W$ without weight decay \eqref{eqn:rnn_parameter_updates}.

\textbf{Deep equilibrium models.} We generalize the above fixed point iterations to more complex recurrent architectures, by using the abstract inversion dynamics of Eq. \ref{eqn:deq_inversion} and the equilibrium equations of Eq. \ref{eqn_app:fixed_points_augmented_deq}. Note that we directly substitute $\tilde{\bpsi}_* = u_*$ in the fixed point equations, to eliminate an equation.  
\begin{equation}
    \begin{split}
    \label{eq:deq_iter}
        \nphi &= \bphi + \ff + \bpsi \\
        \tilde{\bphi}_{\mathrm{next}} &= h(\bphi, \btheta) + \bu \\
        \npsi &= \bpsi + \pder{\bphi}{f}(\phi, \theta)^\top\bpsi + \pder{\bphi}{h}(\bphi, \btheta)^\top \bu \\
        \nbu &= (1-\alpha) \bu - \pder{\by}{L}(\tilde{\bphi})^\top
    \end{split}
\end{equation}
We use the first-order updates discussed in Eq. \ref{eqn_app:deq_updates_htheta} for training the DEQs, with $\tilde{\bpsi}_* = u_*$.

\subsection{Additional experimental details}
Here, we provide additional details on the experiments introduced in Section \ref{section:RNN}, including the architecture of the networks, the used hyperparameters, possible data preprocessing and the used optimization algorithms.  

\subsubsection{Training feedforward networks and equilibrium RNNs on MNIST}

We evaluate the least-control principle combined with the various feedback circuits introduced in Section \ref{section:RNN} on the MNIST digit classification task  \cite{lecun_mnist_1998}, and compare their performance to recurrent backpropagation (RBP).

\textbf{Data preprocessing and augmentation.} No data preprocessing or augmentation was employed in our experiment.

\textbf{Feedforward network architecture.} We use two hidden layers of 256 neurons to have the same amount of parameters as the equilibrium RNN architecture used for MNIST. In the hidden layers, we use $\tanh$ nonlinearities and on the output layer a softmax nonlinearity.

\textbf{RNN architecture.} 
We use 256 recurrent units in the RNN used for the results of table \ref{table:mnist}. We include a trained decoder $\tilde{D}$ inside the recurrent layer, as detailed in Section \ref{sec_app:rnn_specification} and Eq. \ref{eqn_app:rnn_free_augmented}. 
For the results shown in Fig. \ref{fig:overparam}, we run the equilibrium RNN on MNIST with $[23, 32, 45, 64, 90, 128, 181, 256]$ recurrent units. We use the $\tanh$ nonlinearity for all recurrent units, and include bias terms according to Eq. \ref{eqn_app:rnn_bias_terms}. 

\textbf{Optimization.} For the classification benchmark, we train the models for 30  of the MNIST train set, with a minibatch size of 64 and with the Adam optimizer \citepS{kingma_adam_2015}. When we learn the feedback weights $S$ with Kolen-Pollack learning, we apply the weight decay after every gradient update, i.e. we apply it outside of the optimizer. 

For training the feedback weights $Q$ of the linear feedback controller, we run the stochastic dynamics in Eq. \ref{eqn_app:rnn_sde_simulation} for \texttt{max\_steps} timesteps, and we update $Q$ following Eq. \ref{eq:noisy_iter} by averaging $\Delta Q$ over the minibatch. For the simulation of the stochastic dynamics, we use $\Delta t=\tau'^2$, $\tau=\tau'^2$, $\tau_\epsilon=\tau'$ and $\tau_u=1$ for some hyperparameter $\tau'$. Table \ref{tab:hps-mnist} contains details on the used and scanned hyperparameters. 

\textbf{Energy-based dynamics} LCP-EBD models were trained by optimizing the energy for \texttt{max\_steps} steps using the Adam optimizer \citep{kingma_adam_2015} with learning rate \texttt{ebd\_inner\_lr}. Due to computational efficiency reasons, we initialized $\phi$ at the free-dynamic equilibrium before the energy optimization for the RNN experiments. For the feedforward networks, we initialized $\phi$ at zero. As initializing the network to its free equilibrium results effectively in a two-phase algorithm, we investigated whether the same fixed points are reached when initializing at the free equilibrium or at a zero. We observed that both initializations converge to the same fixed point, but that the zero initialization needs many more iterations to reach the fixed point, compared to the combined free-phase iterations and controlled dynamics iterations starting from the free equilibrium.

For the overparametrization experiments, we use the same hyperparameters except for a learning rate of $0.001$ for both RBP and LCP-DI, with a cosine annealing scheduler \citeS{loshchilov_sgdr_2017} annealing the learning rate down to $0.0001$ over a total training time of 50 epochs.

For all RNN experiments, we use gradient clipping whereby the norm of the gradient, computed over all gradients concatenated into a single vector, is normalized to some value (chosen to be 10) whenever above it. The gradient clipping was particularly crucial for a stable training of the RBP model.

\textbf{Additional results.} In table \ref{table:mnist_detailed} we show  additional experimental results for the MNIST experiments. Note that the hyperparameter was chosen to maximise the test set accuracy. 

\textbf{Robustness to $\alpha$.} While our theory characterizes the strong control limit $\alpha \rightarrow 0$, some interesting properties still hold outside this regime. In Proposition S9, we showed that under a strict alignment condition on the feedback mapping $Q$ (Eq. 76), the controlled steady state is a local minimizer of the augmented energy $F(\phi, \theta, \alpha) = ||f(\phi, \theta)||^2 + \alpha^{-1} L(\phi)$. This strict alignment condition holds for energy-based (Eq. 10) and inversion (Eq. 8) dynamics. This has two consequences. First, the update prescribed by Theorem 1 follows the gradient of some surrogate objective function (cf. the end of Section S2.6) which is closely related to the least-control objective (cf. Remark 1 in Section S2.5). Additionally, this surrogate objective (for $\alpha \neq 0$) is still a meaningful one for the original learning problem: similar results to Propositions 3 and 4 hold for this objective (c.f. Section \ref{sec_app:proofs_LCP}). However, the feedback mapping flexibility of Theorem 2 does not hold anymore outside the strong control limit $\alpha \rightarrow 0$, as the column space condition of Eq. 7 is now replaced by the strict alignment condition of Eq. 76.

To corroborate these theoretical insights and to verify whether the least-control principle exhibits robust performance for a wide range of values of $\alpha$, we repeat the MNIST experiments for LCP-DI with a feedforward neural network for varying values of the controller leakage $\alpha$. Note that for $\alpha=0$, a softmax output layer combined with one-hot targets would lead to an infinite control, as a softmax needs an infinite input to produce a one-hot output. This can be remedied by either increasing $\alpha$ or by using \textit{soft targets}, i.e. taking $1-\epsilon$ as target for the correct class, and $\epsilon/(N-1)$ for the other $N-1$ classes. As we want to include $\alpha=0$ in this line of experiments, we resort here to soft targets with $\epsilon = 0.01$, whereas for the experiments in the main text we increase the controller leakage to $\alpha=0.1$ (c.f. Table \ref{tab:hps-mnist}). Table \ref{tab_app:alpha_robust} shows that the performance of LCP-DI is robust to a wide range of values for $\alpha$. 

\begin{table}[h!]
    \centering
        \caption{Test set accuracy on MNIST for a feedforward network (2x256 neurons) trained by LCP-DI with different values for the controller leakage $\alpha$. We used the hyperparameters mentioned in Table \ref{tab:hps-mnist}, combined with forward Euler integration of the dynamical equations with timestep $\texttt{dt} = 0.01$. }
    \label{tab_app:alpha_robust}
    \begin{tabular}{lll}
        \toprule
                & Test accuracy \\ 
        \midrule
        $\alpha = 0$ &  {98.09 \%} \\ 
        $\alpha = 0.01$ &{98.04 \%} \\ 
        $\alpha = 0.1$ & {98.10 \%} \\ 
        $\alpha = 1$ & {97.99 \%} \\
        \bottomrule
    \end{tabular}

\end{table}

\begin{landscape}
\begin{table}[bt]
\begin{center}
\caption{Hyperparameter search space for the MNIST Feed forward and RNN experiment. A grid search with one seed for each configuration was done to find the best parameters, which are marked in bold. The same set of hyperparameters for both RNN and feedforward models yielded the best performance, except for the LCP-EBD model where the RNN model needed shorter inner optimization steps due the free-phase initialization.}
\label{tab:hps-mnist}
\begin{tabular}{@{}llllll@{}}
\toprule
Hyperparameter             & RBP                                     & LCP-DI                                  & LCP-DI (KP)                             & LCP-LF                      & LCP-EBD                                  \\ \midrule
\texttt{num\_epochs}       & 30                                      & 30                                      & 30                                      & 30  & 30                                      \\
\texttt{batch\_size}       & 64                                      & 64                                      & 64                                      & 64   & 64                                      \\
$\alpha$                   & -                                       & 0.1                                     & 0.1                                     & 0.1  & 0.1                                       \\
\texttt{lr}                & $\{3.10^{-4}, \mathbf{10^{-3}}, 3.10^{-3}\}$&$\{3.10^{-4}, \mathbf{10^{-3}}, 3.10^{-3}\}$& $\{3.10^{-4}, \mathbf{10^{-3}}, 3.10^{-3}\}$& $\{3.10^{-4}, \mathbf{10^{-3}}, 3.10^{-3}\}$& $\{3.10^{-4}, \mathbf{10^{-3}}, 3.10^{-3}\}$            \\
\texttt{optimizer}         & ADAM                                    & ADAM                                    & ADAM                                    & ADAM    & ADAM                                    \\
\texttt{scheduler}         &$\{none, \mathbf{cos}\}$                 &$\{none, \mathbf{cos}\}$                 &$\{none, \mathbf{cos}\}$                 &$\{none, \mathbf{cos}\}$ &$\{none, \mathbf{cos}\}$                 \\
\texttt{max\_steps}        & $\{ 50, \mathbf{200}, 800 \}$           & $\{ 50, 200, \mathbf{800} \}$           & $\{200, \mathbf{800} \}$                & $\{200, \mathbf{800} \}$   & 800 (FF) / 200 (RNN)               \\
$\epsilon$                 & $\{10^{-5}, \mathbf{10^{-4}}, 10^{-3}\}$& $\{\mathbf{10^{-6}}, 10^{-5}, 10^{-4}\}$& $\{\mathbf{10^{-6}}, 10^{-5}, 10^{-4}\}$& $\{\mathbf{10^{-6}}, 10^{-5}, 10^{-4}\}$ & -\\
$\gamma$                   & -                                       & -                                       & $\{\mathbf{10^{-6}}, 10^{-5}, 10^{-4}, 10^{-3}\}$& -                              & -\\
$s$                        & -                                       & -                                       & -                                       & $\{\mathbf{0.01}, 0.1\}$                & -\\
$\tau'$                    & -                                       & -                                       & -                                       & $\{0.1, \mathbf{0.2}\}$                & -\\
$\tau_Q$                   & -                                       & -                                       & -                                       & $\{100, 1000,10000, \mathbf{100000}\}$     & -\\
$\gamma_Q$                 & -                                       & -                                       & -                                       & $\{10^{-4}, 10^{-3}, \mathbf{10^{-2}},0.1\}$& -\\
\texttt{ebd\_inner\_optimizer}           & -                                       & -                                       & -                                       &-                                    & ADAM\\
\texttt{ebd\_inner\_lr}           & -                                       & -                                       & -                                       &-                                    & 0.01 (FF) / 0.001 (RNN) \\
\bottomrule
\end{tabular}
\end{center}
\end{table}
\end{landscape}

\begin{table}[]
\begin{tabular}{l|c|c|c|c}
\multirow{2}{*}{}  & \multicolumn{2}{c|}{Train set}                     & \multicolumn{2}{c}{Test set}                        \\ \cline{2-5} 
& \multicolumn{1}{l|}{NLL}      &Accuracy      & \multicolumn{1}{l|}{NLL}      & Accuracy      \\ \Xhline{3\arrayrulewidth}
FF-LCP-LF& \multicolumn{1}{l|} {0.00087 ± 0.00084}&{99.98 ± 0.03} & \multicolumn{1}{l|}{0.1910 ± 0.0055}&{97.73 ± 0.07}\\ 
FF-LCP& \multicolumn{1}{l|} {0.00004 ± 0.00002}&{100.00 ± 0.00} & \multicolumn{1}{l|}{0.1869 ± 0.0154}&{98.11 ± 0.07}\\ 
FF-LCP-KP& \multicolumn{1}{l|} {0.00003 ± 0.00002}&{100.00 ± 0.00} & \multicolumn{1}{l|}{0.2041 ± 0.0173}&{98.14 ± 0.09}\\ 
FF-LCP-EBD& \multicolumn{1}{l|} { 0.02645 ± 0.00183}&{99.90 ± 0.00} & \multicolumn{1}{l|}{0.0848 ± 0.0020}&{98.00 ± 0.03}\\ 
FF-BP& \multicolumn{1}{l|} {0.00002 ± 0.00001}&{100.00 ± 0.00}& \multicolumn{1}{l|}{0.0879 ± 0.0027}&{98.29 ± 0.14}\\ 
\midrule
RNN-LCP-LF& \multicolumn{1}{l|} {0.0021 ± 0.0019} & {99.96 ± 0.04}  & \multicolumn{1}{l|}{0.1917 ± 0.0068} & {97.70 ± 0.11} \\
RNN-LCP& \multicolumn{1}{l|} {0.0121 ± 0.0013} & {99.80 ± 0.02} & \multicolumn{1}{l|}{0.1630 ± 0.0007} & {97.58 ± 0.12} \\
RNN-LCP-KP& \multicolumn{1}{l|} {0.0105 ± 0.0011} & {99.84 ± 0.04} & \multicolumn{1}{l|}{0.1739 ± 0.0012} & {97.75 ± 0.11} \\
RNN-LCP-EBD& \multicolumn{1}{l|} {0.0103 ± 0.0026} & {99.74 ± 0.05}  & \multicolumn{1}{l|}{0.1324 ± 0.0079} & {97.60 ± 0.15} \\
RNN-RBP& \multicolumn{1}{l|} {0.0036 ± 0.0047}&{99.94 ± 0.09}& \multicolumn{1}{l|}{0.0961 ± 0.0022}&{97.87 ± 0.19}\\  
\end{tabular}
\centering
\vspace{0.25cm}
\caption{Classification accuracy ($\%$) and negative log likelihood (NLL) on the MNIST train and test set, of a Feed forward (FF) and RNN model trained by backpropagation (BP/RBP) and least control (LCP). The reported mean±std is computed over 3 seeds. }
\label{table:mnist_detailed}
\end{table}

\subsubsection{Training convolutional networks and deep equilibrium models on CIFAR-10}
\label{sec_app:conv-net-details}
Here, we train a convolutional network and Deep Equilibrium model and evaluate their performances on the CIFAR-10 image classification benchmarks \citep{krizhevsky_learning_2009}. We compare the performance of the models trained via either RBP or LCP.

\textbf{Data preprocessing and augmentation.} We normalize all images channel-wise by subtracting the mean and dividing by the standard deviation, which are computed on the training dataset. We employed no data augmentation in our experiments.

\textbf{Convolutional network architecture.} We use a convolutional network with 3 convolutional layers of kernel 5-by-5 and channel sizes 96-128-256  followed by two hidden fully connected layers of 2048 neurons. All layers are followed by ReLU non linearity. We do not use batch normalization (for simplicity, we dispense with normalization layers altogether). To keep the architecture as simple as possible we do not use max-pooling units except after the first convolutional layer, and resort for all other convolutional layers to simple strided convolutions to implement downsampling.

\textbf{Deep equilibrium model architecture.}  The model consists of 3 sub-modules that are invoked sequentially: an encoder, an implicit module, and a decoder. The encoder consists in a single convolutional layer with bias, followed by a Batchnorm layer \citepS{ioffe_batch_2015}. The implicit module finds the fixed point of the following equation:
\begin{equation}
    f(\phi, e, \theta) = -\phi+\mathrm{norm}(\mathrm{ReLU}(\phi+\mathrm{norm}(e+c_2(\mathrm{norm}(\mathrm{ReLU}(c_1(\phi)))))))
\end{equation}
where both $c_1, c_2$ are convolutional layers, $\mathrm{norm}$ are group normalization layers \citeS{wu_group_2018}, and $e$ is the output of the encoder module $g(x, \btheta)$.
Finally, the decoder module contains a batchnorm layer, followed by an average-pooling layer of kernel size 8-by-8, and a final linear classification layer.
All convolutional layers use 48 channels, a filter size of 3-by-3, padding of 1, and are without biases unless specified otherwise.

\textbf{Optimization.}  We train the model by finding the equilibrium of the controlled dynamics using the fixed point iterations of Eq. \ref{eq:deq_iter} and we update the parameters with the first order updates of Eq. \ref{eqn_app:deq_updates_htheta}. Note that the implicit module contains multiple convolutional layers, hence the partial derivatives in Eq. \ref{eqn_app:deq_updates_htheta} backpropagate the credit assignment signal $\bpsi$ towards the parameters of the different computational layers. We use automatic differentiation for this. If local parameter updates in space are a design requirement, one can augment the system state $\bphi$ with additional states for the output of each convolutional layer, resulting in dedicated control signals for each computational layer, similarly to Section \ref{sec_app:training_deq} where we augment the state $\bphi$ with output neurons to incorporate the decoder. Both RBP and LCP-DI are trained on 75 epochs for the Deep Equilibrium model, while the Convolutional network is trained on 30 epochs. For both models, we use the stochastic gradient descent optimizer with a minibatch size of 64, combined with an annealing learning rate using the cosine scheduler \citeS{loshchilov_sgdr_2017} which divides the learning rate by 10 over the training epochs. We used a controller leakage of $\alpha=1$ to enable a faster convergence to the controlled equilibrium and hence reduce the computational requirements. Table \ref{tab:hps-cifar} contains further details on the hyper parameters.

\textbf{Additional results.} In table \ref{table:cifar_detailed} we show  additional experimental results for the CIFAR-10 experiments. Note that the hyperparameter was chosen to maximise the test set accuracy. As can be seen, the training loss is not close to 0, indicating that the model considered here is not overparametrized enough for perfectly fitting the training set. Yet, LCP still manages to optimize for the training loss, and crucially, in such a way that generalizes competitively.

\begin{table}[bt]
\begin{center}
\caption{Hyperparameters used for the CIFAR-10 experiment, for the convolutional (CONV) and deep equilibrium model (DEQ) for backpropagation (-BP), recurrent backpropagation (-RBP) and LCP-DI training algorithms .}
\label{tab:hps-cifar}
\begin{tabular}{@{}lllll@{}}
\toprule
Hyperparameter            &CONV-BP& CONV-LCP-DI & DEQ-RBP                                     & DEQ-LCP-DI                                  \\ \midrule
\texttt{num\_epochs} &30& 30 & 75                                      & 75                                      \\
\texttt{batch\_size}   &64& 64   & 64                                      & 64                                      \\
$\alpha$                 &- & 0.1  & -                                       & 3          \\
\texttt{lr}               &0.1&0.1 & 0.03    &  0.05     \\
\texttt{optimizer}  &SGD &SGD & SGD                                    & SGD                                    \\
\texttt{scheduler} &cos& cos& cos                                    & cos                                    \\
\texttt{max\_steps}  &-&800      & 200          & 800           \\
$\epsilon$             &-& $10^{-6}$   & $10^{-4}$& $10^{-4}$\\
\texttt{Gradient clipping}               &No&No  & Yes& Yes\\
\bottomrule
\end{tabular}
\end{center}
\end{table}

\begin{table}[]
\begin{tabular}{l|c|c|c|c}
\multirow{2}{*}{}  & \multicolumn{2}{c|}{Train set}                     & \multicolumn{2}{c}{Test set}                        \\ \cline{2-5} 
& \multicolumn{1}{l|}{NLL}      &Accuracy      & \multicolumn{1}{l|}{NLL}      & Accuracy      \\ \Xhline{3\arrayrulewidth}
DEQ-RBP& \multicolumn{1}{l|} {0.2403±0.0027}&{92.50±0.14}& \multicolumn{1}{l|}{0.6172±0.0097}&{80.14±0.20}\\ \hline
DEQ-LCP& \multicolumn{1}{l|} {0.3218±0.0062}&{89.55±0.17} & \multicolumn{1}{l|}{0.5874±0.0077}&{80.26±0.17}\\ 
\end{tabular}
\centering
\vspace{0.25cm}
\caption{Classification accuracy ($\%$) and negative log likelihood (NLL) on the CIFAR-10 train and test set, of a DEQ model trained by recurrent backpropagation (RBP) and least control (LCP). The reported mean±std is computed over 3 seeds. }
\label{table:cifar_detailed}
\end{table}

\subsubsection{Neural implicit representations}
Inspired by the recent success of using implicit layers for implicit neural representations \citep{huang_implicit2_2021}, we investigate using the least-control principle in this setup for learning. Implicit neural representations (INRs) are a quickly emerging field in computer vision and beyond with the aim to represent any signal through a continuous and differentiable neural network. Here we strictly follow the experimental setup of \citet{huang_implicit2_2021} and leverage the continuity of implicit neural representations to generalize on trained images. In particular, given an Image $I$, we train the INR to represent $25\%$ of the pixels (we simply jump over every second pixel and every second row) while testing the networks prediction on $25\%$ of unknown pixels and measuring the peak signal-to-noise ratio (PSNR). This can be regarded as simple image inpainting. 

\textbf{Architecture.} Compared to \citet{huang_implicit2_2021}, we slightly modify the implicit layer resulting in the following free dynamics
\begin{equation}
    f_{\text{INR}}(\bphi; x) = -\bphi + \mathrm{ReLU}(W\bphi) + Ux.
\end{equation}
We include a trained decoder matrix $\tilde{D}$ inside the implicit layer, as discussed in Section \ref{sec_app:rnn_specification}. 
We use fixed point iterations on the inversion dynamics \eqref{eq:deq_iter} for finding an optimal control, and use the weight updates for the equilibrium RNNs \eqref{eqn_app:rnn_parameter_updates_augmented}.
Note that we did not include spectral normalization or a learning rate scheduler to train our models as in \citet{huang_implicit2_2021}. Despite these simplifications we obtain comparable results across methods and models. 

\textbf{Hyperparameters.} Note the small discrepancy between the INR models trained with RBP and LCP-DI in Table \ref{tab:hps-inr}. Interestingly, we had an easier time scanning hyperparameters to obtain competitive performance and stable forward dynamics when using LCP-DI compared to RBP, for which we had to include gradient clipping and weight decay. See Table \ref{tab:hps-inr} for the hyperparameters found when repeating the the grid search identical to the one seen in Table \ref{tab:hps-mnist}. Following \citet{huang_implicit2_2021}, we include an extra hyperparameter (\texttt{Omega}), which is a scalar multiplier applied to all neural activity, and we scale the initialization of the weights accordingly with $1/$\texttt{Omega}. 

\textbf{Additional results.} Table \ref{tab_app:inr_results} shows the full experimental results for learning neural implicit representations on the natural images and text dataset introduced by \citet{huang_implicit2_2021}. As we changed the network architecture slightly, we include the SIREN and iSIREN results of \citet{huang_implicit2_2021} for comparison. We see that the least-control principle achieves competitive performance on this task, compared to RBP and the results of \citet{huang_implicit2_2021}.

\begin{table}[]
\caption{Peak signal-to-noise ratio (PSNR; in dB) for all models on the natural image and text generalization task \citep{huang_implicit2_2021}. The reported mean ± std is taken
over the individual PSNR of the 16 images. SIREN and iSIREN results are taken from \citet{huang_implicit2_2021}. Note that while our INR model is trained without spectral normalization, we reach competitive results. Interestingly, LCP does perform on par with RBP while not relying on weight-decay, learning-rate scheduling or gradient clipping.}
\label{tab_app:inr_results}
\begin{tabular}{l|c|c|c|c}
\multirow{2}{*}{}  & \multicolumn{2}{c|}{Natural}                     & \multicolumn{2}{c}{Text}                        \\ \cline{2-5} 
                   & \multicolumn{1}{l|}{1L-256D}      & 1L-512D      & \multicolumn{1}{l|}{1L-256D}      & 1L-512D      \\ \hline
SIREN (input inj.) & \multicolumn{1}{l|}{22.88 ± 3.0}  & 24.52 ± 3.28 & \multicolumn{1}{l|}{24.54 ± 2.19} & 25.69 ± 2.18 \\ \hline
iSIREN             & \multicolumn{1}{l|}{24.28 ± 3.37} & 24.92 ± 3.58 & \multicolumn{1}{l|}{26.06 ± 2.18} & 26.81 ± 2.09 \\ \Xhline{3\arrayrulewidth}
INR-RBP         & \multicolumn{1}{l|}{24.72 ± 3.90} &      {25.47 ± 4.16}       & \multicolumn{1}{l|}           {25.34 ± 1.67}    &      {26.53 ± 2.40}         \\ \hline
INR-LCP         & \multicolumn{1}{l|} {24.25 ± 2.72}        &      {25.11 ± 2.90}            & \multicolumn{1}{l|}    {26.71 ± 2.40}          &      {27.53 ± 2.00}       \\ 
\end{tabular}
\centering
\vspace{0.25cm}

\end{table}

\begin{table}[bt]
\begin{center}
\caption{Hyperparameter used for the INR experiments.}
\label{tab:hps-inr}
\begin{tabular}{@{}lllll@{}}
\toprule
Hyperparameter             & RBF-Natural            & LCP-Natural                              & RBF-Text            & LCP-Text                              \\ \midrule
\texttt{iteration per image}  & 500        & 500                                      & 500                              & 500                                      \\
$\alpha$                   & -     & 0 & -   & 0          \\
\texttt{lr}                & 0.0001    & 0.0001  & 0.0001 & 0.0001    \\
\texttt{optimizer}  & Adam   & Adam    & Adam    & Adam                                                           \\
\texttt{max\_steps}        & 50          & 50   & 50  & 50          \\
$\epsilon$                 & $10^{-5}$& $10^{-5}$& $10^{-5}$& $10^{-5}$\\
\texttt{weight decay}                 & 0.1& 0.1& 0.1& $0.1$\\
\texttt{Omega}                 & 15 & 35& 35 & 50\\
\texttt{Gradient clipping}                 & Yes& No& Yes& No\\

\bottomrule
\end{tabular}
\end{center}
\end{table}

\subsection{Computational cost}\label{sec_app:comp_cost}
Here, we include a computational comparison between LCP-DI and RBP on the MNIST experiments for equilibrium RNNs. The table below shows the amount of iterations (phase length), the computation time for one single iteration (silicon time; obtained by using the NVIDIA/PyProf profiler offline), and the total training time for both algorithms. Note that RBP has two phases, whereas LCP only one. We see that LCP requires more iterations compared to the combined two phases of RBP. Furthermore, RBP requires significantly fewer computations per iteration, as LCP needs to compute both the system states and control states every iteration, resulting in more matrix-vector products. Consequently, LCP requires significantly more compute compared to RBP on this line of experiments, when simulating on standard digital hardware. Hence, despite its desirable single-phase property, the added computational cost per iteration results in a significantly higher computational cost of LCP. We observed that increasing the controller leakage $\alpha$ decreases the number of iterations required for reaching the controlled equilibrium significantly, and hence also decreases the computational cost. 

\begin{table}[bt]
\begin{center}
\caption{Computational cost}
\label{tab:comp_cost}
\begin{tabular}{@{}lll@{}}
\toprule
            & LCP-DI           & RBP  \\ \midrule
Average 1st phase length  &     483    & 77\\
Average silicon time per iteration (1nd phase) (ns)                 &221634     & 28735\\
Average 2nd phase length               &   -  & 72   \\
Average silicon time per iteration (2nd phase) (ns) & -   & 32479\\
\midrule
Total training time (s)      &35280 &  1683 \\
\bottomrule
\end{tabular}
\end{center}
\end{table}

\subsection{Related work} 
\citet{meulemans_minimizing_2022} introduced an approach for training feedforward neural networks with local update rules by minimizing control. The mechanistic implementation of this method can be considered as a specific instance of the least-control principle applied to feedforward neural networks, where we use direct linear feedback connectivity from an output controller to the network to approximate an optimal control (c.f. Section \ref{sec_app:linear_feedback}). Importantly however, the least-control principle is the first to put the idea of learning by minimizing control on a strong and general theoretical foundation, as the theory of \citet{meulemans_minimizing_2022} contains a problematic assumption. 

When computing the gradient of the surrogate loss $\mathcal{H}$ w.r.t. $\theta$ with the implicit function theorem, Meulemans et al. 2022 consider the direct feedback weights $Q$ to be independent of $\theta$. However, after taking an update, the feedback weights need to change to satisfy the column space condition (Meulemans et al. 2022, Condition 1) again for the new parameter setting $\theta$. Hence, the loss $\mathcal{H}=\|Qu\|^2$ is affected by this change in Q, which was not incorporated in the gradient calculation, making Strong-DFC only a heuristically-derived method. We were not able to incorporate a $\theta$ -and $\phi$ dependent feedback mapping $Q$ into the theoretical framework following the methods of Meulemans et al. 2022, using the implicit function theorem alone. We could only resolve these fundamental issues by formulating least-control as an optimal control problem.

Furthermore, the least-control principle provides a way to learn general dynamical systems that reach an equilibrium. These systems perform computations both in space and time and cannot in general be reduced to a feedforward computational graph. Consequently, the least-control principle provides credit assignment both in space and time, significantly widening the range of systems that can be learned compared to \citet{meulemans_minimizing_2022}. 

Finally, we introduce various different controller designs compatible with the least-control principle, that move beyond the direct linear feedback controller introduced by \citet{meulemans_minimizing_2022}. Importantly, the inversion dynamics of Eq. \ref{eqn:deq_inversion} and the energy-based dynamics of Eq. \ref{eqn:generalized_energy_dynamics} satisfy the column space condition of Theorem \ref{theorem:columnspace} exactly and thereby lead to gradient-following updates, in contrast to the direct linear feedback that cannot satisfy the column space condition perfectly and hence results in only approximate gradient updates.

\newpage
\section{The least-control principle for meta-learning} \label{sec_app:meta_learning}

In Section~\ref{section:metalearning}, we consider a meta-learning problem in which the objective is to improve the performance of a learning algorithm as new tasks are encountered. To do so, we equip a neural network with fast weights $\bphi$ that are learned on each task, and slow parameters $\btheta=\{\bomega, \blambda \}$ that parametrize the learning algorithm. The fast weights $\bphi$ quickly learn how to solve a task $\tau$ and $\lambda$ determines how strongly $\bphi$ is pulled towards the slow weights $\bomega$, which consolidate the knowledge that has been previously acquired. The capabilities of the learning algorithm are then measured by evaluating $L^\mathrm{eval}_\tau(\bphi\fss_\tau)$, the performance of the learned neural network on data from task $\tau$ that was held out during learning. Meta-learning finally corresponds to improving the learning performance on many tasks by adjusting the meta-parameters $\btheta = \{\bomega, \blambda\}$. This can be formalized through the following bilevel optimization problem \cite{zucchet_contrastive_2022, rajeswaran_meta-learning_2019}:
\begin{equation}\label{eqn_app:metalearning_bo}
	\min_{\btheta} \, \EE_\tau \! \left [ L^\mathrm{eval}_\tau(\bphi\fss_\tau) \right ] \quad \mathrm{s.t.} \enspace \bphi_\tau\fss = \argmin_{\bphi} L^\mathrm{learn}_\tau(\bphi) + \sum_{i} \frac{\lambda_i}{2} (\bphi_i - \bomega_i)^2.
\end{equation}
Note that compared to Eq.~\ref{eqn:metalearning_bo} in the main text, we now consider the more general setting in which each synapse $i$ has its own attraction strength $\lambda_i$ and meta-learn them. 

We now derive the update rules prescribed by the least-control principle, review the meta-learning algorithms we are comparing against and detail our experimental setup. 

\subsection{Derivation of the meta-parameter updates}

The first step of the derivation is to replace the $\argmin$ in Eq.~\ref{eqn_app:metalearning_bo} by the associated stationarity condition:
\begin{equation}
    \nabla_{\bphi} L^\mathrm{learn}_\tau(\bphi\fss_\tau) + \lambda (\bphi\fss_\tau - \bomega) = 0.
\end{equation}
This equation can be obtained from a dynamical perspective: let us assume that our learning algorithm is gradient descent, so that
\begin{equation}
    \dot{\bphi} = f(\bphi, \btheta) = -\nabla_{\bphi} L^\mathrm{learn}_\tau(\bphi) - \lambda (\bphi - \bomega).
\end{equation}
The stationarity conditions therefore correspond to $f(\bphi\fss_\tau, \btheta) = 0$ and we can apply the least control principle.

To do so, we run the controlled dynamics
\begin{equation}
    \label{eqn_app:controlled_dynamics_meta_learning}
    \dot{\bphi} = -\nabla_{\bphi} L^\mathrm{learn}_\tau(\bphi) - \lambda (\bphi - \bomega) + \bu, \quad \mathrm{and} \quad \dot{u} = -\alpha u - \nabla_{\bphi} L_\tau^\mathrm{eval}(\bphi),
\end{equation}
and note $\bphi\css^\tau$ and $\bu\css^\tau$ the equilibrium. Note that this corresponds to taking $Q=\mathrm{Id}$ and $h(\bphi) = \bphi$ in the general feedback dynamics of Eq.~\ref{eqn:generalized_dynamics}. Theorem~\ref{theorem:columnspace} therefore guarantees that $\bphi\css^\tau$ is an optimally controlled state and $\bu\css^\tau$ an optimal control in the limit $\alpha \rightarrow 0$, as long as $\spder{\bphi}{f}(\bphi\css^\tau, \btheta)$ and $\spder{\by}^2L(h(\bphi\css^\tau))$ are invertible, and the column space condition
\begin{equation}
    \mathrm{col}\left [Q \right ] = \mathrm{row} \left [ \pder{\bphi}{h}(\bphi\css^\tau) \pder{\bphi}{f}(\bphi\css^\tau, \btheta)^{-1} \right ]
\end{equation}
is satisfied. The column space conditions here rewrites
\begin{equation}
    \mathrm{col}[\Id] = \mathrm{row}\left[ \left (\pder{\bphi^2}{^2 L^\mathrm{learn}_\tau}(\bphi\css^\tau) + \lambda \Id \right )^{-1}\right ]\!,
\end{equation}
and is automatically satisfied through the assumption that the Jacobian of $f$ is invertible. To summarize, the controlled dynamics are optimal in the limit $\alpha \rightarrow 0$, if $\spder{\by}{^2L^\mathrm{learn}_\tau}(\bphi^\tau\css) + \lambda \Id$ and $\spder{\by}{^2L^\mathrm{eval}_\tau}(\bphi^\tau\css)$ are invertible.

We now put ourselves in a regime in which the conditions detailed above are satisfied, so that $\bphi\css^\tau = \bphi\opt^\tau$ and $\bu\css^\tau = \bu\opt^\tau$. Theorem~\ref{theorem:first_order_updates} then gives
\begin{equation}
    \Delta \btheta = -\left(\der{\btheta}{}\HH(\btheta)\right)^\top = \pder{\btheta}{f}(\bphi\css^\tau, \btheta)^\top \bu\css^\tau = -\pder{\btheta}{f}(\bphi\css^\tau, \btheta)^\top f(\bphi\css^\tau,\btheta),
\end{equation}
as $\psi = Q\bu$ is here equal to $\bu$. A simple calculation gives
\begin{equation}
    \begin{split}
        \pder{\bomega}{f}(\bphi\css^\tau, \btheta) &= \mathrm{diag}(\lambda)\\
        \pder{\blambda}{f}(\bphi\css^\tau, \btheta) &= -\mathrm{diag}(\bphi\css^\tau - \bomega),
    \end{split}
\end{equation}
where $\mathrm{diag}(x)$ is a diagonal matrix whose diagonal values are the elements of the vector $x$. It further implies that the updates of the parameters are local, as
\begin{equation}
    \label{eqn_app:update_meta_learning}
    \begin{split}
        \Delta \omega &= \lambda \bu\css^\tau = \lambda \left (\nabla_\bphi L_\tau^\mathrm{learn}(\bphi\css^\tau) + \lambda(\bphi\css^\tau-\omega) \right ) \\
        \Delta \lambda &= -(\bphi\css^\tau - \omega) \bu\css^\tau = -(\bphi\css^\tau - \omega) \left (\nabla_\bphi L_\tau^\mathrm{learn}(\bphi\css^\tau) + \lambda(\bphi\css^\tau-\omega) \right ) \! ,
    \end{split}
\end{equation}
the multiplications being performed element wise.

\subsection{Comparison to existing meta-learning algorithms}

We here review existing meta-learning rules that we compare against in Section~\ref{section:metalearning}. We separate those methods in two categories, the one relying on implicit differentiation and the Reptile algorithm \cite{nichol_first-order_2018}.

\textbf{Implicit differentiation methods.} The first set of methods we are reviewing aims at directly minimizing the bilevel optimization problem \eqref{eqn_app:metalearning_bo} using the implicit gradient
\begin{equation}
    \label{eqn_app:implicit_gradient_meta_learning}
    \begin{split}
        \der{\btheta}{}L^\mathrm{eval}_\tau(\bphi\fss_\tau)
        &= \pder{\phi}{L^\mathrm{eval}_\tau}(\phi_\tau^*) \left ( \pder{\phi^2}{^2L_\tau^\mathrm{train}}(\phi_\tau^*) + \mathrm{diag}(\lambda) \right )^{-1} \pder{\btheta\partial \bphi}{^2}\left [ \sum_{i}\frac{\lambda_i}{2}(\bphi^*_{\tau, i} - \omega_i)^2 \right ] \!.
    \end{split}
\end{equation}
In practice, this quantity is intractable because of the inverse of the Hessian and has to be estimated. This lead to different algorithms. T1-T2 \cite{luketina_scalable_2016} uses a first-order approximation of the implicit gradient by simply ignoring the inverse Hessian term:
\begin{equation}
    \label{eqn_app:meta_learning_udpate_T1T2}
    \Delta \omega^\mathrm{T1-T2}_i = \lambda_i \pder{\phi_i}{L^\mathrm{eval}_\tau}(\phi_\tau^*), \quad \mathrm{and} \quad \Delta \lambda_i^\mathrm{T1-T2} = -(\bphi_{\tau, i}\fss - \omega_i) \pder{\phi_i}{L^\mathrm{eval}_\tau}(\phi_\tau^*).
\end{equation}
Alternatively, a better estimate of the gradient can be obtained by iteratively minimizing the quadratic form
\begin{equation}
    \label{eqn_app:quadratic_form_meta_learning}
    \delta \mapsto \delta \frac{1}{2}\left( \pder{\phi^2}{^2L_\tau^\mathrm{train}}(\phi_\tau^*) + \mathrm{diag}(\lambda) \right) \delta^\top + \delta \pder{\bphi}{L^\mathrm{eval}_\tau}(\bphi\fss_\tau)
\end{equation}
and updating the meta-parameters with
\begin{equation}
    \label{eqn_app:meta_learning_udpate_impl}
    \Delta \omega^\mathrm{Impl}_i = -\lambda_i \delta^*_{\tau,i}, \quad \mathrm{and} \quad \Delta \lambda^\mathrm{Impl}_i = (\bphi_{\tau, i}\fss - \omega_i) \delta^*_{\tau,i},
\end{equation}
where $\delta^*_\tau$ is the result of the minimization procedure. Using gradient descent to minimize the quadratic form \eqref{eqn_app:quadratic_form_meta_learning} yields the recurrent backpropagation algorithm \cite[RBP]{almeida_backpropagation_1989, pineda_recurrent_1989}\citeS{liao_reviving_2018}, also known as Neumann series approximation in this context \cite{lorraine_optimizing_2020}, we discussed in more details in Section~\ref{sec_app:rbp}. Gradient descent is a generic algorithm and more specialized optimization algorithms exist for quadratic forms, such as the conjugate gradients one used by \citet[iMAML]{rajeswaran_meta-learning_2019}. Note that the updates \eqref{eqn_app:update_meta_learning}, \eqref{eqn_app:meta_learning_udpate_T1T2} and \eqref{eqn_app:meta_learning_udpate_impl} share some similar structure, but do not incorporate the teaching signal in the same one. The update of T1-T2 does not correspond to any gradient, but the other two do. The estimation of the implicit gradient through Eq.~\ref{eqn_app:meta_learning_udpate_impl} requires computing second derivatives, whereas the other two updates don't need to.

The implicit gradient can be approximated in a totally different way, using the equilibrium propagation theorem \cite{scellier_equilibrium_2017}, as done by the contrastive meta-learning rule \cite{zucchet_contrastive_2022}. The idea is to first minimize the so-called augmented loss
\begin{equation}
    \mathcal{L}_\tau(\bphi, \btheta, \beta) := L^\mathrm{learn}_\tau(\bphi) + \sum_{i} \frac{\lambda_i}{2} (\bphi_i - \bomega_i)^2 + \beta L^\mathrm{eval}_\tau(\bphi)
\end{equation}
for two different values of $\beta$ ($\beta=0$ and some small positive value, note $\bphi_{\tau,0}^*$ and $\bphi_{\tau,\beta}^*$ the solutions) and then update the meta-parameters with
\begin{equation}
    \label{eqn_app:meta_learning_update_cml}
    \Delta \bomega^\mathrm{CML}_i = \frac{\bphi_{\tau, \beta, i}^* - \bphi_{\tau, 0, i}^*}{\beta}, \quad \mathrm{and} \quad \Delta \blambda^\mathrm{CML}_i = -\frac{(\bphi_{\tau, \beta, i}^* - \omega_i)^2 - (\bphi_{\tau, 0,i}^* - \omega_i)^2}{2\beta}.
\end{equation}
The equilibrium propagation theorem guarantees that those updates will converge to the true gradient when $\beta\rightarrow 0$. Both the contrastive meta-learning rule and the rule prescribed by our principle are first-order, as that they don't require computing second derivative. However, the former approximates the implicit gradient when the latter approximates the gradient of the least-control objective $\HH$. Interestingly, as we mentioned in Section~\ref{sec:constrained_energy}, the augmented solution $\bphi\fss_\beta$ and the controlled equilibrium $\bphi\css^\tau$ are both stationary points of the augmented loss $\calL_\tau$ (taking $\beta = \alpha^{-1}$ for $\bphi\css^\tau$). However, the contrastive meta-learning update \eqref{eqn_app:meta_learning_update_cml} is justified in the weak nudging limit ($\beta \rightarrow 0$), whereas the update \eqref{eqn_app:update_meta_learning} is justified in the perfect control limit ($\alpha = \beta^{-1} \rightarrow 0$).

We refer to \citetS{zucchet_beyond_2022} for a more detailed review of those methods.

\textbf{Reptile.} The Reptile algorithm \cite{nichol_first-order_2018} meta-learns the initialization of a neural network such that it can adapt to a new task $\tau$ in very few gradient descent steps. Let us denote $\omega$ the parameters of the network at initialization. The Reptile algorithm first minimizes the learning loss $L^\mathrm{learn}_\tau(\bphi)$ and then update $\omega$ through
\begin{equation}
    \Delta \omega^\mathrm{Reptile} = \phi_\tau^* - \omega,
\end{equation}
with $\phi_\tau^*$ the result of the learning loss minimization. Contrary to the methods presented above, this algorithm is only heuristically motivated and does not rely on any theoretical foundations. Still, it performs surprisingly well on many few-shot learning tasks. 

\subsection{Experimental details}

We compare the LCP to existing meta-learning and few-shot learning algorithms on the Omniglot few-shot image classification benchmark \cite{lake_human-level_2015}. We follow the standard training and evaluation used in prior works \cite{finn_model-agnostic_2017, rajeswaran_meta-learning_2019}. We focus on the 20-way 1-shot setting, and only meta-learn $\btheta = \{\bomega\}$ while choosing a global $\blambda$ as a hyperparameter. The iMAML result in Table~\ref{tab:omniglot} is taken from \citet{rajeswaran_meta-learning_2019} while those of Reptile and FOMAML are taken from \citet{nichol_first-order_2018}. 

\textbf{Data preprocessing and augmentation.} We resize the Omniglot images to $28\times 28$ and augment the images by random rotation by multiples of $90$ degrees during meta-training.

\textbf{Architecture.} For our classifier, we use the same architecture as in previous works \cite{finn_model-agnostic_2017, rajeswaran_meta-learning_2019}, which consists in 4 convolutional modules with a $3\times3$ convolutions and 64 filters, each followed by a batch normalization layer, a ReLU nonlinearity, and a $2 \times 2$ max-pooling layer. The output of the module is then flattened and fed into a final classification layer followed by a softmax activation.

\textbf{Optimization.} We train all  models on 100000 tasks, split into meta batches of size 16. Each task consists in a 20-way classification problem with a single example per class provided for the adaptation phase, and 15 examples per class for evaluating the adaptation. For T1-T2, the adaptation consists in 100 gradient steps with learning rate $\eta=0.03$. For LCP, we use 100 steps of the discretized version of the dynamics \eqref{eqn_app:controlled_dynamics_meta_learning}:
\begin{equation}
    \begin{split}
        \bphi_{\mathrm{next}} &= \bphi +\eta \left (-\nabla_{\bphi} L^\mathrm{learn}_\tau(\bphi) - \lambda (\bphi - \bomega) + \bu \right)\\
        u_{\mathrm{next}} &= u+ \eta \left (-\alpha u - \nabla_{\bphi} L_\tau^\mathrm{eval}(\bphi) \right )\!,
    \end{split}
\end{equation}
where $\eta=0.03$.
The Adam optimizer takes the $\Delta \omega$ as input and updates the meta-parameters, with learning rate 0.001. Finally, the models are evaluated on 100 test tasks. See Table \ref{tab:hps-omniglot} for an overview of the hyperparameters.

\begin{table}[bt]
\begin{center}
\caption{Hyperparameter used for the 20-way 1-shot Omniglot experiment.}
\label{tab:hps-omniglot}
\begin{tabular}{@{}lll@{}}
\toprule
Hyperparameter             & T1-T2                                     & LCP                                  \\ \midrule
\texttt{num\_tasks}   & 100000                                      & 100000                                      \\
\texttt{meta\_batch\_size}   & 16                                      & 16                                      \\
\texttt{optimizer\_outer}  & ADAM                                    & ADAM                                    \\
\texttt{lr\_outer}         & 0.001             & 0.001     \\
\texttt{steps\_inner}      & 100             & 100      \\
$\eta$         & 0.03             & 0.03     \\
$\alpha$                   & -                                       & 0.1          \\
$\lambda$                  & 0.001                                       & 0.001                                       \\
\bottomrule
\end{tabular}
\end{center}
\end{table}
\newpage

\bibliographystyleS{unsrtnat}
\bibliographyS{references}
\end{document}